\newif\ifconfver
\confverfalse      

\newif\ifonecoltab
\onecoltabtrue        

\newif\ifplainver  
\plainvertrue

\ifplainver
\confverfalse                         
\fi

\ifconfver
\documentclass[10pt,journal]{IEEEtran}
\else
\ifplainver
        \documentclass[10pt]{article}
\usepackage{palatino}
\usepackage{fullpage}
\else
\documentclass[12pt,draftcls,onecolumn]{IEEEtran}
\fi
\fi

\usepackage{calc,amsfonts,amssymb,amsmath,bm,url,color,theorem,graphicx,cite,epstopdf,nicefrac,stmaryrd}
\usepackage{psfrag,subfigure,float,epstopdf,bbold}
\usepackage{multirow,comment}
\usepackage{algorithm,algpseudocode}
\usepackage[normalem]{ulem}
\usepackage{mdframed}
\usepackage{footnote,makecell}
\makesavenoteenv{tabular}
\definecolor{orange}{RGB}{255,107,0}
\def\black{\color{black}}
\def\red{\color{red}}

\usepackage{threeparttable}

\usepackage{soul,xcolor}

\setstcolor{blue}

\newtheorem{Fact}{Fact}

\theoremstyle{definition}

\newtheorem{Assumption}{Assumption}

\algnewcommand{\algorithmicforeach}{\textbf{for each}}
\algdef{SE}[FOR]{ForEach}{EndForEach}[1]
  {\algorithmicforeach\ #1\ \algorithmicdo}
  {\algorithmicend\ \algorithmicforeach}

\newcommand{\M}{\boldsymbol{M}}
\newcommand{\A}{\boldsymbol{A}}

\newcommand{\T}{{\!\top\!}}

\DeclareMathOperator*{\minimize}{\textrm{minimize}}

\newtheorem{theorem}{Theorem}
\newtheorem{lemma}{Lemma}
\newtheorem{proposition}{Proposition}
\newtheorem{corollary}{Corollary}
\newtheorem{definition}{Definition}
\newtheorem{remark}{Remark}
\newtheorem{proof}{Proof}

\makeatletter 
\pretocmd\@bibitem{\color{black}\csname keycolor#1\endcsname}{}{\fail}
\newcommand\citecolor[1]{\@namedef{keycolor#1}{\color{blue}}}
\makeatother

\begin{document}

	\newcommand{\papertitle}{
	\bf	Mixed Membership Graph Clustering via Systematic Edge Query
	}
	
	\newcommand{\paperabstract}{%
	This work considers clustering nodes of a largely incomplete graph. Under the problem setting, only a small amount of queries about the edges can be made, but the entire graph is not observable. This problem finds applications in large-scale data clustering using limited annotations, community detection under restricted survey resources, and graph topology inference under hidden/removed node interactions. {\black Prior works tackled this problem from various perspectives, e.g., convex programming-based low-rank matrix completion and active query-based clique finding.
	Nonetheless, many existing methods are designed for estimating the single-cluster membership of the nodes, but nodes may often have mixed (i.e., multi-cluster) membership in practice.
	Some query and computational paradigms, e.g.,
	the random query patterns and nuclear norm-based optimization  advocated in the convex approaches, may give rise to scalability and implementation challenges.}
	This work aims at learning mixed membership {\black of nodes using queried edges}. 
	{\black The proposed method is developed together with} {a systematic query principle that} can be controlled and adjusted by the system designers to accommodate implementation challenges---e.g., to avoid querying edges that are physically hard to acquire. 
	Our framework also features a lightweight and scalable algorithm with membership learning guarantees. 
	Real-data experiments on crowdclustering and community detection are used to showcase the effectiveness of our method.
	}
	
	
	\ifplainver
	
	\date{\today}
	
	\title{\papertitle}
	
	\author{
		 Shahana Ibrahim and Xiao Fu
		\\ ~ \\
		School of Electrical Engineering and Computer Science\\ Oregon State University\\
		Corvallis, OR 97331, USA
		\\~
	}
	
	\maketitle

	\begin{abstract}
		\paperabstract
	\end{abstract}
	
	\else
	\title{\papertitle}
	
	\ifconfver \else {\linespread{1.1} \rm \fi
		
		\author{Shahana Ibrahim and Xiao Fu
			
			\thanks{
				%
				
				
				S. Ibrahim and X. Fu are with the School of Electrical Engineering and Computer Science, Oregon State University, Corvallis, OR 97331, USA. email (ibrahish, xiao.fu)@oregonstate.edu

				This work is supported in part by the National Science Foundation under Project NSF IIS-2007836 and NSF IIS-1910118, and the Army Research Office under Project ARO W911NF-19-1-0407.

			}
		}
		
		\maketitle
		
		\ifconfver \else
		\begin{center} \vspace*{-2\baselineskip}
		\end{center}
		\fi
		
		\begin{abstract}
			\paperabstract
		\end{abstract}
		
\begin{IEEEkeywords}
	Graph clustering, mixed membership, sampled edge query, nonnegative matrix factorization
\end{IEEEkeywords}
		
		\ifconfver \else \IEEEpeerreviewmaketitle} \fi
	
	\fi
	
	\ifconfver \else
	\ifplainver \else
	\newpage
	\fi \fi



\section{Introduction}
{\black {\it Graph clustering} (GC)} is of broad interest in data science, which aims at associating the nodes of a graph with different clusters in an unsupervised manner \cite{schaeffer2007graph}.
{\black The GC technique}
is well-motivated, since network data frequently arise in various applications (e.g., in social network analysis, brain signal processing, and biological/ecological data mining). 


{\black In the past two decades, much effort was invested onto dealing with the GC problem; see, e.g., \cite{airoldi2008mixed,kuang2012symmetric,van2008graph, saade2016clustering,schaub2020blind}.} {\black Nonetheless,} some major challenges arise in the era of big data. 
Notably, many graph data are highly incomplete for various reasons.
For example, {\black a large social network} follower-followee data could contain billions of nodes, which translates to $\approx 10^{18}$ edges. {\black Data} acquisition at such a scale is a highly nontrivial task. In many cases, instead of collecting edge information of the entire network, data analysts have to {\it sample} edges of interest, and use the sampled network to perform {\black GC} \cite{jure2006sampling,hu2013asurvey}.
Graph sampling (or, edge query)-based network analysis is also well-motivated {\black for other reasons}---e.g., in community detection of networks where edges are intentionally removed or hidden (e.g., terrorist networks or radical group networks) \cite{matthew2004sampling} and in biological/ecological networks where acquiring the complete edge information is too resource-consuming \cite{Malek2008using,gill2015surveysoybean}.

A number of works considered {\black GC} under incomplete edge observation. {\black For example, the works in  \cite{korlakai2014graph,korlakai2016crowdsourced,chen2014clustering} used uniformly random edge queries and proposed convex low-rank matrix completion-based algorithms for GC. Another line of work \cite{arya2017clustering,arya2016clustering, ramya2018graph} proposed active edge query-based methods in which the learning algorithms query the edges on-the-fly, in order to reduce the number of queries needed.}

{\black The approaches in \cite{korlakai2014graph,korlakai2016crowdsourced,chen2014clustering,arya2017clustering,arya2016clustering, ramya2018graph} are insightful, and showed that edge query-based GC is a largely viable task, in both theory and practice.
Nonetheless, a} number of challenges remain. {\black  First, the methods in \cite{korlakai2014graph,korlakai2016crowdsourced,chen2014clustering,arya2017clustering,arya2016clustering, ramya2018graph} were not designed to estimate multiple membership of the nodes.
However, each node is often associated with different clusters in real-world networks} (e.g., a person in a co-author network could belong to the machine learning and statistics communities simultaneously). 
{\black  Second, the theoretical guarantees in  \cite{korlakai2014graph,korlakai2016crowdsourced,chen2014clustering,arya2017clustering,arya2016clustering, ramya2018graph} were also established under single membership models---and oftentimes based on random query strategies (see, e.g., \cite{korlakai2014graph,korlakai2016crowdsourced,chen2014clustering}).
Note that in certain applications (e.g., field surveys based graph/network analysis \cite{sarndal1992model}), systematic edge query strategies may be more preferred. 
Third, the existing methods are not always scalable. In particular, the convex optimization approaches in \cite{korlakai2014graph,korlakai2016crowdsourced,chen2014clustering}}
recast the edge query-based graph clustering problem as a nuclear norm minimization problem, which entails $N^2$ (where $N$ is the number of nodes) optimization variables---making it hard to {\black handle} real-world large-scale graphs.




\smallskip

\noindent
{\bf Contributions.} 
In this work, we offer an alternative framework for learning the node membership from {\black an} incomplete graph.
Some notable features of our framework are as follows:

\noindent
$\bullet$ {\bf Systematic Edge Query Strategy.} Our algorithm design comes with {\black a} {\it systematic} edge query {\black scheme}, under which the node membership can be provably learned. As discussed, systematic query may be more preferable than random query strategies in a variety of applications.

\noindent
$\bullet$ {\bf {\black Guaranteed} Mixed Membership Identification.} 
Unlike existing provable edge-query based graph clustering methods in{\black \cite{korlakai2014graph,korlakai2016crowdsourced,chen2014clustering,arya2017clustering,arya2016clustering, ramya2018graph}} {\black whose goals are to} learn single membership of nodes belonging to disjoint clusters, we model the undirected adjacency graphs using a mixed membership model that allows nodes to be associated with multiple overlapping clusters.   
{Our model is reminiscent of the \textit{mixed membership stochastic block} (MMSB) model in overlapped community detection \cite{airoldi2008mixed}.}
Accordingly, we offer mixed membership identification guarantees using queried edges.
{\black Another contribution is that we derive a new sufficient condition for mixed membership identification. 
Compared to existing MMSB identifiability conditions, e.g., those in \cite{mao2017mixed,mao2020estimating,Panov2017consistent,huang2019detect},
our new condition is more closely tied to key characterizations of the GC problem, e.g., the Dirichlet prior of node membership, the graph size, and the cluster-cluster interaction intensity.}




\noindent
$\bullet$ {\bf Scalable and Lightweight Algorithm.}
We propose a simple procedure that only consists of the truncated {\it singular value decomposition} (SVD) of small matrices and a Gram--Schmidt-type greedy algorithm for {\it simplex-structured matrix factorization} (SSMF). This procedure is {\black computationally very} economical, {\black especially when} compared to the convex optimization approaches in \cite{korlakai2014graph,korlakai2016crowdsourced,chen2014clustering,yi2012semi}. We also provide sample complexity and noise robustness analyses for the proposed framework.

\noindent
$\bullet$ {\bf Evaluation on Real-World Datasets.} We conduct {\black a series of experiments} on real-world datasets. First, we consider the query-based crowdclustering task \cite{korlakai2014graph,korlakai2016crowdsourced}, using annotator-labeled graph data to assist image clustering. The data used in our evaluation was uploaded to the \textit{Amazon Mechanical Turk} (AMT) platform and labeled by unknown human annotators. The acquired AMT data has been made publicly available\footnote{ \url{https://github.com/shahanaibrahimosu/mixed-membership-graph-clustering}}. 
Second, we also use co-authorship network datasets, {\it Digital Bibliography \& Library Project} (DBLP) and \textit{Microsoft Academic Graph} (MAG), to evaluate our method on community detection tasks.

\smallskip

An initial and limited version of this work was {\black published in the proceedings of} IEEE ICASSP 2021 \cite{ibrahim2021icassp}. In this journal version, we additionally include (1) detailed identifiability analysis of the proposed algorithm (Theorem~\ref{prop:subspace_estim} and its proof); (2) the integrated performance characterization of the SVD and SSMF stages (Theorem~\ref{prop:subspace_estim_noisy} and its proof); 
(3) a new application, namely, crowdclustering; (4) a series of new real-data experiments; (5) and a newly acquired dataset for crowdclustering that is made publicly available.

\smallskip

\noindent
{\bf Notation.}
$x,\bm x,\bm X$ denote a scalar, vector and matrix, respectively.  $\kappa(\bm X)$ denotes the condition number of the matrix $\bm X$ and is given by $\kappa(\bm X) = \frac{\sigma_{\max}(\bm X)}{\sigma_{\min}(\bm X)}$ where $\sigma_{\max}$ and $\sigma_{\min}$ are the largest and smallest nonzero singular values of $\bm X$, respectively.  $\bm X \ge \bm 0$ denotes that all the elements of $\bm X$ are nonnegative.  $\|\bm X\|_{2}$ represents the 2-norm of matrix $\bm X$, i.e., $\|\bm X\|_{2} = \sigma_{\max}(\bm X)$.  $\|\bm X\|_{\rm F}$ denotes the Frobenius norm of $\bm X$. {\black $\|\bm X\|_{\ast}$ denotes the nuclear norm of $\bm X$. ${\rm svd}_K(\bm X)$ denotes the top-$K$ singular value decomposition (SVD) of $\bm X$.} ${\sf range}(\bm X)$ denotes the subspace spanned by the columns of $\bm X$. $\bm X(i,j)$ and $x_{i,j}$  both denote the element of $\bm X$ from its $i$th row and $j$th column. 
$\bm X(:,j)$ and $\bm x_j$ both denote the $j$th column of $\bm X$. 
$\|\bm x\|_2$ and $\|\bm x\|_1$ denote $\ell_2$ and $\ell_1$ norms of vector $\bm x$, respectively.  $^\T$ and $^\dag$ denote transpose and pseudo-inverse, respectively. $|{\cal C}|$ denotes the cardinality of set ${\cal C}$. $\cup$ denotes the union operator for sets. { $[T]$ represents the set of positive integers up to $T$, i.e., $[T]=\{1,\ldots,T\}$. }
 ${\rm Diag}(\bm x)$ is a diagonal matrix that holds the entries of $\bm x$ as its diagonal elements.

\section{Problem Statement}
Consider $N$ data {entities (e.g., persons in a social network or image data in a sample-sample similarity network)} that are from $K$ clusters. We allow a node to belong to multiple clusters; i.e., we consider the case where the clusters have overlaps and the nodes admits {\it mixed membership} {\black \cite{airoldi2008mixed}}.
Assume that the $n$th {entity} belongs to cluster $k$ with probability $m_{k,n}$, where
\begin{equation}\label{eq:softmember}
   \sum_{k=1}^K m_{k,n}=1,~m_{k,n}\geq 0.
\end{equation}
Then, the vector $\bm m_n =[m_{1,n},\ldots,m_{K,n}]^\T$ is referred to as the {\it membership vector} of data entity $n$, since it reflects the association of entity $n$ with different clusters. All such vectors together constitute the \textit{membership matrix} $\bm M = [\bm m_1,\ldots,\bm m_N] \in \mathbb{R}^{K \times N}$. 

In {\black GC}, the entities are the nodes of the graph, and the relationship between the nodes are represented as the edges of the graph. In this work, we consider undirected graphs that are represented as
symmetric adjacency matrices, i.e., $\bm A \in \{0,1\}^{N \times N}$, where each entry $\bm A(i,j)$ encodes the pairwise relationship between nodes $i$ and $j$ in terms of binary values (i.e., $\A(i,j)\in\{0,1\}$). 
Our goal is to learn the membership vector $\bm m_n$ for each {node from limited edge queries about the adjacency matrix $\bm A$ (i.e., the number of edges queried is much smaller than total number of edges, i.e., $N(N-1)/2$).}

\subsection{Motivating Examples}\label{sec:examples}
Our problem setting is motivated by a couple of important applications, namely, the crowdclustering problem and incomplete graph based overlapping community detection.

\noindent
$\bullet$~{\bf Example 1: Crowdclustering.}
\textit{Crowdclustering} is a technique that clusters data samples with the assistance from crowdsourced annotators \cite{korlakai2014graph,korlakai2016crowdsourced, ryan2011crowdclustering,yi2012semi,jinfeng2012crowdclustering,saade2016clustering}{\black \cite{ramya2018graph,arya2016clustering,arya2017clustering}}.
In crowdclustering, the data entities are presented to the annotators {in pairs}. The annotators determine if the pair of entities are from the same category or not; i.e., the annotators apply a rule such that $    \bm A(i,j) =  1$ if entities $i,j~(i<j)$ are believed to be from the same cluster, and $\A(i,j)=0$ otherwise.
Instead of asking the annotators to determine the exact membership of the $n$th data entity (i.e., the $\bm m_n$ vector), the above mentioned rule only asks the annotators to output binary labels (i.e., if two entities are similar or not). This annotation paradigm is arguably more accurate than exact membership labeling, since it can work under much less expertise. 




{
Crowdclustering amounts to learning $\bm m_n$ from $\bm A$, which is a {\black GC} problem.
The challenge is that annotating the full adjacency graph {\black requires $O(N^2)$ pairs to be compared and labeled}---a heavy workload if large data sets are considered.
{\black Hence, edge query-based GC techniques are natural to handle the crowdclustering task. This way, only a subset of the $O(N^2)$ sample pairs needs to be compared and annotated.
}

}



\noindent
$\bullet$~{ \bf Example 2: Edge Sampling-Based Overlapping Community Detection.} {The mixed membership modeling in \eqref{eq:softmember} and the learning problem of interest can also be applied to {\it overlapping community detection} (OCD) \cite{xie2013overlapping}---since the mixed membership setting implies that the communities have overlaps. Classic OCD algorithms work with statistical generative models, e.g., the MMSB \cite{airoldi2008mixed,huang2019detect} {and the \textit{Bayesian nonnegative matrix factorization model}  \cite{ioannis2011overlapping}}. The mixed membership identifiability guarantees were established under full observation of $\A$ in {\black \cite{huang2019detect,Panov2017consistent,mao2017mixed, mao2020estimating }}. Mixed membership identification for OCD under sampled edges is of great interest for applications like field survey based community analysis \cite{sarndal1992model} or community detection in link-removed/hidden networks (e.g., terrorist networks) \cite{matthew2004sampling}.
In both cases, one may not be able to observe the entire $\A$ due to reasons such as resource limitations and difficulty of edge acquisition.
However, theoretical guarantees and provable algorithms for this problem have been elusive.}

\subsection{Prior Work}
A theoretical challenge associated with our problem is as follows: Is it possible to learn the membership vector $\bm m_n$ from {the binary matrix} $\bm A$? This gives rise to a membership {\it identifiability} problem.
If $\bm A$ is fully observed,
answering the identifiability question amounts to understanding theoretical guarantees for similarity graph-based membership learning techniques, e.g., MMSB identification and spectral clustering---which has been thoroughly studied in the machine learning community; see, e.g., {\black \cite{Panov2017consistent,ng2002spectral,huang2019detect,mao2017mixed, mao2020estimating}}. However, when $\A$ is only partially observed, the identifiability problem has not been
fully understood.

To handle the {\black GC} problem with partial observations, the work in \cite{korlakai2014graph,korlakai2016crowdsourced} models the generating process of $\bm A$ using the SBM followed by a random edge query stage for data acquisition. 
The SBM can be summarized as follows. Assume that $\bm B\in\mathbb{R}^{K\times K}$ represents a cluster-cluster similarity matrix which is symmetric and $\bm B(p,q)$ represents the probability that cluster $p$ is connected with cluster $q$. In addition, assume that the nodes are connected through their membership identities. Then, the probability that $\bm A(i,j)=1$ is $\bm P(i,j)=\bm m_i^\T\bm B\bm m_j$, i.e., $    \bm A(i,j) \sim {\sf Bernoulli}\left(  \bm m_i^\T\bm B\bm m_j \right)$ where $i<j$---the adjacency matrix $\bm A$ is sampled from Bernoulli distributions specified by the entries of the matrix $\bm P=\bm M^\T\bm B\bm M$. Note that under SBM, the membership vector $\bm m_i$ is always a unit vector; i.e., only single membership and disjoint clusters are considered under SBM.

To recover $\bm M$ from partially observed $\A$ under SBM, the work in  \cite{Samet2011finding,chen2014clustering,korlakai2014graph,korlakai2016crowdsourced} use the following convex program and its variants:
    \begin{align}\label{eq:convex}
        \minimize_{\bm L\in {\black [0,1]^{N\times N}},\bm S}&~\|\bm L\|_{\ast} +\lambda\|\bm S\|_1  \\
        {\rm subject~to}&~{\bm L}(i,j) + \bm S(i,j) = \bm A(i,j),~\forall (i,j)\in \bm \varOmega.  \nonumber
    \end{align}
In the above, $\bm L$ is used to model the low-rank component of $\bm A$, $\bm S$ is a sparse error matrix, and $\bm \varOmega$ denotes the set of pairs $(i,j)$ such that the corresponding $\bm A(i,j)$'s are observed.
Notably, it was shown in \cite{korlakai2014graph,korlakai2016crowdsourced,chen2014clustering} that if $\bm A$ follows the SBM, then the above convex program (and its variants) recovers $\bm A$ up to certain errors with provable guarantees using {\it random} edge queries.
The results from \cite{korlakai2014graph,korlakai2016crowdsourced,chen2014clustering} have opened many doors for edge query-based graph clustering, but several caveats exist. {In particular, the recoverability is established based on single membership models, but the mixed membership case is of more interest. In addition, the optimization problem involves $O(N^2)$ variables (memory$\approx 3,000$GB when $N=10^6$), and thus is hardly scalable.} {\black Also}, the provable guarantees in \cite{korlakai2014graph,korlakai2016crowdsourced,chen2014clustering} rely on random {edge} query, which may not be easy to implement under some scenarios, as we discussed. 

{\black Query-based GC was also studied in another line of work \cite{arya2017clustering,arya2016clustering, ramya2018graph}. 
There,
the focus is developing active/online edge query-based methods that can discover the clusters on-the-fly, while
attaining certain information-theoretic query lower bounds. 
Unlike the methods in \cite{korlakai2014graph,korlakai2016crowdsourced,chen2014clustering}
that take a latent model identification perspective,
the algorithms in \cite{arya2017clustering,arya2016clustering, ramya2018graph} are more from a successive clique (i.e., fully connected subgraph) finding perspective (although some connections to the SBM were also made). Similar to the convex approaches, the methods in \cite{arya2017clustering,arya2016clustering, ramya2018graph} are designed to learn single membership.
}





\section{Proposed Approach}




In this work, we relax the SBM assumption on $\bm P=\bm M^\T\bm B\bm M$ by allowing the nodes to have mixed membership, {\black as advocated in  \cite{airoldi2008mixed,anandkumar2014tensor,huang2019detect}}. Hence, the constraints on $\bm M$ becomes
\begin{equation}\label{eq:MMSB_M}
\bm 1^\T\bm M = \bm 1^\T,\quad \bm M\geq \bm 0;
\end{equation}
i.e., $\bm m_n$ resides in the probability simplex, instead of being the vertices of the simplex as in the SBM.
Under the mixed membership assumption in \eqref{eq:MMSB_M}, the following Bernoulli model, i.e.,
\begin{equation}\label{eq:A_gen}
\begin{aligned}
     \begin{cases}
&\bm A(i,j) \sim {\sf Bernoulli}\left(  \bm m_i^\T\bm B\bm m_j \right),~i < j,\\
    &\bm A(i,j)=\bm A(j,i), ~i>j,\\
    &\bm A(i,i)=0,~i\in[N],
    \end{cases}    
\end{aligned}
\end{equation}
{is adopted in our generative model for the complete adjacency matrix. Note that $ \bm m_i^\T\bm B\bm m_j$ physically corresponds to the probability that nodes $i$ and $j$ admit an edge (i.e., ${\sf Pr}(\A(i,j)=1)= \bm m_i^\T\bm B\bm m_j$), and thus this Bernoulli model makes intuitive sense; {\black see more discussions in  \cite{airoldi2008mixed,anandkumar2014tensor,huang2019detect}}.}

\subsection{Systematic Edge Query}


Our goal is to learn $\bm M$ from {controlled} edge sampling strategies. {To proceed}, 
we first divide the nodes into $L$ disjoint groups $\mathcal{S}_{1}, \mathcal{S}_{2}, \dots, \mathcal{S}_{L}$ such that $\mathcal{S}_{1} \cup \dots \cup \mathcal{S}_{L} = [N]$.
Let $\bm A_{\ell,m}\in\mathbb{R}^{|\mathcal{S}_{\ell}| \times |\mathcal{S}_{m}|}$ denote the adjacency matrix between groups of nodes indexed by $\mathcal{S}_{\ell}$ and $\mathcal{S}_{m}$. 
We propose an edge query principle as described in the following:


\begin{mdframed}
{\bf Edge Query Principle (EQP).}
Query the blocks of edges such that the sampled $\A_{\ell,m}$'s satisfy the following two conditions:\\
$\bullet$  For every $\ell\in[L]$, 
$    K\leq |\mathcal{S}_{\ell}| $ holds. \\
$\bullet$ Let $m_r \in [L]$ and $\{\ell_r\}_{r=1}^L = [L]$.   For every $\ell_r$, there exists a pair of indices $m_r$ and $\ell_{r+1}$ where ${\ell_{r+1}} \neq {\ell_r}$ such that the edges from the blocks $\bm A_{\ell_r,m_r}$ and $\bm A_{\ell_{r+1},m_r}$   are queried.
\end{mdframed} 

\vspace{.25cm}

Note that since $\bm A$ is symmetric, when we  select a block $\bm A_{\ell,m}$ to be queried, it also covers $\bm A_{m,\ell}$ since $\bm A_{m,\ell} = \bm A_{\ell,m}^{\top}$. 
A couple of remarks are in order: 

\begin{remark}
The proposed EQP covers a large variety of query `masks'---as shown in Fig.~\ref{fig:query}. Since the query pattern can be {\it by design} instead of random, this entails the flexibility to avoid querying edges that are known {\it a priori} hard to acquire, e.g., edges that may have been intentionally removed to conceal information or edges that correspond to interactions between groups that are hard to survey due to various reasons, e.g., long geographical distance. 
\end{remark}

\begin{remark}
Instead of sampling individual edges, we sample {\it blocks} of edges under the proposed EQP. 
As one will see, this simple block query pattern allows us to design a provable and lightweight algorithm for mixed membership learning.
{\black The proposed} query strategy can be easily implemented when the query patterns are controlled by the network analysts. For example, in crowdclustering, one can dispatch blocks of sample pairs for similarity annotation. In community analysis, field surveys can be designed using EQP such that the blocks which are easier to be queried are selected. 

\end{remark}


 \begin{figure}
	\centering
	\includegraphics[width=0.32\linewidth]{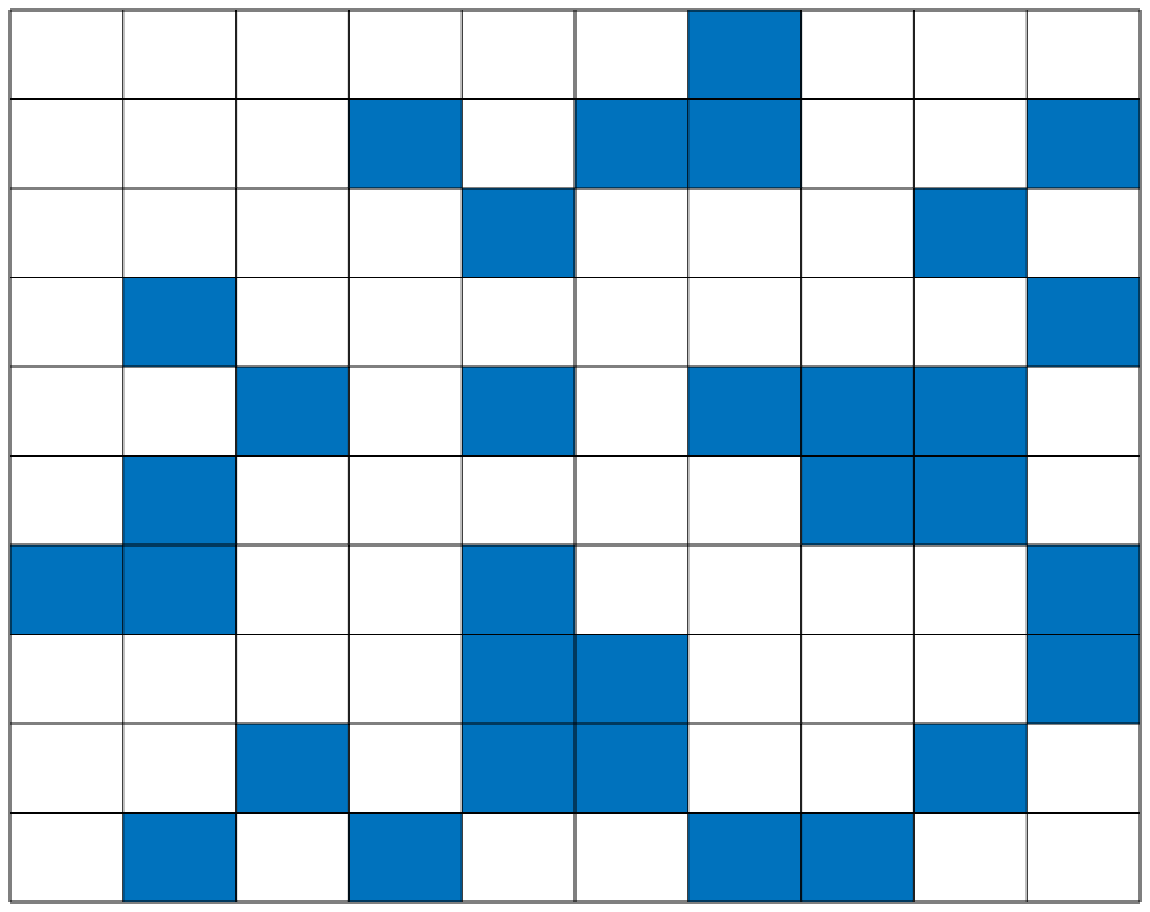}
		\includegraphics[width=0.32\linewidth]{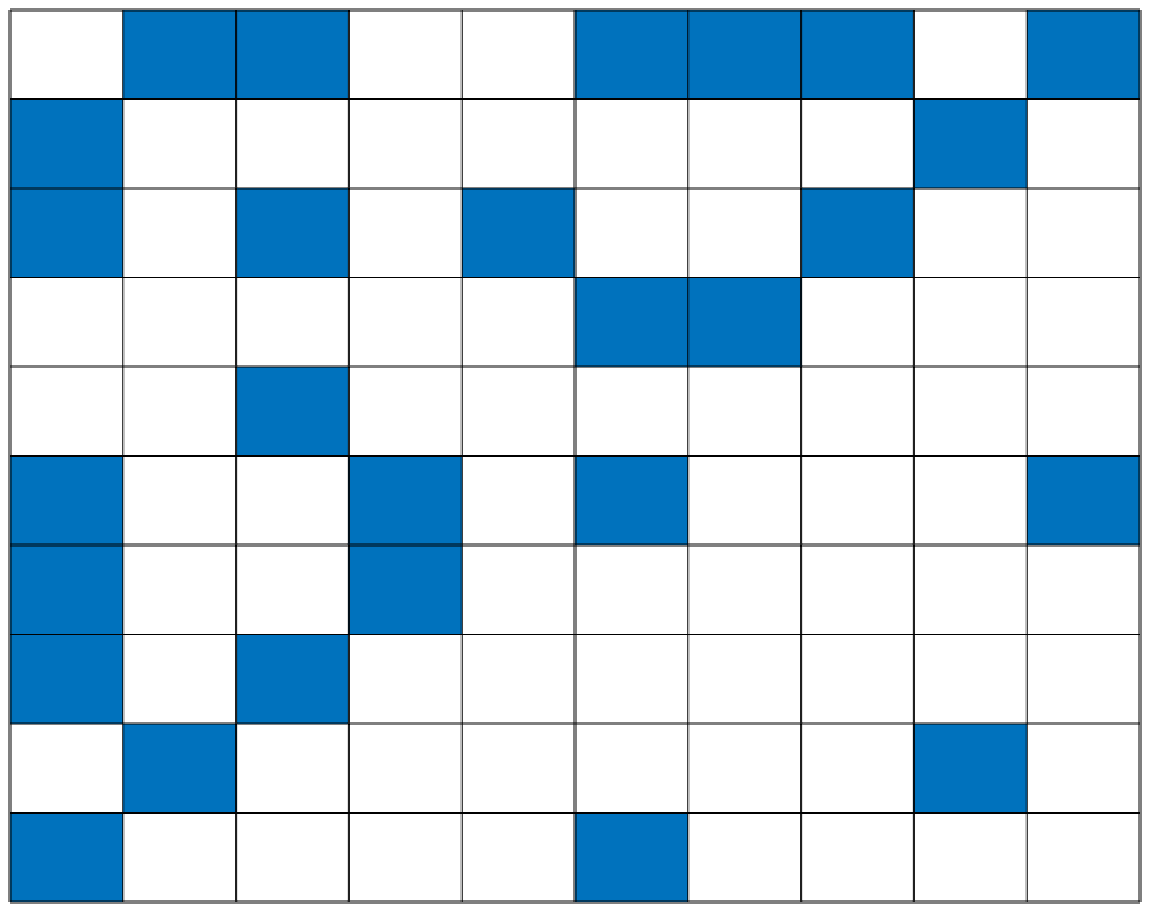}
		\includegraphics[width=0.32\linewidth]{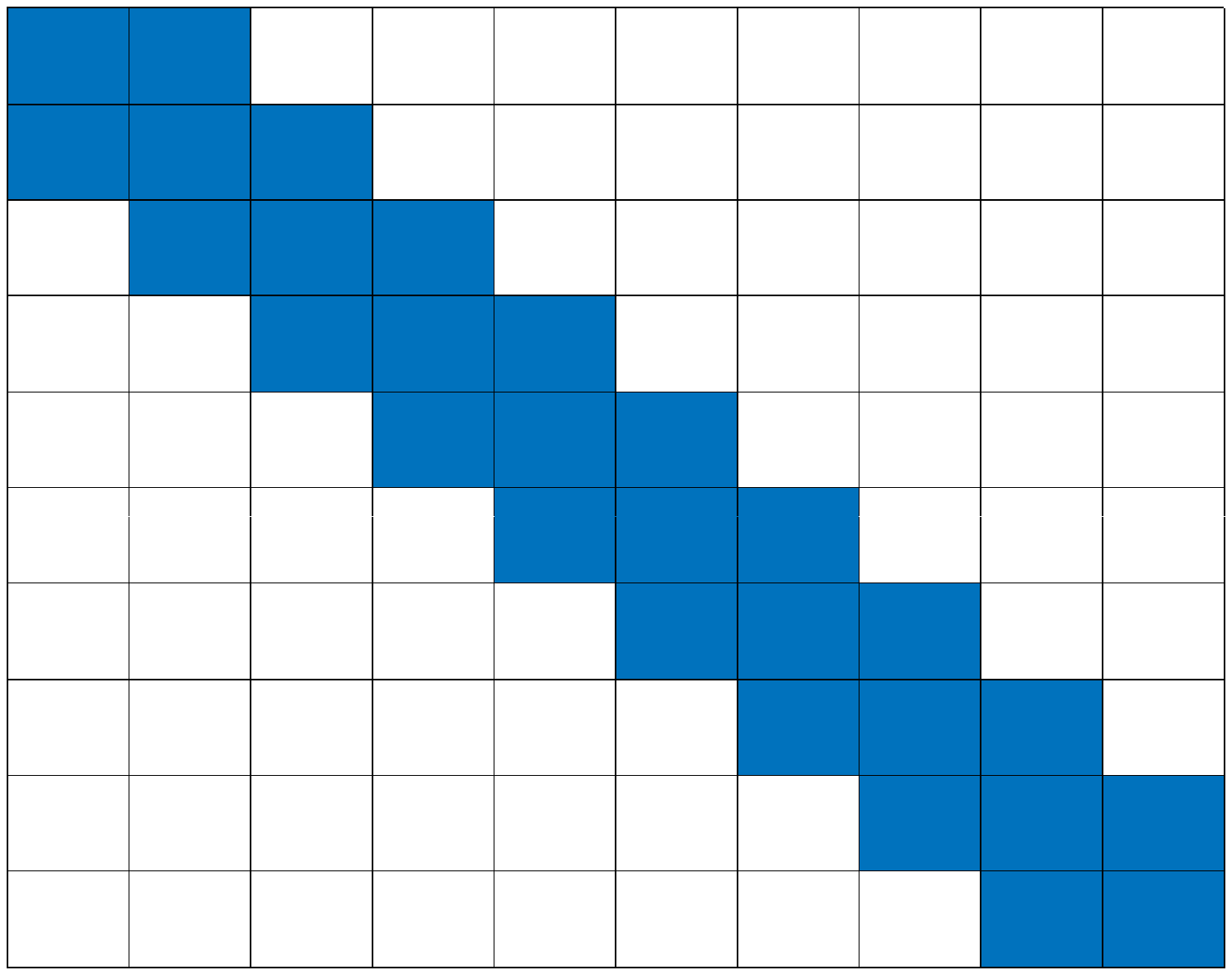}
	\caption{Some example patterns for $\bm A$ following EQP with $N=1000, K=5$ and $L=10$. The shaded blue region represents the blocks queried.}
	\label{fig:query}
	\vspace{-.5cm}
\end{figure}

\subsection{Connections to Matrix Completion}
From a matrix completion viewpoint, our EQP is most closely related to those in \cite{franz2015algebraicjmlr,pimentel2015char} that consider recoverable sampling patterns of matrix completion (MC). The work in \cite{franz2015algebraicjmlr} relates the recoverability of a low-rank matrix to the structural properties of a bipartite graph by considering the rows and columns as nodes. It establishes the ``finite" (not unique) completability of the matrix based on combinatorial conditions on the observed edges in its subgraphs. In a similar spirit, the work in \cite{pimentel2015char} shows that the range space of a low-rank matrix can be uniquely identified under the existence of certain patterns formed by the sampled entries.

From an MC perspective, our block sampling pattern can be considered as special cases covered by those in \cite{franz2015algebraicjmlr,pimentel2015char}---with important distinguishing features. First, the work in \cite{franz2015algebraicjmlr,pimentel2015char} established recoverability under the sampling pattern, but had no tractable algorithms. Our EQP features a polynomial-time lightweight algorithm via exploiting the special block sampling pattern. Second, the recoverability results in \cite{franz2015algebraicjmlr,pimentel2015char} only cover sampled continuous real-valued low-rank matrices without noise. Our approach can provably estimate the range space of binary matrices (or, highly noisy low-rank matrices) under reasonable conditions.

\subsection{Algorithm Design}
In this section, we propose an algorithm under the EQP.
Specifically, under the considered model in \eqref{eq:A_gen}.
we develop an algorithm that consists of simple SVD operations to estimate ${\sf range}(\bm M^\T)$ and a subsequent SSMF stage to `extract' $\bm M$ from the estimated range space.

\subsubsection{Main Idea -- A Toy Example} \label{sec:stitch}
To shed some light on how our algorithm approximately identifies ${\sf range}(\M^\T)$, let us consider the ideal case where $\bm A_{\ell,m}=\bm P_{\ell,m} = \bm M_{\ell}^{\top}\bm B \bm M_{m}$.  We start by analyzing a toy example with $L=3$ and $K\leq N/3$; see Fig.~\ref{fig:pattern_small}. We assume that the blocks $\bm P_{1,2}\in\mathbb{R}^{N/3\times N/3}$, $\bm P_{2,2}\in\mathbb{R}^{N/3\times N/3}$, and $\bm P_{3,1}\in\mathbb{R}^{N/3\times N/3}$ are queried. Hence, by symmetry, the following blocks are known: 
\begin{align} \label{eq:Ps}
\bm P_{1,2}&= \M_{1}^{\top}{\bm B}\M_{2}, ~
\bm P_{2,2}= \M_{2}^{\top}{\bm B}\M_{2},\\\bm P_{2,1}&= \M_{2}^{\top}{\bm B}\M_{1},~
\bm P_{3,1}= \M_{3}^{\top}{\bm B}\M_{1},
\end{align}


\color{black}
Define 
$
\bm C_1 := [\bm P_{1,2}^{\top}~,~ \bm P_{2,2}^{\top}]^{\top}$ and $\bm C_2 := [\bm P_{2,1}^{\top}~,~ \bm P_{3,1}^{\top}]^{\top}$.
The top-$K$ SVD of $\bm C_1$ and $\bm C_2$ can be represented as follows:
\begin{align} \label{eq:CD}
\bm C_1 = [\bm U_1^{\top}, \bm U_{2}^{\top}]^{\top}\bm \Sigma~{\bm V_2^{\top}},~\bm C_2 = [\overline{\bm U}_{2}^{\top}, \overline{\bm U}_{3}^{\top}]^{\top} \overline{\bm\Sigma}~\overline{\bm V}_1^{\top}.
\end{align}
Combining \eqref{eq:Ps}-\eqref{eq:CD}, and under the assumption that ${\rm rank}(\bm M_\ell)={\rm rank}(\bm B)=K$ and $K\leq |\mathcal{S}_\ell|$, for all $\ell$,
one can express the bases of ${\sf range}(\bm M_{1}^\T)$, ${\sf range}(\bm M_{2}^\T)$ and ${\sf range}(\bm M_{3}^\T)$ as
$
\bm U_1 = \bm M_1^{\top}\bm G$, $\bm U_{2} = \bm M_{2}^{\top}\bm G$,
$\overline{\bm U}_{3} = \bm M_{3}^{\top} \overline{\bm G}$, respectively,
where $\bm G\in \mathbb{R}^{K\times K}$ and $\overline{\bm G} \in \mathbb{R}^{K \times K}$ are certain nonsingular matrices.
Our hope is to ``stitch'' the bases above to have 
\begin{align}
    {\sf range}(\bm U) &= {\sf range}([\bm U_{1}^{\T},\bm U_{2}^{\T},\bm U_{3}^{\T}]^{\top})\nonumber\\
    &= {\sf range}([\bm M_{1},\bm M_{2},\bm M_{3}]^\T),\label{eq:goal_ideal}
\end{align}
with $\bm U_1$ and $\bm U_2$ in \eqref{eq:CD} and a certain $\bm U_3$.
Note that $\overline{\bm U}_{3}$ cannot be directly combined with $\bm U_{1}$ and $\bm U_{2}$ to attain the above, since $\bm G = \overline{\bm G}$ does not generally hold.
To fix this, we define the following operation:
$\bm U_{3}:=  \overline{\bm U}_{3} \overline{\bm U}_{2}^{\dagger} \bm U_{2} .$
It is not hard to see that
$$
   \overline{\bm U}_{3} \overline{\bm U}_{2}^{\dagger} \bm U_{2} =  \bm M_{3}^{\top} \overline{\bm G} \times   \left(\bm M_{2}^{\top} \overline{\bm G} \right)^{\dagger}  \times \bm M_{2}^{\top}\bm G=   \bm M_{3}^{\top} \bm G,
$$
which leads to \eqref{eq:goal_ideal}. 
\color{black}
{The idea conveyed by this simple example can be recursively applied to cover a general $L$ block case under the proposed EQP, and the range space estimation accuracy can be guaranteed even when $\A$ is a noisy (binary) version of $\bm P$---which will be detailed in the next part. }
After $\bm U^\T=\bm G^\T\bm M$ is obtained, extracting $\bm M$ from $\bm U^\T$ is the so-called SSMF problem, which has been extensively studied \cite{fu2018nonnegative}; see {\black a brief discussion in Sec.~\ref{sec:ssmf} and the supplementary material (Sec.~\ref{sup:spa})}.

 \begin{figure}[t]
    \centering
    \includegraphics[scale=0.25]{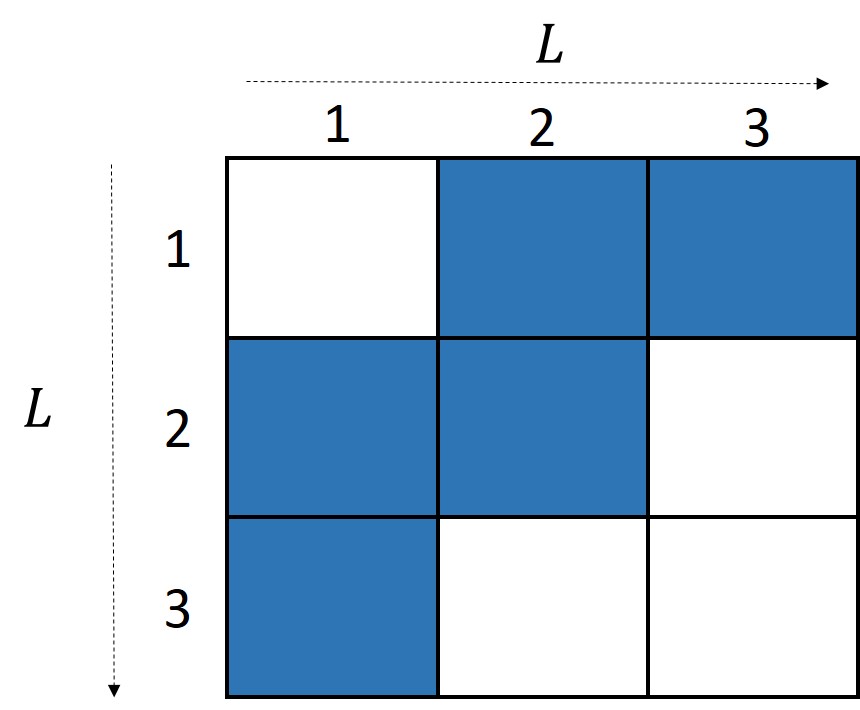}
    \caption{\black An illustrative case for subspace identifiability analysis. The shaded blue region represents the blocks queried.}
\label{fig:pattern_small}
\end{figure}


\subsubsection{Proposed Algorithm} \label{sec:alg}
In a similar spirit as alluded in the toy example,
we propose Algorithm~\ref{algo:proposed} (which is referred to as the {\it \underline{b}lock \underline{e}dge \underline{que}ry-based graph \underline{c}lustering} (\texttt{BeQuec}) algorithm).
One can see that the \texttt{BeQuec} algorithm only consists of top-$K$ SVD and least squares, which can be applied to very large graphs as long as $K$ is of a moderate size. The SSMF stage (carried out by the {\it successive projection algorithm} (\texttt{SPA}) {\black\cite{Gillis2012}}) is a Gram--Schmidt type algorithm that is also scalable (see Algorithm~\ref{algo:SPA} {\black in the supplementary material in Sec. \ref{sup:spa}}).
Note that in Algorithm~\ref{algo:proposed}, $\bm U_{\rm ref1}$ and $\bm U_{\rm ref2}$ are temporary variables defined to pass the subspace estimated in the current iteration to the subsequent iteration.
We further explain the two major stages of \texttt{BeQuec}, namely, the range space estimation stage and the membership extraction stage as follows:



\begin{algorithm}[t]
\black 
\caption{\small \texttt{Proposed Algorithm: BeQuec}}
\label{algo:proposed}
\begin{algorithmic}[1]
\footnotesize
	\Require $L$, $K$, $\{\bm A_{\ell_r,m_r}\}_{r=1}^L$, $\{\bm A_{\ell_{r+1},m_{r}}\}_{r=1}^{L-1}$; 
	


\State  $ T \leftarrow  \left\lfloor L/2 \right\rfloor$; \label{line:Tinit}

\State $\bm C_{T} \leftarrow [\bm A_{\ell_{T},m_{T}}^{\top}~,~\bm A_{\ell_{T+1},m_{T}}^{\top}]^{\top}$; \label{algoline:Ct}

 \State  $ [{\bm U}_{\ell_T}^{\top}, {\bm U}_{\ell_{T+1}}^{\top}]^{\top}\bm \Sigma_T{\bm V}_{m_T}^{\top}\leftarrow {\rm svd}_K(\bm C_{T})$;\label{algoline:Ut}
 
 \State $[{\bm U}_{\rm ref2}^{\top}, {\bm U}_{\rm ref1}^{\top}]^{\top} \leftarrow  [{\bm U}_{\ell_T}^{\top}, {\bm U}_{\ell_{T+1}}^{\top}]^{\top}$;\algorithmiccomment{\parbox[t]{0.33\linewidth}{Init. temp. variables} }

	\ForEach{$r=T+1:1:L-1$}  \label{algoline:iterative1_begin}\algorithmiccomment{\parbox[t]{0.33\linewidth}{Est. $\{{\bm U}_{\ell_{r}}\}_{r=T+2}^{L}$} }

 \State ${\bm U}_{\ell_{r+1}} \leftarrow $  \texttt{PairStitch}$(\bm A_{\ell_r,m_r}, \bm A_{\ell_{r+1},m_r},\bm U_{\rm ref1})$ 
 
 \State $\bm U_{\rm ref1} \leftarrow {\bm U}_{\ell_{r+1}}$;

 \EndForEach   \label{algoline:iterative1_end}

	\ForEach{$r=T:-1:2$}\label{algoline:iterative2_begin}\algorithmiccomment{\parbox[t]{0.33\linewidth}{Est. $\{{\bm U}_{\ell_{r}}\}_{r=1}^{T-1}$} }
 
 
  \State ${\bm U}_{\ell_{r-1}} \leftarrow$  \texttt{PairStitch}$(\bm A_{\ell_{r-1},m_{r-1}}, \bm A_{\ell_{r},m_{r-1}},\bm U_{\rm ref2})$

 \State $\bm U_{\rm ref2} \leftarrow {\bm U}_{\ell_{r-1}}$;
 
\EndForEach \label{algoline:iterative2_end}

 \State  $\widehat{\bm U} \leftarrow \begin{bmatrix} \bm U_1^\T,\dots, \bm U_L^\T\end{bmatrix}^\T$; \label{line:Uestim}
 
\State  $\widehat{\bm G} \leftarrow$ {\texttt {SPA}}($\widehat{\bm U} $, $K$) \cite{Gillis2012}; \label{line:Gestim}\algorithmiccomment{\parbox[t]{0.33\linewidth}{Est. $\bm G$ under \eqref{eq:SMF}}}

  \State 
  $\widehat{\M}\leftarrow \widehat{\bm G}^{-\T}\widehat{\bm U}^{\top}$ (or constrained least squares);\label{line:Mestim}
  
 \State \Return  $\widehat{\bm M}$
\end{algorithmic}

\end{algorithm}
{\black 
\begin{algorithm}[t]
\black
\caption{\small \texttt{PairStitch}}
\label{algo:subspace_stitiching}
\begin{algorithmic}[1]
\footnotesize
	\Require  $\bm A_{\ell,m}$, $\bm A_{\ell',m}$,  $\bm U_{\rm ref}$

\State  ${\bm C} \leftarrow [\bm A_{\ell,m}^{\top}~,~ \bm A_{\ell',m}^{\top}]^{\top}$;
 
 \State 
 $ [\overline{\bm U}_{\ell}^{\top}, \overline{\bm U}_{\ell'}^{\top}]^{\top}\bm \Sigma\bm V_{m}^{\top}\leftarrow{\rm svd}_K(\bm C)$;
 
 \State  ${\bm U}_{\ell} \leftarrow  \overline{\bm U}_{\ell} \overline{\bm U}_{\ell'}^{\dagger}\bm U_{\rm ref}$;\algorithmiccomment{``{Stitch}'' ${\bm U}_{\ell}$ and $\bm U_{\rm ref}$ }


 
 
 \State \Return  ${\bm U}_{\ell}$.
\end{algorithmic}

\end{algorithm}
}

    

\subsection{\black Key Ingredients of \texttt{BeQuec}}
\subsubsection{\black Range Space Estimation  (Lines \ref{line:Tinit}-{\black \ref{line:Uestim}})} 
In a nutshell, \texttt{BeQuec} uses the idea from the toy example in an iterative fashion on the queried blocks $\bm A_{\ell_r,m_r}$ and $\bm A_{\ell_{r+1},m_{r}}$ for $r=1,\dots,L-1$---to estimate $\bm U^\T=\bm G^\T\bm M$, i.e., the range space of $\bm M^\T$. We start the iterations from $r=\left\lfloor L/2 \right\rfloor$ and perform the subspace stitching of the blocks in the ascending and descending orders, respectively---which helps reduce the subspace estimation error from the overall procedure, as will be seen in the analysis of Theorem~\ref{prop:subspace_estim_noisy}.

\subsubsection{ SSMF for Membership Extraction (Lines {\black \ref{line:Gestim}-\ref{line:Mestim}})} \label{sec:ssmf}
If $\bm U$ is perfectly estimated, the second stage estimates $\bm M$ from the following SSMF model:
\begin{equation}\label{eq:SMF}
    \bm U^\T = \bm G^\T \bm M,~\bm M\geq \bm 0,~\bm 1^\T\bm M =\bm 1^\T,
\end{equation}
where $\bm G\in \mathbb{R}^{K\times K}$ is nonsingular.
The SSMF problem is the cornerstone of a number of core tasks in machine learning and signal processing, e.g., hyperspectral unmixing and topic modeling, which admits a plethora of algorithms for provably estimating $\bm M$ under certain conditions \cite{fu2018nonnegative}. To be precise, line \ref{line:Gestim} of Algorithm~\ref{algo:proposed} invokes an SSMF algorithm, namely, \texttt{SPA} \cite{Gillis2012}. {\black
Under the model in \eqref{eq:SMF},
\texttt{SPA} guarantees the identifiability of $\bm M$ utilizing the so-called {\it anchor node assumption} (ANC). Here, ``identifiability'' of $\bm M$ means that learning $\bm M$ up to a row permutation ambiguity can be guaranteed; see \cite{MVES,fu2018identifiability,huang2014non}. 

In essence, the ANC is identical to the so-called \textit{separability} condition \cite{donoho2003does} from the SSMF and nonnegative matrix factorization (NMF) literature:
\begin{definition}(ANC/Separability) \cite{donoho2003does,ibrahim2019crowdsourcing,Panov2017consistent} \label{def:sep}
The nonnegative matrix $\bm M \in \mathbb{R}^{K \times N}$ is said to satisfy the separability condition/ANC if there exists an index set $\bm \varLambda=\{q_1,\dots, q_K\}$ such that $\bm M(:,q_k)= \bm e_k$ for all $k \in [K]$. When there exists $\bm \varLambda =\{q_1,\dots, q_K\}$ such that $\|\bm M(:,q_k)-\bm e_k\|_2\leq \varepsilon$ for all $k \in [K]$, then $\bm M$ is said to satisfy the $\varepsilon$-separability condition.
\end{definition}
In the context of GC, the ANC means that there exists a node $q_k$ (i.e., an anchor node) that only belongs to cluster $k$ for $k\in[K]$.
Under the ANC, one can see that if $\bm \varLambda$ is known, then $\bm G=\bm U(\bm \varLambda,:)$ can be easily ``read out'' from the rows of $\bm U$, which also enables estimating $\bm M$ using (constrained) least squares as well (see line \ref{line:Mestim} in Algorithm~\ref{algo:proposed}).
The \texttt{SPA} algorithm identifies $\bm \varLambda$ using a Gram-Schmidt-like procedure in $K$ iterations and its per-iteration complexity is $O(NK^2)$ operations \cite{fu2018nonnegative,Gillis2012}; also see Sec.~\ref{sup:spa}  in the supplementary material for details.

\begin{remark}
A number of MMSB learning algorithms rely on the ANC to establish identifiability of $\bm M$; see, e.g., \cite{mao2017mixed,Panov2017consistent,mao2020estimating}.  However,
ANC does not reflect many important aspects of the GC problem, other than the existence of some anchor nodes.
In the next subsection, we will present a sufficient condition that has a different flavor, and is arguably more intuitive in the context of GC (e.g., graph size and the cluster-cluster interaction pattern). Another remark is that \texttt{SPA} is not the only algorithm for separable SSMF/NMF; see more options in \cite{kumar2012fast,mizutani2014ellipsoidal,fu2015robust, gillis2014successive,fu2014self,fu2016robust}. Nonetheless, \texttt{SPA} strikes a good balance between complexity and accuracy.
\end{remark}

} 
 




\color{black}

\subsection{Performance Analysis} \label{sec:perf}
\subsubsection{{\black Algorithm} Complexity} Note that Algorithm~\ref{algo:proposed} only consists of top-$K$ SVD on small blocks that have a size of $ 2(N/L) \times (N/L)$ assuming $|\mathcal{S}_\ell|=N/L$ for all $\ell$, which has a complexity of $O((N/L)K^2)$ flops---a similar complexity order of the SPA stage.
In terms of memory, the process never needs to instantiate any $N\times N$ matrix that is used in the convex programming based methods [cf. \eqref{eq:convex}]. Instead, the variables involved are with the size of $N\times K$, thereby being economical in terms of memory (since the number of clusters $K$ is often much smaller than $N$).

\subsubsection{Accuracy} We first use the ideal case where $\bm A(i,j)=\bm P(i,j)$ to {\black analyze} the key idea behind our proposed approach. Building on the understanding to the ideal noiseless case, the binary observation case, i.e, the model in \eqref{eq:A_gen}, will be analyzed by treating it as a noisy version of the ideal case, with extensive care paid to the noise induced by the Bernoulli observation process.

To be specific, we have the following subspace identification result:
\begin{theorem} (Ideal Case) \label{prop:subspace_estim}
{Assume that $\bm A_{\ell, m}=\bm P_{\ell, m} = \bm M_{\ell}^{\top}\bm B \bm M_{m} \in \mathbb{R}^{|\mathcal{S}_{\ell}|\times |\mathcal{S}_{m}|}$ holds true for all $\ell,m \in [L]$ and ${\rm rank}(\bm M_\ell)={\rm rank}(\bm B)=K$. Suppose that the $\A_{\ell,m}$'s are queried according to the proposed EQP.
 Then, the output $\widehat{\bm U}$ by Algorithm~\ref{algo:proposed} satisfies $
{\sf range}(\widehat{\bm U})  = {\sf range}(\bm M^{\top})$.}
\end{theorem}
The proof of the theorem is relegated to Appendix~\ref{supp:subspace_estim}.
The proof showcases how the range spaces of $\M_\ell^\T$'s are `stitched' together to recover ${\sf range}(\M^\T)$, if $\A_{\ell,m}=\bm P_{\ell,m}$.

{\black However, the Bernoulli parameters (i.e., the $\bm P_{\ell, m}$'s) are not observed in practice}. Instead, one observes $\bm A_{\ell, m}$'s with binary values, following the Bernoulli observation model in \eqref{eq:A_gen}.  
The Bernoulli observation process can be understood as a noisy data acquisition process. To characterize the performance under this noisy case, we consider the following {\black assumption that is also reminiscent of the classic MMSB model \cite{airoldi2008mixed}:}
\begin{Assumption}\label{as:as1}
{\black     The membership vectors $\bm m_n$ $\forall n \in [N]$ are independently drawn from the Dirichlet distribution with parameter $\bm \nu= [\nu_1,\dots,\nu_K]^{\top}$, i.e., $\bm m_n\sim{\sf Dir}(\bm \nu)$ for $n\in[N]$. 
}
\end{Assumption}
{\black We first connect the Dirichlet model with ANC/separability using the following proposition:
}
{\black 
 \begin{proposition}\label{lem:Lm}
	Let $\mu > 0 ,\varepsilon >0$. Under Assumption \ref{as:as1}, assume that the number of nodes $N$ satisfies
	\begin{align} \label{eq:thmLm}
	N = \Omega\left((G(\varepsilon,\bm \nu))^{-2}{\rm log}\left(K/{\mu}\right)\right),
	\end{align}
	where 
	\begin{align} \label{eq:Gdef}
	    G(\varepsilon,\bm \nu) =  \frac{\Gamma (\sum_{i=1}^K \nu_i)}{\prod_{i=1}^K \Gamma(\nu_i)} \min_k\frac{\varepsilon^{\sum_{i \neq k}\nu_i}}{\prod_{i \neq k}\nu_i},
	\end{align}
	and the gamma function $\Gamma(z) = \int_0^{\infty} x^{z-1}\exp(-x)dx$ for $z > 0$.
	  Then, with probability of at least $1-\mu$, $\bm M$ satisfies the $\varepsilon$-separability condition.
\end{proposition}
}
{\black The proof can be found in Appendix~\ref{app:dirichlet}. This result translates the ANC to a condition that relies on the graph size $N$.}
{\black 
To explain the condition in \eqref{eq:thmLm},
it is critical to understand $G(\varepsilon,\bm \nu)$. The value of $G(\varepsilon,\bm \nu)$ is affected by how well $\bm m_n$'s are spread in the probability simplex. 
To illustrate this, consider a case with $N=1000$, $K=3$ and $\varepsilon=0.1$. Table \ref{tab:Gval} tabulates the value of $G(\varepsilon,\bm \nu)$ under this case for different Dirichlet parameters $\bm \nu$, namely, $[0.5,0.5,0.5]^\T$, $[2.0,0.5,0.5]^\T$, and $[3.0,3.0,3.0]^\T$. The corresponding $\bm m_n$'s distributions in the probability simplex are shown in Fig.~\ref{fig:alpha11}. 
One can see from Table \ref{tab:Gval} and \eqref{eq:thmLm} that $\bm \nu= [0.5,0.5,0.5]^{\top}$ has a larger $G$-function value and thus
requires a smaller $N$ to satisfy the $\varepsilon$-separability condition. 
This is consistent with the illustration in Fig.~\ref{fig:alpha11}.





Under Assumption \ref{as:as1}, if $|\mathcal{S}_{\ell}| \ge K$, then ${\rm rank}(\bm M_\ell) =K$ holds with probability one \cite{sidiropoulos2014parallel,sidiropoulos2012multi}. This means that the following also holds with probability one:
\begin{align*}
    \gamma := \max_{\ell}\kappa(\bm M_{\ell}) < \infty.
\end{align*}

\begin{table}[t]
\black 
  \centering
  \caption{The value of the function $G(\varepsilon,\bm \nu)$ for different Dirichlet parameter $\bm \nu$ fixing $K=3$ and $\varepsilon=0.1$. }
  \resizebox{0.3\linewidth}{!}{
    \begin{tabular}{c|c}
    \hline
    \multicolumn{1}{c|}{$\bm \nu$} & \multicolumn{1}{c}{$G(\varepsilon,\bm \nu)$} \\
    \hline
    \hline
       $[0.5,0.5,0.5]^{\top}$   & $0.045$ \\
    \hline
        $[2.0,0.5,0.5]^{\top}$  & $8.4\times 10^{-4}$ \\
    \hline
        $[3.0,3.0,3.0]^{\top}$  & $7.0\times 10^{-5}$ \\
    \hline
    \hline
    \end{tabular}%
    }
  \label{tab:Gval}%
\end{table}%
 }

\begin{figure}[t]
\black 
	\centering 
	\includegraphics[scale=0.35]{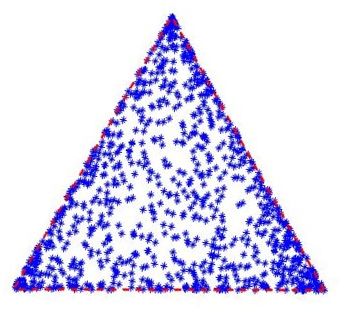} 
	\includegraphics[scale=0.35]{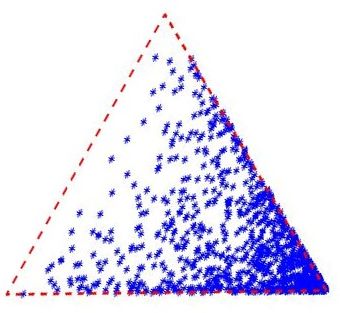}	
	\includegraphics[scale=0.35]{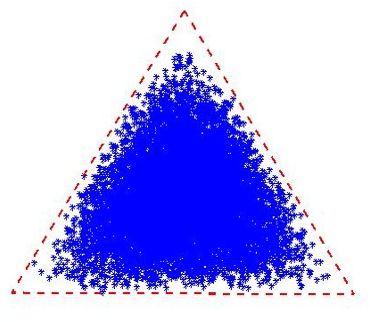}
	\caption{The columns of $\bm M \in \mathbb{R}^{3 \times 1000}$ generated via the Dirichlet distribution with $\bm \nu =[0.5,0.5,0.5]^{\top}$ (left), $\bm \nu =[2.0,0.5,0.5]^{\top}$ (middle)  and $\bm \nu=[3.0,3.0,3.0]^{\top}$ (right) spread in the 2-dimensional probability simplex.}
	\label{fig:alpha11}
\end{figure}

{\black To proceed,}
we adopt the following definition of node degree that is widely used in network analysis {\black \cite{lux2007tutorial,scott1988social,lei2015consistency}}:
\begin{definition}
The degree of node $i$ is the number of ``similar nodes'' it has in the adjacency graph; i.e., ${\sf degree}(i)= \sum_{j=1}^N\A(i,j)$, where $\A\in\{0,1\}^{N\times N}$.
\end{definition}

{\black Using Assumption~\ref{as:as1}},
{\black Proposition~\ref{lem:Lm}, the notations $\gamma$ and {\sf degree}($i$),
we state the main theorem}:

{\black 
\begin{theorem}
 (Binary Observation Case) \label{prop:subspace_estim_noisy} Assume that the matrix $\A$ is generated following \eqref{eq:MMSB_M} and \eqref{eq:A_gen} and ${\rm rank}(\bm B)=K$. Also assume that $\A_{\ell,m}$'s are queried following the EQP, with $|\mathcal{S}_\ell| = N/L$ for every $\ell$, where $N/L$ is an integer.
	Let $\rho := \max_{i,j}\bm P(i,j)$ be the maximal entry of $\bm P$. Suppose that $\rho = \Omega(L\log(N/L)/N)$ and $L = O(\rho N/d)$ where $d$ is the maximal degree of all the nodes. Also assume that \begin{equation}\label{eq:Ncond}
	    N = \Omega\left(\max\left(L^2,\frac{\log\left(NK/{L^2}\right)}{G(\varepsilon,\bm \nu)^2},\frac{(K\gamma^2)^L\rho\kappa^2(\bm B)}{\sigma_{\min}^2(\bm B)}\right)\right) 
	\end{equation}for a certain constant ${\black 0 < \varepsilon} = O\left( \frac{1}{K\gamma^3}\right)$.
	Then, the output $\widehat{\bm U}$ and $\widehat{\bm M}$ by Algorithm~\ref{algo:proposed} satisfy the following with probability of at least $1-O(L^2/N)$:
	\begin{align*}
	\|\widehat{\bm U}-\bm{U}\bm O\|_{\rm F} &= \zeta =  O\left(\frac{(K\gamma^2)^{L/2}\kappa(\bm B)\sqrt{\rho}}{\sigma_{\min}(\bm B)\sqrt{N/L} }\right),\\
   \min_{\bm \Pi}	{\|\widehat{\bm M}-\bm \Pi\bm M\|_{\rm F}} &= 
 O\left(K^2\sigma_{\max}(\bm M)\gamma^3\max\left(\varepsilon, \zeta\right)\right),
	\end{align*}
	 where $\bm U$ is an orthogonal basis of ${\sf range}(\bm M^\T)$, $\bm O \in \mathbb{R}^{K \times K}$ is an orthogonal matrix and $\bm \Pi \in \mathbb{R}^{K \times K}$ is a permutation matrix. 
\end{theorem}
}


The proof can be found in { Appendix}~\ref{supp:subspace_estim_noisy}, {\black which is an integration of (i) range space estimation accuracy under Bernoulli observations, (ii) least squares perturbation analysis, and (iii) noise robustness analysis of \texttt{SPA}.
In particular, the conditions on $\rho$ and $d$ are needed for range space estimation under binary observations, while the conditions on $N$ are used across parts (i)-(iii).


Theorem~\ref{prop:subspace_estim_noisy} can be understood as follows. 
First, under the presumed model, the proposed method favors larger graph sizes (i.e., larger $N$'s) and larger query blocks (i.e., smaller $L$'s). 
Particularly, the nested condition on $N$ in \eqref{eq:Ncond} means that $N$ should be larger than a certain threshold, since $N\geq E_1\log(E_2 N)$ for positive constants $E_1$ and $E_2$ holds if $N$ is sufficiently large.
Second, if the nodes tend to be more focused on a single cluster (i.e., $\nu_k$ is small for every $k$), the $G$-function's value in \eqref{eq:Gdef} becomes larger (cf. Fig.~\ref{fig:alpha11} and Table~\ref{tab:Gval}). Then, our method can reach a good accuracy level under smaller $N$'s and larger $L$'s. Third, the cluster-cluster interaction matrix $\bm B$ also matters: If the intra-cluster interactions dominate all node interactions and the cluster activities are balanced (i.e., $\bm B(k,k)\approx \bm B(j,j)$ for $j\neq k\in[K]$), then the condition number of $\bm B$ is expected to be small, thereby making membership learning easier.
}

{\black Theorem~\ref{prop:subspace_estim_noisy} also sheds light on the impact of $L$.
Ideally, one hopes to increase $L$ since using a larger $L$ means that the EQP uses fewer queries. With a proper $L$, the number of queries under the proposed EQP, i.e., $\Omega(N^2/L)$, can approach $NK$, which is the number of parameters to characterize any (unstructured) $K$-dimensional subspace in an $N$-dimensional space. However,
} 
$L$ cannot be too large, {\black as a large $L$} may lead to the EQP condition $K\leq |\mathcal{S}_{\ell}|=N/L$ being violated. In addition, a larger $L$ makes the error bound in Theorem~\ref{prop:subspace_estim_noisy} looser, since the factor $(K\gamma^2)^{L/2}$ increases with $L$.
Hence, the theorem reveals an important trade-off that is of interest to query designers.

\begin{remark}\label{rmk:baseline}
{\black Given the underlying SSMF structure of each block $\A_{\ell,m}$, a natural question is as follows: Instead of first estimating ${\sf range}(\bm M^\T)$ and then estimating $\bm M$, why not estimate $\bm M_\ell$ (or $\bm M_m$) from $\A_{\ell,m}$ (using SSMF algorithms on the range space of $\A_{\ell,m}$, i.e., $\bm U_{\ell,m}\approx\bm M_\ell^{\top}\bm G_m$) and then stitch them together to recover $\bm M$? This is not infeasible---but is prone to failure.
}
{\black
One reason is that this route needs that every block $\bm M_\ell$ satisfies the ANC (at least approximately).
This is much more unrealistic than requiring the entire $\bm M$ to satisfy the ANC (which is also reflected in Proposition~\ref{lem:Lm}). 
}
\end{remark}

\section{Experiments}\label{sec:experiments}


\begin{table}[t]
\black 
  \centering
  \caption{The subspace distance between $\widehat{\bm U}$ and $\bm U$ for \texttt{BeQuec} in ideal and binary observation cases; $K=5$; $L=10$; $\bm \nu = [\frac{1}{5} \frac{1}{5} \frac{1}{5} \frac{1}{5} \frac{1}{5}]^{\top} $.}
   \resizebox{0.55\linewidth}{!}{

    \begin{tabular}{c|c||c}
    \hline
    \multirow{2}[1]{*}{{Graph Size ($N$)}} & {Ideal Case} & {Binary Obs. Case} \\
\cline{2-3}          & {Dist} & {Dist} \\
    \hline
    \hline
    2000  &  2.48$\times 10^{-14}$     & 0.3214 \\
    \hline
    4000  &    4.86$\times 10^{-14}$   & 0.2080 \\
    \hline
    8000  &  2.63$\times 10^{-14}$     & 0.1412 \\
    \hline
    10000 &  1.28$\times 10^{-13}$     & 0.1257 \\
    \hline
    \hline
    \end{tabular}%
  \label{tab:dist}%
  }
\end{table}%

\begin{table}[t]
\black
  \centering
 \caption{The MSE, RE and averaged SRC of $\bm M$ for $K=5$, $L=10$ and $\bm \nu = [\frac{1}{5} \frac{1}{5} \frac{1}{5} \frac{1}{5} \frac{1}{5}]^{\top} $.} 
 \resizebox{0.7\linewidth}{!}{

    \begin{tabular}{c|c|c|c|c|c}
    \hline
    {Graph Size ($N$)} & {Metric} & \texttt{BeQuec} & \texttt{GeoNMF} & \texttt{CD-MVSI} & \texttt{CD-BNMF} \\
    \hline
    \hline
    \multirow{3}[2]{*}{2000} & {MSE} & \textbf{0.0536} & 0.1060 & 0.8160 & 0.2632 \\
\cline{2-6}          & {RE} & \textbf{0.2403} & {0.2603} & 0.9757 & 0.5330 \\
\cline{2-6}          & {SRC} & \textbf{0.8058 }     & 0.7912 & 0.2060 & 0.6404 \\
    \hline
    \multirow{3}[2]{*}{4000} & {MSE} & \textbf{0.0236} & 0.0894 & 0.7874 & 0.2185 \\
\cline{2-6}          & {RE} & \textbf{0.1607} & 0.2464 & 0.9795 & 0.4407 \\
\cline{2-6}          & {SRC} & \textbf{0.8486}      & 0.7988 & 0.2556 & 0.7095 \\
    \hline
    \multirow{3}[2]{*}{8000} & {MSE} & \textbf{0.0116} & 0.0331 & 0.7459 & 0.2006 \\
\cline{2-6}          & {RE} & \textbf{0.1141} & 0.1338 & 0.9932 & 0.3492 \\
\cline{2-6}          & {SRC} & \textbf{0.8909}      & 0.8905 & 0.2887 & 0.7190 \\
    \hline
    \multirow{3}[2]{*}{10000} & {MSE} & 0.0093 & \textbf{0.0070} & 0.7405 & 0.1668 \\
\cline{2-6}          & {RE} & 0.1012 & \textbf{0.0827} & 0.9947 & 0.2707 \\
\cline{2-6}          & {SRC} &  0.8992     & \textbf{0.9018} & 0.2916 & 0.7321 \\
    \hline
    \hline
    \end{tabular}%
    }
  \label{tab:synthM}%
\end{table}%

\begin{table}[t]
\black 
  \centering
  \caption{The MSE, RE and averaged SRC of $\bm M$ for $K=5$, $L=10$ and $\bm \nu = [\frac{1}{5} \frac{1}{5} \frac{1}{5} \frac{1}{5} 1]^{\top} $}
  \resizebox{0.7\linewidth}{!}{
    \begin{tabular}{c|c|c|c|c|c}
    \hline
    {Graph Size ($N$)} & {Metric} & \texttt{BeQuec} & \texttt{GeoNMF} & \texttt{CD-MVSI} & \texttt{CD-BNMF} \\
    \hline
    \hline
    \multirow{3}[6]{*}{2000} & {MSE} & \textbf{0.4302} & 0.4895 & 0.9063 & 0.8228 \\
\cline{2-6}          & {RE} & \textbf{0.6060} & 0.6633 & 0.8417 & 0.8156 \\
\cline{2-6}          & {SRC} & \textbf{0.5235} & 0.4739 & 0.1572 & 0.1832 \\
    \hline
    \multirow{3}[6]{*}{4000} & {MSE} & \textbf{0.3037} & 0.3275 & 0.8975 & 0.7774 \\
\cline{2-6}          & {RE} & \textbf{0.4420} & 0.4906 & 0.8515 & 0.8044 \\
\cline{2-6}          & {SRC} & \textbf{0.6013} & 0.5678 & 0.1614 & 0.2089 \\
    \hline
    \multirow{3}[6]{*}{8000} & \textbf{MSE} & \textbf{0.1073} & 0.4198 & 0.8969 & 0.8144 \\
\cline{2-6}          & {RE} & \textbf{0.2631} & 0.5661 & 0.8533 & 0.8171 \\
\cline{2-6}          & {SRC} & \textbf{0.7304} & 0.4948 & 0.1540 & 0.1794 \\
    \hline
    \multirow{3}[6]{*}{10000} & \textbf{MSE} & \textbf{0.0863} & 0.5793 & 0.9013 & 0.8174 \\
\cline{2-6}          & {RE} & \textbf{0.2442} & 0.5724 & 0.8464 & 0.8165 \\
\cline{2-6}          & {SRC} & \textbf{0.7535} & 0.4082 & 0.1489 & 0.1736 \\
    \hline
    \hline
    \end{tabular}%
    }
  \label{tab:synthM1}%
\end{table}%

\noindent 
For both synthetic and real data experiments, we compare our algorithm with a number of state-of-the-art mixed membership learning algorithms, namely, {\it geometric intuition-based nonnegative matrix factorization} (\texttt{GeoNMF}) \cite{mao2017mixed}, {\it community detection via minimum volume simplex identification} (\texttt{CD-MVSI}) \cite{huang2019detect}, and \textit{community detection via Bayesian nonnegative matrix factorization} (\texttt{CD-BNMF}) \cite{ioannis2011overlapping}. These baselines are applied onto the queried blocks (e.g., $\A_{\ell,m}$) to estimate $\bm M_\ell$ and $\M_m$, as discussed in {\black Remark~\ref{rmk:baseline}}. The estimated $\M_\ell$'s are aligned (by a nontrivial aligning procedure) to unify the row permutation ambiguity. This step is necessary since the estimated $\bm M_\ell$'s from different blocks are subject to different row permutation ambiguities; see details in {\black Appendix~\ref{sup:baselines}}.

We have two different sets of real data experiments, i.e., crowdclustering of dog breed image data and community detection in co-author networks.
 For real data experiments, we additionally apply {\black another mixed membership learning algorithm, namely, \textit{sequential projection after cleaning} (\texttt{SPACL}) \cite{mao2020estimating},} two versions of the spectral clustering algorithm, namely, the unnormalized and normalized spectral clustering algorithms---denoted as \texttt{SC (unnorm.)} and \texttt{SC (norm.)}, respectively \cite{shi2000normalized} and \texttt{$k$-means}. {\black The \texttt{$k$-means} algorithm employed in these methods are repeated with 20 random initializations.} We also use a convex optimization-based graph clustering algorithm with edge queries (denoted as \texttt{Convex-Opt}) \cite{korlakai2014graph} as a benchmark.
We follow the setup in \cite{korlakai2014graph} and apply \texttt{$k$-means} directly on the adjacency graphs [denoted as \texttt{$k$-means} (graph)].
Note that in the crowdclustering experiment in Sec.~\ref{sec:crowd} (also see Sec.~\ref{sec:examples} for the problem definition), we also use \texttt{$k$-means} on the original image samples [denoted as \texttt{$k$-means} (data)], since the raw data samples are available in this particular problem. The \texttt{$k$-means} (data) method serves as a basic benchmark where no annotator input is used. The source code and the real data used in our experiments are made publicly available (see the footnote in page 2).

\subsection{Synthetic Data}We consider $N$ nodes (where $N\in{\black [2000,10000]}$) and $K=5$ clusters. The membership vectors $\bm m_n$ are drawn from
the Dirichlet distribution with parameter {\black $\bm \nu\in\mathbb{R}^5$}. The symmetric cluster-cluster interaction matrix $\bm B \in \mathbb{R}^{K \times K}$ is generated {\black by randomly sampling its diagonal and non-diagonal entries from uniform distributions between 0.8 to 1 and between 0 to $\eta$, respectively. The diagonal entries are chosen to be relatively large because intra-cluster connections are supposed to be more often. The parameter $\eta \in (0,1]$ decides the level of interactions across different clusters, which is expected to be lower than that of the intra-cluster interactions.
}

\noindent 
{\black {\bf Metrics.} To evaluate the performance of the algorithms, we use a number of standard metrics. First, we use the subspace distance (`Dist') between between ${\sf range}(\widehat{\bm U})$ and the ground-truth ${\sf range}(\bm M^\T)$, i.e.,
$$\text{Dist} = \|\bm O_{\bm M}^\perp \bm Q_{\bm U}\|_2,$$
where $\bm O_{\bm M}^\perp = \bm I-\bm M^{\top} (\bm M\bm M^{\top})^{-1}\bm M$ and $\bm Q_{\bm U}$ is the orthogonal basis of $\widehat{\bm U}^{\top}$ \cite{golub2012matrix}.
To measure the accuracy of the estimated membership, we define the {\it mean squared error} (MSE) as
follows \cite{xiao2015blind,fu2020block}:
\begin{equation*}
        \text{MSE} =
    \underset{\bm \Pi}{\text{min}} \frac{1}{K}\sum_{k=1}^K  \left\lVert \frac{\bm \Pi \bm M(k,:)}{\|\bm \Pi \bm M(k,:)\|_2}-\frac{\widehat{\bm M}(k,:)}{\|\widehat{\bm M}(k,:)\|_2}\right\rVert_{2}^2
\end{equation*}
where $\widehat{\bm M}$ denotes the estimate of $\bm M$ and $\bm \Pi \in \mathbb{R}^{K \times K}$ is a permutation matrix (due to the intrinsic row permutation of the outputs of SSMF algorithms such as \texttt{SPA}).
As an additional reference, we also present the {\it relative error} (RE), i.e.,
$$
\text{RE} = \min_{\bm \Pi}\frac{\|\bm \Pi \bm M-\widehat{\bm M}\|_{\rm F}}{\|\bm M\|_{\rm F}}.
$$
We also use the averaged {\it Spearman's rank correlation coefficient} (SRC).
The averaged SRC is often used in the mixed membership learning literature (e.g., \cite{mao2017mixed,huang2019detect}). The SRC takes values between $-1$ and $1$ and has a higher value if the ranking of the entries in two vectors are more similar---which is desired.

}


\noindent
{\black {\bf Results. }} We first test the identifiability claims under the ideal case (i.e., $\A=\bm P$). The blocks of the adjacency matrix with the leftmost query pattern in Fig.~\ref{fig:query} is used. {\black We choose the number of groups $L=10$ (percentage of queried edges = 27.92\%), $\eta=0.1$ and $\bm \nu = [\frac{1}{5} \frac{1}{5} \frac{1}{5} \frac{1}{5} \frac{1}{5}]^{\top} $.} 
Table~\ref{tab:dist} shows the subspace estimation accuracy of our method measured using `Dist'.
The results are averaged from 20 random trials. One can see that the proposed method estimates the subspace of the membership matrix $\bm M$ very accurately {\black in the ideal case}, which verifies our subspace identifiability analysis in {Theorem}~\ref{prop:subspace_estim}. {\black The right column of Table~\ref{tab:dist} presents the `Dist' measure in the binary observation case. It shows that even under the binary ``quantization'' of $\bm P$, the proposed method can still recover $\bm U$ to a good accuracy, if $N$ is sufficiently large. This is consistent with Theorem~\ref{prop:subspace_estim_noisy}.}

Next, we consider the rightmost pattern in Fig \ref{fig:query} to evaluate the proposed algorithm and the baselines in the binary observation case. 
{\black We fix $L=10$ and $\eta=0.1$.
We test two different cases of $\bm \nu$, i.e., $\bm \nu=[\frac{1}{5} \frac{1}{5} \frac{1}{5} \frac{1}{5} \frac{1}{5}]^{\top}$ and $\bm \nu=[\frac{1}{5} \frac{1}{5} \frac{1}{5} \frac{1}{5} 1]^{\top}$. The results, which are averaged from 20 random trials, are presented in Tables~\ref{tab:synthM} and \ref{tab:synthM1}, respectively. 
The former case of $\bm \nu$ is easier, since many nodes' membership vectors tend to reside on the boundaries/corners of the probability simplex (cf. Fig.~\ref{fig:alpha11})---leading to the existence of many anchor nodes.
Hence, every block $\bm M_\ell$ for $\ell\in[L]$ may satisfy the ANC, and \texttt{BeQuec} and \texttt{GeoNMF} both work reasonably well (cf. discussions in Remark~\ref{rmk:baseline}).
The latter case is more challenging due to $\nu_5=1$, which means that there is a cluster that has (much) fewer anchor nodes. This means that many $\bm M_\ell$'s do not satisfy the ANC. Consequently, \texttt{BeQuec} outperforms the baselines, since our method does not need ANC to hold on each block $\bm M_\ell$. Particularly, in the latter case, the best baseline \texttt{GeoNMF} outputs high MSE = 0.579 even when $N=10,000$, while \texttt{BeQuec}'s performance is considerably better under all three evaluation metrics (with MSE = 0.086).
}

{\black More synthetic data experiments with different values of $\eta$ and $\bm \nu$ can be found in Sec. \ref{supp:more_exp} of the supplementary material.}

 \begin{figure}[t]
	\centering
	\includegraphics[width=0.4\linewidth]{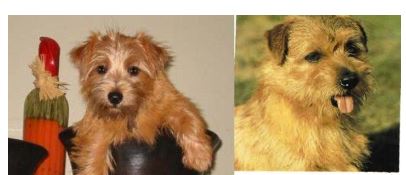}
	\includegraphics[width=0.47\linewidth]{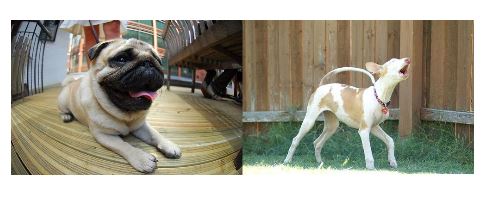}
		\includegraphics[width=0.43\linewidth]{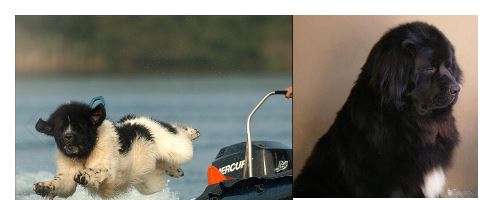}
			\includegraphics[width=0.38\linewidth]{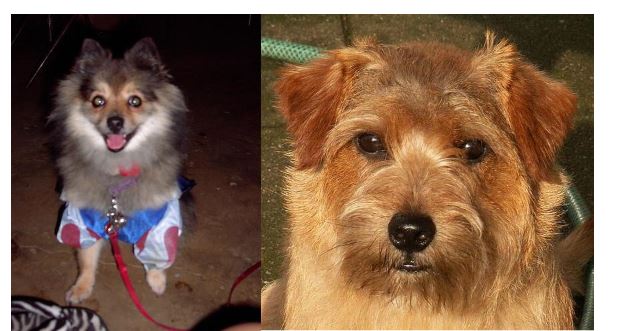}
	\caption{{ Some examples of the pairs presented to the AMT annotators in the crowdclustering experiment. The image pairs have true similarity labels 1,0,0,1, respectively, in clockwise order. The annotators marked 1,0,1,0, respectively.}}
	\label{fig:dogs}
\end{figure}

\begin{table}[t]
  \centering
    
    
    \color{black}
  \caption{ Graph clustering performance on the dog breed dataset with AMT similarity annotations. $N=900, K=5, L=15$. Annotation proportion = 19.02\%.}
   \resizebox{0.6\linewidth}{!}{\small
    \begin{tabular}{c|c|c|c}
    \hline
    {Algorithms} & \multicolumn{1}{c|}{{ ACC (\%)}}& \multicolumn{1}{c|}{NMI}&
    \multicolumn{1}{c}{{Time (s)}}\\
    \hline
    \hline
    \texttt{BeQuec} (proposed)  & \textbf{99.22} & \textbf{0.972} &\textbf{0.21}\\
    \hline
    \texttt{\texttt{GeoNMF}} & 96.88 & 0.929 &2.80\\
    \hline
    \texttt{\texttt{\black SPACL}} & \black 97.11 & \black 0.936 & \black \textbf{0.22}\\
    \hline
    \texttt{CD-MVSI} & 38.02 & 0.199 & 1.30\\
   \hline 
   {\texttt{ \texttt{CD-BNMF}}}& 81.89 & 0.683 & 0.28\\
    \hline
    \texttt{SC (unnorm.)} & \black 22.90  & \black 0.032 & \black 1.40 \\
    \hline
    \texttt{SC (norm.)} & \black 94.52 & \black 0.926 & \black 1.05 \\
    \hline
    \texttt{$k$-means} (graph) & \black 65.73 & \black 0.507 & \black 1.14 \\
    \hline
        \texttt{$k$-means} (data) & \black 23.56 &  0.052 &  2034.50\\
    \hline
   \texttt{ Convex-Opt}& \black \textbf{99.78 } & \black \textbf{0.991}& \black 354.22\\

    \hline
    \hline
    \end{tabular}\label{tab:AMT}
    }
\end{table}
 
 \begin{table}[t]
     \centering
      \color{black}
 \caption{ Graph clustering performance on the dog breed dataset for different annotation error rates. $N=900, K=5, L=15$. Annotation proportion = 19.02\%.}     
 \resizebox{0.7\linewidth}{!}{\small
   \begin{tabular}{c|c|c|c|c}
    \hline
    \multicolumn{1}{c|}{Algorithms} & \multicolumn{2}{c|}{annotation error=15\%} & \multicolumn{2}{c}{annotation error=20\%} \\
    \hline
    \hline
          & \multicolumn{1}{c|}{ACC(\%)} & \multicolumn{1}{c|}{NMI} & \multicolumn{1}{c|}{ACC(\%)} & \multicolumn{1}{c}{NMI} \\
    \hline
    \texttt{BeQuec} (proposed) & \textbf{95.83} & \textbf{0.941} & \textbf{87.20} & \textbf{0.740 }\\
    \hline
    \texttt{\texttt{GeoNMF}} & 85.33 & 0.688 & 79.83 & 0.629 \\
    \hline
    \black \texttt{SPACL} & \black 84.10  & \black 0.758 & \black \textbf{83.71} & \black \textbf{0.757} \\
    \hline
    \texttt{ CD-MVSI} & 42.28 & 0.194 & 38.57 & 0.165 \\
    \hline
    {\texttt{ CD-BNMF}} & 67.64 & 0.609 & 66.89 & 0.574 \\
    \hline
    \texttt{SC (unnorm.)} & \black 23.50 & \black 0.033 & \black 23.68 & \black 0.036 \\
    \hline
    \texttt{SC (norm.)} & \black \textbf{87.27} & \black \textbf{0.818} & \black {80.27} & \black {0.732} \\
    \hline
    \texttt{$k$-means} (graph) & \black 64.96 & \black 0.511 & \black 62.48 & \black 0.475 \\
    \hline
   \texttt{ Convex-Opt}& \black 70.98 & \black 0.853 & \black 50.57 & \black 0.582 \\
    \hline
    \hline
    \end{tabular}%
   }
  \label{tab:amt_exp2}%
    
  %
\end{table}%

\subsection{Crowdclustering using AMT Annotations} \label{sec:crowd}
We consider the task of clustering a set of images of dogs into different breeds with the assistance of self-registered annotators from a crowdsourcing platform, i.e., the AMT platform. 

\noindent
{\bf AMT Data.} We upload $N=900$ images of $K=5$ different breeds of dogs from the Standford Dogs Dataset \cite{khosla2011novel}. The chosen breeds are Pug (180 images), Ibizan Hound (180 images), Pomeranian (180 images), Norfolk Terrier (172 images) and Newfoundland (188 images). 
We followed the setup in Sec.~\ref{sec:examples} to design the experiment.
Specifically, we showed each annotator a pair of dog images and ask a simple question (query): ``Do this pair of dogs belong to the same category?". If the annotator marks ``Yes", then we fill the corresponding entry of $\bm A$ with ``1", and otherwise ``0". In this case, if one hopes to fill all the entries of $\bm A$, there are approximately $ 4 \times 10^5$ distinct pairs of images to be queried, which is costly. We selected image pairs following the block-diagonal pattern (i.e., the rightmost pattern in Fig.~\ref{fig:query}) by fixing $L=15$. As a result, only 19.02\% of the pairs (76,950) were queried\footnote{As a result, $\$81.60$ was paid for our annotation task to AMT, but annotating all the pairs of images for this experiment would have cost $\$429.02$.}. Fig.~\ref{fig:dogs} shows some pairs of images queried in this experiment.
We obtained the responses from 642 AMT annotators. Each annotator answered (non-overlapping) 120 queries on average. The error rate of the answers given by the annotators was 9.90$\%$.

\noindent
{\bf Metric.} To quantify the performance, we use standard metrics for evaluating the clustering algorithms.
Specifically, since each image is only labeled with a single cluster label, we round the learned membership vectors to the closest unit vectors. This way, standard metrics, i.e., {\it clustering accuracy} (ACC) and {\it normalized mutual information} (NMI) can be used for performance characterization (see definitions of ACC and NMI in \cite{canyi2013correntropy}). {ACC ranges from 0\% to 100\% where $100\%$ means perfect clustering. NMI takes values in between 0 and 1 with 1 being the best performance.}

\noindent
{\bf Results.}
Table~\ref{tab:AMT} summarizes the results output by the algorithms under test, {\black where the first and the second best results are highlighted.} {\black  One can see that both the proposed method \texttt{BeQuec} and \texttt{Convex-Opt} output surprisingly high ACC and NMI results. 
The baselines \texttt{GeoNMF} and \texttt{SPACL} also yield reasonable ACCs/NMIs.
It is somewhat expected both SBM and MMSB based methods could work, since the nodes in this case admit single membership.
Nonetheless, \texttt{BeQuec} and \texttt{Convex-Opt}'s better clustering performance over other baselines may be due to various reasons. For example, the fact that \texttt{BeQuec} works under less stringent conditions may make it more resilient to modeling mismatches.
The all-at-once optimization strategy of \texttt{Convex-Opt} may help avoid error propagation as in greedy methods used in other baselines.
In terms of runtime, notably, \texttt{BeQuec} runs much faster (approximately 1500 times faster than \texttt{Convex-Opt}, which strikes a good balance between accuracy and efficiency for this particular task.}

{\black Another interesting observation is that} \texttt{BeQuec} largely outperforms directly applying clustering methods to the data samples without annotators' assistance [cf. \texttt{$k$-means} (data)]. {Note that \texttt{$k$-means} (data) sees all the data and their features (pixels in this case). It does not work well perhaps because the dogs across some breeds are naturally confusing. However, with annotating less than 20\% of the data pairs' similarity (without pointing out the exact breeds), the clustering error approaches zero. This experiment may have articulated the value of this simple annotation strategy.}


In order to better understand the impact of annotation errors on crowdclustering, we also manually ``inject'' some errors to the annotations.
Specifically, we randomly flip a certain percentage of the correct annotations given by the AMT annotators, and observe how this affects the performance of the algorithms. 

Table~\ref{tab:amt_exp2} presents the clustering performance obtained from two different error rates. The results are averaged over 20 trials, where in each trial the flipped annotations are randomly chosen.  One can observe that the proposed \texttt{BeQuec} method offers the most competitive clustering performance. {When the annotation error rate increases from 9.9\% (Table~\ref{tab:AMT}) to 15\% and 20\%, the performance degradation of \texttt{BeQuec} is much less obvious compared to the baselines. 
In particular, one can see that NMI of \texttt{BeQuec} decreases from 0.972 to 0.941 when the annotation error rate is changed from 9.9\% to 15\%. However, \texttt{GeoNMF}'s NMI changes from 0.929 to 0.688, which is a much more drastic drop.
This shows \texttt{BeQuec}'s robustness to annotation errors.}

\begin{table}[t]
  \centering
  \caption{Details of the co-author network datasets used in community detection experiments }
  \resizebox{0.25\linewidth}{!}{
    \begin{tabular}{c|c|c}
    \hline
    \multicolumn{1}{c|}{Dataset} & \multicolumn{1}{c|}{$N$} & \multicolumn{1}{c}{$K$} \\
    \hline
    \hline
       $\text{DBLP}_{-1}$   & 15,075 & 13 \\
    \hline
       $\text{DBLP}_{-2}$   & 18,106 & 16 \\
    \hline
        $\text{DBLP}_{-3}$  & 18,646 & 16 \\
    \hline
       $\text{DBLP}_{-4}$   & 18,195 & 16 \\
    \hline
       $\text{DBLP}_{-5}$   & 14,982 & 15 \\
    \hline
       MAG1   & 37,680 & 3 \\
    \hline
       MAG2   & 19,457 & 3 \\
    \hline
    \hline
    \end{tabular}%
    }
  \label{tab:dataset}%
\end{table}%

\begin{table}[t]
\centering
  \color{black}
  \caption{Averaged SRC and runtime in seconds. { $L=10$ (percentage of queried edges = 27.92\%). } }
\resizebox{1\linewidth}{!}{
    \begin{tabular}{c|c|c|c|c|c|c|c|c|c|c}
    \hline
    Datasets & \multicolumn{2}{c|}{\makecell{\texttt{BeQuec}\\ (proposed)}} & \multicolumn{2}{c|}{{\texttt{GeoNMF}}} &  \multicolumn{2}{c|}{{\texttt{\black SPACL}}} & \multicolumn{2}{c|}{\texttt{ CD-MVSI}} & \multicolumn{2}{c}{\texttt{ CD-BNMF}} \\
    \hline
    \hline
          & SRC   & Time(s.) & SRC   & Time(s.) & {\black SRC}   & {\black Time(s.)} & SRC   & Time(s.) &SRC   & Time(s.)\\
    \hline
    $\text{DBLP}_{-1}$  & \textbf{0.177} & 0.47  & \textbf{0.075} & 15.87 & \black 0.072 & \black 0.73  & 0.074 & 0.95 &0.058 &5.61\\
    \hline
    $\text{DBLP}_{-2}$  & \textbf{0.166} & 0.38  & 0.089 & 11.13 & \black \textbf{0.103} & \black 1.02  & 0.055 & 0.76& 0.070 &8.20 \\
    \hline
    $\text{DBLP}_{-3}$  & \textbf{0.166} & 0.46  & 0.076 & 17.18 & \black \textbf{0.094} & \black 1.01& 0.064   & 0.98 & 0.060 &8.50\\
    \hline
    $\text{DBLP}_{-4}$  & \textbf{0.191} & 0.48  & \textbf{0.098} & 15.76 & \black 0.095 & \black 0.95  & 0.070 & 0.95& 0.079 &8.50 \\
    \hline
    $\text{DBLP}_{-5}$  & \textbf{0.195} & 0.41  & 0.084 & 15.65 & \black \textbf{0.156} & \black 0.67 & 0.073 & 0.88& 0.074 &5.85 \\
    \hline
    MAG1  & \textbf{0.125} & 0.26  & 0.122 & 1.79   & \black \textbf{0.123} & \black 0.97  & 0.089 & 0.59 & 0.069 &37.29\\
    \hline
    MAG2   & \textbf{0.441} & 0.23  & 0.240 & 4.66 & \black \textbf{0.267} & \black 0.50   & 0.249 & 0.53 & 0.122 &9.92\\
    \hline
    \hline
    \end{tabular}%

    \label{tab:src}
   }
  \label{tab:addlabel}%
\end{table}%

\begin{table}[t]
  \centering
  \color{black}
  \caption{Clustering accuracy (\%) of MAG2 under various $L$ ({percentage of queried edges (\%)}). $N = 19457$, $K=3$.}
   \resizebox{0.8\linewidth}{!}{
    \begin{tabular}{c|c|c|c|c|c}
    \hline
    {Alorithms} & \makecell{\small $L=10$\\ (27.92\%)}& \makecell{\small $L=25$\\(11.58\%) }  & \makecell{\small $L=50$\\(5.82\%)}  & \makecell{\small $L=75$\\(3.86\%)}  & \makecell{\small $L=100$\\(2.87\%)}  \\
    \hline
    \hline
    {\texttt{BeQuec} (proposed)} & \textbf{78.70} & \textbf{77.19} & \textbf{67.81} & \textbf{61.85} & \textbf{56.98} \\
    \hline
    {\texttt{GeoNMF}} & 58.16 & 57.87 & 56.88 & 52.68 & 52.33 \\
    \hline
    \black \texttt{SPACL} & \black 57.49 & \black 58.17 & \black 55.25 & \black 54.22 & \black 53.07 \\
  \hline
    {\texttt{ CD-MVSI}} & 53.45 & 21.82 & 14.57 & 13.53 & 11.71 \\
    \hline
   {{\texttt{CD-BNMF}}} & 51.19 & 49.35 & 49.37 & 48.99 & 48.97 \\
    \hline
    {\texttt{SC (unnorm.)}} & \black 56.97 & \black 56.49 & \black 56.38 & \black 55.88 & \black 55.51 \\
    \hline
    {\texttt{SC (norm.)}} & \black \textbf{64.24} & \black \textbf{62.15} & \black \textbf{57.15} & \black \textbf{56.45} & \black \textbf{55.80} \\
    \hline
    {\texttt{$k$-means}} & \black 54.31 & \black 53.47 & \black 53.87 & \black 53.15 & \black 51.25 \\
    \hline
    \hline
    \end{tabular}%
    }
  \label{tab:mag2_acc}%
\end{table}%

\subsection{ Co-author Network Community Detection}
We test the edge query-based graph clustering algorithms using the co-authorship network data from MAG \cite{sinha2015AnOverview} and DBLP \cite{ley2002DBLP}. 
In all the datasets, node-node connection (co-authorship) is represented by a binary value (``1'' means that the two nodes have co-authored at least one paper and ``0'' means otherwise).
For our experiments, we use the version of these networks published by the authors of \cite{mao2017mixed}. Each network is provided with ground-truth mixed
membership of the nodes. In these networks, nodes represent the authors of the research papers from different fields of study. Some nodes exhibit multiple memberships, since some authors publish in different fields. 

\noindent
{\bf MAG Data.}
In our experiment, we use a subset of nodes that have a degree of at least 5.
Consequently, two co-author networks, namely, MAG1 and MAG2, have 37,680 and 19,457 authors, respectively. In MAG1 and MAG2, all the authors are from 3 different fields of study, i.e., $K=3$.

\noindent
{\bf DBLP Data.}
The 5 DBLP datasets (denoted as DBLP1, ..., DBLP5) used in \cite{mao2017mixed} have 6176, 3145, 2605, 3056, and 6269
nodes, respectively. In addition, the number of clusters present in the DBLP datasets are 6, 3, 3, 3, and 4, respectively. 
To make the clustering problem more challenging, we combine the smaller original DBLP networks into larger-size networks following the setting in \cite{huang2019detect}. We use the notation $\text{DBLP}_{-i}$ to denote the network where all the 5 DBLP networks except the $i$-th one are combined. This way, different $\text{DBLP}_{-i}$'s can be produced for testing the algorithms. In addition, each $\text{DBLP}_{-i}$, for $i=1,\ldots,5$, contains 13-16 clusters (see Table~\ref{tab:dataset} for details), presenting more challenging scenarios.

\noindent
{\bf Metric.}
The methods are evaluated using the averaged {\black SRC as in the literature \cite{mao2017mixed,huang2019detect}}.
We let all the algorithms to access only part of the network under the block-diagonal query pattern in Fig.~\ref{fig:query}. We randomize the node order in each of the 20 trials and present the averaged performance.

\noindent
{\bf Results.}
We first test the mixed membership learning performance by setting $L=10$ for all the MAG and DBLP datasets under test. The percentage of edges queried in this setting is 27.92\%.
From Table~\ref{tab:src}, one can see that the proposed \texttt{BeQuec} algorithm consistently outperforms the baselines in all the graphs under test. In addition, \texttt{BeQuec}'s runtime performance is also the most competitive.

Table~\ref{tab:mag2_acc} uses MAG2 dataset to show the performance of \texttt{BeQuec} under various $L$'s. We also apply the clustering algorithms used in crowdclustering to benchmark \texttt{BeQuec}'s performance.
The ACC metric is used in this table since the clustering algorithms do not output mixed membership. One can see that the performance of all algorithms decreases along with $L$'s increase---which is consistent with our analysis in Theorem~\ref{prop:subspace_estim_noisy}.
{For each $L$ used, the corresponding queried percentage of edges is listed in the first row.}
Notably, when $L=50$, i.e, only 5.82\% of $\bm A$ is observed, the proposed method still outputs a reasonable clustering accuracy. This demonstrates a promising balance between the query sample complexity and the graph clustering accuracy.







\section{Conclusion}
In this work, we proposed a graph query scheme that enables provable graph clustering with partially observed edges. Unlike previous works relying on random graph edge query and computationally heavy convex programming, our method
features a lightweight algorithm and works with systematic edge query patterns that are arguably more realistic in some applications. Our method also learns mixed membership of nodes from overlapping clusters, which improves upon existing provable graph query-based methods that work with single membership and disjoint clusters. We tested the proposed \texttt{BeQuec} algorithm on two real-world network data analytics problems, namely, crowdclustering and network community detection. \texttt{BeQuec} offers promising results in both experiments. 

{\black A potential future direction is to reduce the number of queries by allowing $N/L <K$. This may be achieved by using third-order graph statistics from the sampled blocks and a tensor decomposition framework. Using tensor-based approaches may likely present a challenging optimization problem without polynomial-time algorithms (as opposed to the matrix based framework as in this work). Studying its query complexity and computational cost trade-off presents an interesting research topic.}  


\bibliographystyle{IEEEtran}




\appendix

\section{Proof of Theorem~\ref{prop:subspace_estim}}\label{supp:subspace_estim}

Algorithm \ref{algo:proposed} aims to output $\widehat{\bm U}$ such that ${\sf range}(\widehat{\bm U}) = {\sf range}([\bm M_1,\ldots,\bm M_{L}]^\T)$. It implies that 
$
\widehat{\bm U} = [\bm U_1^{\top},\ldots,\bm U_L^{\top}]^{\top}= [\bm M_{1},\bm M_{2},\ldots, \bm M_{L}]^{\top}\bm G= \bm M^{\top} \bm G$,
for a certain nonsingular $\bm G \in \mathbb{R}^{K \times K}$. 
To show this, we start by considering the below four blocks for any $r \in [L-2]$ :
\begin{subequations}\label{eq:Ps1}
\begin{align} 
\bm P_{\ell_r,m_r}&= \M_{\ell_r}^{\top}{\bm B}\M_{m_{r}},\\
\bm P_{\ell_{r+1},m_r}&= \M_{\ell_{r+1}}^{\top}{\bm B}\M_{m_r}, \\
\bm P_{\ell_{r+1},m_{r+1}}&= \M_{\ell_{r+1}}^{\top}{\bm B}\M_{m_{r+1}},\\  
\bm P_{\ell_{r+2},m_{r+1}}&= \M_{\ell_{r+2}}^{\top}{\bm B}\M_{m_{r+1}}.
\end{align}
\end{subequations}
Algorithm \ref{algo:proposed} defines the following matrices in $r$th and $(r+1)$th iterations, respectively (recall that we have assumed $\bm A_{\ell,m}=\bm P_{\ell,m} = \bm M_{\ell}^{\top}\bm B \bm M_{m}$ for any $\ell,m \in [L]$ in the ideal case):
$$
\bm C_r := [\bm P_{\ell_r,m_r}^{\top}, \bm P_{\ell_{r+1},m_r}^{\top}]^{\top},~\bm C_{r+1} := [\bm P_{\ell_{r+1},m_{r+1}}^{\T} \bm P_{\ell_{r+2},m_{r+1}}^{\T}]^{\T}.$$
The top-$K$ SVDs of $\bm C_r$ and $\bm C_{r+1}$ can be represented as follows:
\begin{subequations}\label{eq:CD1}
\begin{align} 
\bm C_r &= [{\bm U}_{\ell_r}^{\top}, {\bm U}_{\ell_{r+1}}^{\top}]^{\top}\bm \Sigma_r{\bm V_{m_r}^{\top}},\\
\bm C_{r+1} &= [\overline{\bm U}_{\ell_{r+1}}^{\top}, \overline{\bm U}_{\ell_{r+2}}^{\top}]^{\top}\bm \Sigma_{r+1}{\bm V}_{m_{r+1}}^{\top}.
\end{align}
\end{subequations}
Combining \eqref{eq:Ps1} and \eqref{eq:CD1}, and using the assumption that ${\rm rank}(\bm M_{\ell})={\rm rank}(\bm B)=K$ and $K\leq |\mathcal{S}_{\ell}|$ for every $\ell$, the following equations hold:
\begin{subequations}\label{eq:UVs}
\begin{align} 
{\bm U}_{\ell_r} = \bm M_{\ell_r}^{\top} \bm G_r, \quad &{\bm U}_{\ell_{r+1}} = \bm M_{\ell_{r+1}}^{\top}\bm G_r, \label{eq:UV1}\\
\overline{\bm U}_{\ell_{r+1}} = \bm M_{\ell_{r+1}}^{\top} \overline{\bm G}_{r+1}, \quad &\overline{\bm U}_{\ell_{r+2}} = \bm M_{\ell_{r+2}}^{\top}\overline{\bm G}_{r+1}, \label{eq:UV3}
\end{align}
\end{subequations}
where $\bm G_r\in \mathbb{R}^{K\times K}$ and $\overline{\bm G}_{r+1} \in \mathbb{R}^{K \times K}$ are certain nonsingular matrices. The equations in \eqref{eq:UVs} imply that  ${\bm U}_{\ell_r}$, ${\bm U}_{\ell_{r+1}}$ and $\overline{\bm U}_{\ell_{r+2}}$ are the bases of ${\sf range}(\bm M_{\ell_r}^\T)$, ${\sf range}(\bm M_{\ell_{r+1}}^\T)$ and ${\sf range}(\bm M_{\ell_{r+2}}^\T)$, respectively.  Since $\bm G_r = \overline{\bm G}_{r+1}$ does not generally hold, the basis $\overline{\bm U}_{\ell_{r+2}}$ cannot be directly combined with the bases ${\bm U}_{\ell_r}$ and ${\bm U}_{\ell_{r+1}}$ to obtain
$ {\sf range}([\bm M_{\ell_r},\bm M_{\ell_{r+1}},\bm M_{\ell_{r+2}}]^\T)$.
Therefore, we are interested in obtaining the basis defined as below:
\begin{align} \label{eq:U}
{\bm U}_{\ell_{r+2}} := \bm M_{\ell_{r+2}}^{\top} \bm G_r.
\end{align}
In order to identify ${\bm U}_{\ell_{r+2}}$ as defined in \eqref{eq:U}, we use the second equality in \eqref{eq:UV1} and the first equality in \eqref{eq:UV3} to construct the following:
\begin{align} \label{eq:phis}
 \overline{\bm U}_{\ell_{r+1}}^{\dagger} {\bm U}_{\ell_{r+1}} = \overline{\bm G}_{r+1}^{-1}\bm G_r.
\end{align}
Combining the second equality in \eqref{eq:UV3} with \eqref{eq:phis}, we have
\begin{align}
\overline{\bm U}_{\ell_{r+2}} \overline{\bm U}_{\ell_{r+1}}^{\dagger} {\bm U}_{\ell_{r+1}}  &= \left(\bm M_{\ell_{r+2}}^{\top} \overline{\bm G}_{r+1}\right)\left( \overline{\bm G}_{r+1}^{-1}\bm G_r\right) = \bm M_{\ell_{r+2}}^{\top}  \bm G_r . \nonumber 
\end{align}
Hence, the following holds:
\begin{align} \label{eq:U2}
{\bm U}_{\ell_{r+2}} = \overline{\bm U}_{\ell_{r+2}} \overline{\bm U}_{\ell_{r+1}}^{\dagger} {\bm U}_{\ell_{r+1}}  = \bm M_{\ell_{r+2}}^{\top} \bm G_r,
\end{align}
where $\overline{\bm U}_{\ell_{r+2}}$ and $\overline{\bm U}_{\ell_{r+1}}$ are obtained from the top-$K$ SVD of $\bm C_{r+1}$ and ${\bm U}_{\ell_{r+1}}$ is obtained from the top-$K$ SVD of $\bm C_r$.

Similarly, ${\bm U}_{\ell_{r-1}}$ can be identified using the top-$K$ SVDs of $\bm C_r$ and $\bm C_{r-1}$ by the following:
\begin{align} \label{eq:U2_1}
{\bm U}_{\ell_{r-1}} = \overline{\bm U}_{\ell_{r-1}} \overline{\bm U}_{\ell_{r}}^{\dagger} {\bm U}_{\ell_{r}}  = \bm M_{\ell_{r-1}}^{\top} \bm G_r,
\end{align}
where ${\bm U}_{\ell_{r}}$ is obtained from the top-$K$ SVD of $\bm C_r$ and $\overline{\bm U}_{\ell_{r-1}}$ and $\overline{\bm U}_{\ell_{r}}$ are obtained from the top-$K$ SVD of $\bm C_{r-1} := [\bm P_{\ell_{r-1},m_{r-1}}^{\T}, \bm P_{\ell_{r},m_{r-1}}^{\T}]^{\T}$.

Algorithm~\ref{algo:proposed} first applies top-$K$ SVD for $\bm C_{T}$ where $T= \left\lfloor L/2\right\rfloor$  and the bases ${\bm U}_{\ell_T}$ and ${\bm U}_{\ell_{T+1}}$ are identified (lines \ref{algoline:Ct}-\ref{algoline:Ut}).  By letting $\bm G_T=\bm G$, we thus obtain
\begin{align} \label{eq:UT}
{\bm U}_{\ell_T} = \bm M_{\ell_{T}}^{\top} \bm G\quad \text{and} \quad{\bm U}_{\ell_{T+1}} = \bm M_{\ell_{T+1}}^{\top}\bm G.
\end{align}

Next, we use induction to show how the iterations in Algorithm \ref{algo:proposed} (lines \ref{algoline:iterative1_begin}-\ref{algoline:iterative1_end})  identify $\bm U_{\ell_{T+2}},\ldots,\bm U_L$. 
  The induction hypothesis is that before any iteration at $r = T+1,T+2,\dots,L-1$,  the basis $\bm U_{\ell_{r}}$ is identified from the previous iteration such that $\bm U_{\ell_{r}} = \bm M_{\ell_{r}}^{\top}\bm G$.
  Note that the induction hypothesis holds true for $r=T+1$ from lines \ref{algoline:Ct}-\ref{algoline:Ut} of Algorithm \ref{algo:proposed} by performing top-$K$ SVD of $\bm C_T$.
  From \eqref{eq:U2}, we have seen that given $\bm U_{\ell_r}$, one can identify $\bm U_{\ell_{r+1}}$ using the matrices resulted from the top-$K$ SVD of $\bm C_{r}$.
  Therefore, by induction, we can establish that Algorithm \ref{algo:proposed} identifies
    \begin{align}\label{eq:Uis2}
  \bm U_{\ell_{r}} = \bm M_{\ell_{r}}^{\top}\bm G, ~\text{for every }~r \in \{T+2,\dots, L\}. \end{align}
  
  Similarly, using induction via employing the results in  \eqref{eq:U2_1} and \eqref{eq:UT}, we can establish that lines \ref{algoline:iterative2_begin}-\ref{algoline:iterative2_end} of Algorithm \ref{algo:proposed} identifies
    \begin{align}\label{eq:Uis4}
  \bm U_{\ell_{r}} = \bm M_{\ell_{r}}^{\top}\bm G, ~\text{for every }~r \in \{1,\dots, T-1\}. \end{align}

 Combining \eqref{eq:UT}, \eqref{eq:Uis2},  and \eqref{eq:Uis4}, one can see that Algorithm \ref{algo:proposed} identifies $\widehat{\bm U} = [\bm U_1^{\top}, \bm U_2^{\top},~~\ldots ~~, \bm U_L^{\top}]^{\top}$ such that 
$ {\sf range}(\widehat{\bm U})= {\sf range}([\bm M_1,\bm M_{2},\ldots,\bm M_{L}]^\T).$

{\color{black}
\section{Proof of Proposition \ref{lem:Lm} } \label{app:dirichlet}

\subsection{Preliminaries of Dirichlet Distribution}
Let $\mathcal{P}^K = \{ \bm z \in \mathbb{R}^K | \bm{z}^\top \bm{1} = 1 , \bm z \ge 0\}$ denote the $(K-1)$-dimensional probability simplex.
The Dirichlet distribution with parameter $\bm \nu = [\nu_1,\dots,\nu_K]^{\top}$ has the following probability density function (PDF):
\begin{align} \label{eq:pdfdirich}
    f(\bm z_{-k}; \bm \nu) = \frac{\Gamma (\sum_{i=1}^K \nu_i)}{\prod_{i=1}^K \Gamma(\nu_i)}  (1-\sum_{i\neq k}z_i)^{\nu_k-1}\prod_{i\neq k}z_i^{\nu_i-1},
\end{align}
for any $k \in [K]$ where $\bm z_{-k} = [z_1,\dots,z_{k-1},z_{k+1},\dots,z_K]^{\top}$ and $ [z_1,\dots,z_K]^{\top} \in \mathcal{P}^K$ \cite{frigyik2010introduction}. The choice of $k$ in \eqref{eq:pdfdirich} is arbitrary since $\sum_{i \neq k} z_i = 1-z_k$ holds for any $k$.  
The cumulative distribution function (CDF) of the Dirichlet distribution is given by 
\begin{align} 
F(\bm z_{-k}; \bm \nu) &= 
\int_{0}^{z_1}\dots\int_{0}^{z_{k-1}}\int_{0}^{z_{k+1}}\dots \int_{0}^{z_K} f(\bm z_{-k}; \bm \nu) d\bm z_{-k} \nonumber\\
&=\frac{\Gamma (\sum_{i=1}^K \nu_i)}{\prod_{i=1}^K \Gamma(\nu_i)}L_k(\bm m,\bm \nu)\prod_{i \neq k}\frac{z_i^{\nu_i}}{\nu_i} , \label{eq:cdf1}
\end{align}
where the function $L_k(\bm m,\bm \nu)$ is known as the Lauricella function of the first kind and $L_k(\bm m,\bm \nu)> 1$ for any $k$ \cite{gouda2010numerical}. The CDF $F(\bm z_{-k}; \bm \nu)$ measures the cumulative probability starting from the $k$-th vertex of the probability simplex till each coordinate $i \neq k$ reaches $z_i$.



\subsection{Proof {of the Proposition}}

 According to Assumption \ref{as:as1}, $\bm m_n \in \mathcal{P}^K$ holds for all $n \in [N]$ and the vectors $\bm m_n$ are sampled randomly from a Dirichlet distribution with parameter $\bm \nu$. Using Assumption \ref{as:as1}, we aim to analyze the conditions under which  the matrix $\bm M \in [\bm m_1,\dots,\bm m_N]$ satisfies $\varepsilon$-separability (see Definition \ref{def:sep}), i.e., there exists $\bm \varLambda =\{ q_1,\ldots,q_K \}$ such that for $k=1,\ldots,K$,
\begin{align} \label{eq:condMvec}
\|\bm m_{q_k}-\bm e_k\|_2\leq \varepsilon, ~ \bm m_{q_k} \in \mathcal{P}^K.
\end{align} 
The criterion in \eqref{eq:condMvec} {means that there exist} $\bm m_{q_k} \in \mathcal{P}_{k}(\varepsilon)$ for $k=1,\dots,K$ where
\begin{align} \label{eq:calPk}
\mathcal{P}_{k}(\varepsilon) = \{\bm z \in \mathcal{P}^K | \|\bm z-\bm e_k\|_2 \le \varepsilon\}.
\end{align}


Next, we proceed to check under what conditions there exists at least one $\bm m_{n}$ among $\bm m_1,\dots, \bm m_N$ such that $\bm m_{n} \in \mathcal{P}_{k}(\varepsilon)$. For this, let us define two random variables as below:
\begin{align} 
E_{n,k}&=
\begin{cases}
1, & \text{if~}\ \bm m_n \in \mathcal{P}_{k}(\varepsilon) \\
0, & \text{otherwise}.\label{eq:Endef}
\end{cases}\\
H_k &= \sum_{n=1}^N E_{n,k}.
\end{align}

In order to characterize the probability that there exists at least one $\bm m_n$ among $\bm m_1,\dots, \bm m_N$ such that $\bm m_{n} \in \mathcal{P}_{k}(\varepsilon)$, we need to characterize ${\sf Pr}\left(H_k > 0\right)$. To achieve this, we invoke Chernoff-Hoeffding
bound \cite{wainwright2019high}:

\begin{theorem} \label{lem:chernoff}
	Let $E_{1,k},\dots,E_{N,k}$ be independent bounded random variables such
	that $E_{n,k} \in [a_{n,k},b_{n,k}]$. Then for any $t > 0$,
	\begin{align*}
	{\sf Pr}\left(\sum_{n=1}^N E_{n,k} - \sum_{n=1}^N\mathbb{E}[E_{n,k}] \le -t\right) \le e^{\frac{-2t^2}{\sum_{n=1}^N (b_{n,k}-a_{n,k})^2}}.
	\end{align*}
\end{theorem}
By letting $t = \sum_{n=1}^N\mathbb{E}[E_{n,k}]$ and using the fact that $E_{n,k} \in [0,1]$, we get the below from Theorem \ref{lem:chernoff}:
\begin{align} \label{eq:chernoff1}
{\sf Pr}\left(H_k > 0\right) &\ge 1- e^{-\frac{2\left(\sum_{n=1}^N\mathbb{E}[E_{n,k}]\right)^2}{N}}.
\end{align}
Next, we need to characterize the term $\mathbb{E}[E_{n,k}]$ which appears in the R.H.S. of \eqref{eq:chernoff1}. From the definition of the random variable $E_{n,k}$ in \eqref{eq:Endef}, we can see that

\begin{align}
\mathbb{E}[E_{n,k}] = {\sf Pr}(\bm m_n \in \mathcal{P}_{k}(\varepsilon)) 
& = \int_{\bm z \in \mathcal{P}_{k}(\varepsilon)} f(\bm z_{-k}; \bm \nu) d\bm z_{-k}, \nonumber\\
&\ge \int_{z_1=0}^{\frac{\varepsilon}{\sqrt{2}}}\dots\int_{z_{k-1}=0}^{\frac{\varepsilon}{\sqrt{2}}}\int_{z_{k+1}=0}^{\frac{\varepsilon}{\sqrt{2}}}\dots \int_{z_K=0}^{\frac{\varepsilon}{\sqrt{2}}} f(\bm z_{-k}; \bm \nu) d\bm z_{-k} \nonumber\\
&= F(\bm z_{-k}^{*}; \bm \nu)  \ge   \frac{\Gamma (\sum_{i=1}^K \nu_i)}{\prod_{i=1}^K \Gamma(\nu_i)}
\prod_{i\neq k} \frac{\left(\frac{\varepsilon}{\sqrt{2}}\right)^{\nu_i}}{\nu_i}\nonumber\\
&= \frac{\Gamma (\sum_{i=1}^K \nu_i)}{\prod_{i=1}^K \Gamma(\nu_i)} \frac{(\varepsilon/\sqrt{2})^{\sum_{i \neq k}\nu_i}}{\prod_{i \neq k}\nu_i} := G_k(\varepsilon,\bm \nu), \label{eq:cdf}
\end{align}
where $\bm z_{-k}^{*} = [z_1^*,\dots,z^*_{k-1},z^*_{k+1},z^*_K]^{\top} = [\frac{\varepsilon}{\sqrt{2}},\dots,\frac{\varepsilon}{\sqrt{2}}]^{\top}$. The first inequality is obtained by generalizing the geometric property as shown in Fig.~\ref{fig:simplex} for any $K-1$ dimensional probability simplex and the second inequality is obtained by combining \eqref{eq:cdf1} and the property that $L_k(\bm m,\bm \alpha)> 1$ for any $k$ \cite{gouda2010numerical}.
\color{black}



\begin{figure}[t!]
\black 
	\centering 
	\includegraphics[scale=0.15]{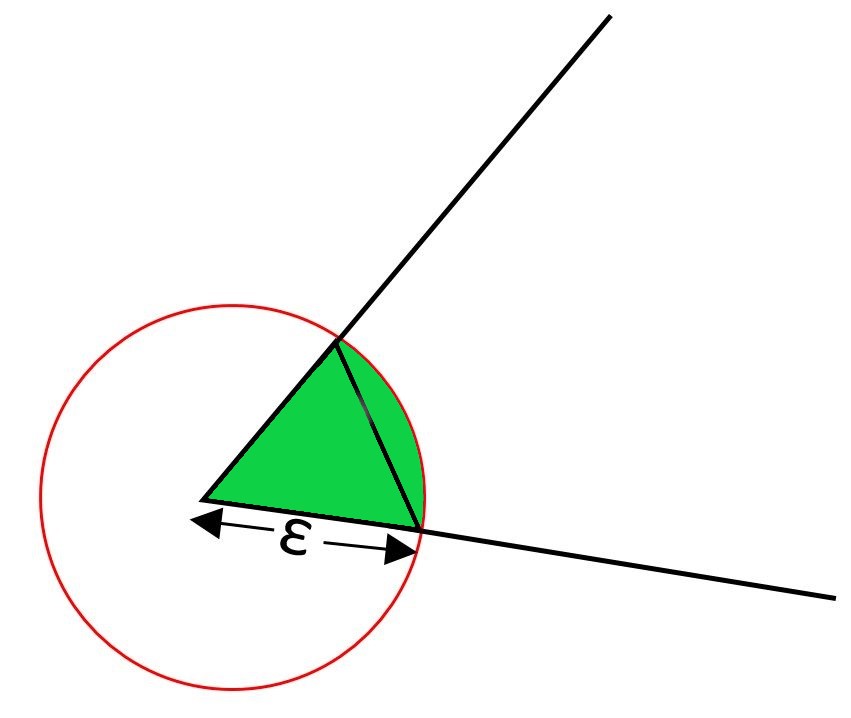}
	\caption{The figure shows any vertex $k \in \{1,2,3\}$ of the probability simplex $\mathcal{P}^3$. The green shaded region represents $\mathcal{P}_{k}(\varepsilon)$ (see \eqref{eq:calPk}) which is the intersection of the probability simplex $\mathcal{P}^3$ and a Euclidean ball of radius $\varepsilon$ with its center at $\bm e_k$. The small triangle near the vertex is $2$-dimensional probability simplex with edge lengths $\varepsilon$; i.e, it intersects the co-ordinate axes at $\frac{\varepsilon}{\sqrt{2}}\bm e_k$. It can be observed that the volume of the green shaded region is larger than the volume of the small triangle.}
	\label{fig:simplex}
\end{figure}

Applying \eqref{eq:cdf} in \eqref{eq:chernoff1}, we get the probability such that there exists at least one $\bm m_n$ such that $\bm m_n \in \mathcal{P}_{k}(\varepsilon)$ as below:
\begin{align*}
{\sf Pr}\left(H_k > 0\right) &\ge 1- e^{-2N(G_k(\varepsilon,\bm \nu))^2}.
\end{align*}
To characterize the corresponding probability for all the $K$ vertices, we can apply the union bound to obtain the below
\begin{align} \label{eq:eventprob}
{\sf Pr}\left(\bigcap_{k=1}^K \left(H_k > 0\right)\right) &\ge 1- \sum_{k=1}^Ke^{-2N(G_k(\varepsilon,\bm \nu))^2} \\
&\ge 1- Ke^{-2N(G(\varepsilon,\bm \nu))^2} ,
\end{align}
where $G(\varepsilon,\bm \nu) = \min_k G_k(\varepsilon,\bm \nu)$.

Eq.~\eqref{eq:eventprob} implies that with probability greater than or equal to $1- Ke^{-2N(G(\varepsilon,\bm \nu))^2} $, the matrix $\bm M=[\bm m_1,\dots, \bm m_N]$ satisfies $\varepsilon$-separability. By letting $\mu =Ke^{-2N(G(\varepsilon,\bm \nu))^2}$, one can see that if 
	$N \ge {\rm log}\left(\frac{K}{\mu}\right)/(2G(\varepsilon,\bm \nu))^2$,
the matrix $\bm M=[\bm m_1,\dots, \bm m_N]$ satisfies $\varepsilon$-separability with probbaility greater than $1-\mu$ where
\begin{align*}
G(\varepsilon,\bm \nu) =  \frac{\Gamma (\sum_{i=1}^K \nu_i)}{\prod_{i=1}^K \Gamma(\nu_i)} \min_k\frac{(\varepsilon/\sqrt{2})^{\sum_{i \neq k}\nu_i}}{\prod_{i \neq k}\nu_i}.
\end{align*}

}

\section{Proof of Theorem \ref{prop:subspace_estim_noisy}} \label{supp:subspace_estim_noisy}

{\black The analysis for Theorem  \ref{prop:subspace_estim_noisy}
consists of three major steps. We provide an overview to the proof as follows:

\noindent    
{\bf (A) Range Space Estimation.}    In this step, lines \ref{line:Tinit}-\ref{line:Uestim} of Algorithm \ref{algo:proposed} that learn the matrix $\bm U$ such that ${\sf range}({\bm U})= {\sf range}(\bm M^\T)$ is analyzed.
 \begin{itemize}
     \item[-]     {\bf A1)  }
    Under the binary observation model, the top-$K$ SVD's are performed on the adjacency blocks $\bm A_{\ell,m}$'s not on blocks $\bm P_{\ell,m}$'s. The noise associated with this is characterized by using Lemmas \ref{lem:svds} and \ref{lem:adjacency}.
    \item[-]     {\bf A2)  }
    Each matrix block $\bm U_{r}$ is estimated using the sub-algorithm \ref{algo:subspace_stitiching}. 
    Lemmas~\ref{lem:betaminmax}-\ref{lem:matrixmult} are derived to characterize this basis stitching operation. A number of lemmas (Lemmas \ref{lem:pseudoinverse}-\ref{lem:sigmamaxUr}) are employed to assist this step, for e.g., to characterize the condition numbers of the latent matrices.
    
    \item[-]     {\bf A3)  }
     The final estimation error bound for the matrix $
{\bm U} = [\bm U_1^{\top},\ldots,\bm U_L^{\top}]^{\top}$ is obtained by recursively applying the error bound for each $\bm U_r$ and by employing Lemma \ref{lem:sigmabound}.
 \end{itemize}   

\noindent    
{\bf (B) Membership Estimation. } In this step, the error bound for the membership matrix $\bm M$ (see lines \ref{line:Gestim}-\ref{line:Mestim} of Algorithm \ref{algo:proposed}) is analyzed.
\begin{itemize}
    \item[-]     {\bf B1)  } The conditions for $\bm M$ to approximately satisfy the ANC is characterized by utilizing Proposition \ref{lem:Lm}.
    \item[-]     {\bf B2)  }
    The final estimation error bound for $\bm M$ is derived using Lemma \ref{lem:Mestim}. Lemmas \ref{lem:Mestim} characterizes the combined noise in the model which is contributed by both the noise in the estimate of $\bm U$ and the deviation from the ANC.
    
\end{itemize}

\noindent \textbf{(C) Putting Together.}  In this step, various conditions used in Lemma \ref{lem:svds}-\ref{lem:Mestim} and Proposition \ref{lem:Lm} are combined to derive the final conditions of Theorem \ref{prop:subspace_estim_noisy}.

}
\subsection{\black Range Space Estimation}

In Theorem \ref{prop:subspace_estim},
the subspace identifiability is established via successively applying the top-$K$ SVD on $
\bm C_r := [\bm P_{\ell_r,m_r}^{\top}, \bm P_{\ell_{r+1},m_r}^{\top}]^{\top}$, i.e.,
$\bm C_r = [\overline{\bm U}_{\ell_r}^{\top}, \overline{\bm U}_{\ell_{r+1}}^{\top}]^{\top}\bm \Sigma_r{\bm V^{\top}_{m_r}}.
$ However, since $\bm C_r$ is not observed in practice, Algorithm~\ref{algo:proposed} performs the top-$K$ SVD on its estimates $\widehat{\bm C}_r$, which is defined as
$
\widehat{\bm C}_r := [\bm A_{\ell_r,m_r}^{\top}, \bm A_{\ell_{r+1},m_r}^{\top}]^{\top}.$ 
The factors extracted from the top-$K$ SVD of $\widehat{\bm C}_r$ are denoted as $\widehat{\overline{\bm U}}_{\ell_r}$ and $\widehat{\overline{\bm U}}_{\ell_{r+1}}$, which correspond to the `noiseless versions' $\overline{\bm U}_{\ell_r}$ and $\overline{\bm U}_{\ell_{r+1}}$, respectively.
To proceed, we invoke the following lemma to bound the estimation accuracy of $\widehat{\overline{\bm U}}_{\ell_r}$ and $\widehat{\overline{\bm U}}_{\ell_{r+1}}$.

\begin{lemma} \label{lem:svds}
	Denote $\sigma_1 := \underset{r\in [L-1]}{\max}~ \sigma_{\max}(\bm C_r)$ and $\sigma_K := \underset{r\in [L-1]}{\min}~ \sigma_{\min}(\bm C_r)$. Assume that
	\begin{align} \label{eq:lem1assump}
	{\rm rank}(\bm C_r)=K,\quad
	\|\widehat{\bm C}_r-\bm C_r\|_2 &\le \|\bm C_r\|_2.
	\end{align}
	Then, {there exists an orthogonal matrix} $\bm O_r$ such that
		\begin{align}
		\|\widehat{\overline{\bm U}}_{\ell_r}  - {\overline{\bm U}}_{\ell_r}{\bm O}_r\|_{\rm F} &\le \frac{6\sqrt{2K}\sigma_1\|\widehat{\bm C}_r-\bm C_r\|_{2}}{\sigma_K^2}. \label{eq:lemsvds1}
		\end{align}
	In addition, $\|\widehat{\overline{\bm U}}_{\ell_{r+1}}  - {\overline{\bm U}}_{\ell_{r+1}}{\bm O}_r\|_{\rm F}$ admits the same upper bound in \eqref{eq:lemsvds1}.
\end{lemma}
The proof of the lemma is given in Sec.~\ref{app:svds} in the supplemental material.

In order to utilize the bound in \eqref{eq:lemsvds1}, we proceed to characterize the term $\|\widehat{\bm C}_r-\bm C_r\|_{2}$. To this end, we present the following lemma:
\begin{lemma} \label{lem:adjacency}
	Let $\rho :=\underset{i,j \in [N]}{\max}~\bm P(i,j)$. Assume there exists a positive constant $c_0$ such that
 $L \le \frac{4\rho N}{d}$ and $\rho \ge \frac{Lc_0 \log(3N/L)}{3N}$.
 Then, there exists another positive constant  $\tilde{c}$ (which is a function of $c_0$) such that
 \begin{align}
\|\widehat{\bm C}_r-\bm C_r\|_2 
&\le \tilde{c}\sqrt{\rho N/L},\quad \forall~r\in[L-1] \label{eq:cd1}
\end{align}
holds with probability of at least $1-\frac{L}{2N}$.
\end{lemma}
The proof of the lemma is given in Sec.~\ref{sup:adjacency} of the supplementary material.

Since $\{\ell_r\}_{r=1}^L=[L]$ according to EQP, without loss of any generality {\black (w.l.o.g.)}, in the following sections of the proof, we can fix  $\ell_r=r$ for every $r$. Combining Lemmas~\ref{lem:svds} and \ref{lem:adjacency}, the inequalities, i.e.,
\begin{align} \label{eq:svdnoise}
\|\widehat{\overline{\bm U}}_r  - {\overline{\bm U}}_r{\bm O}_r\|_{\rm F} &\le \phi~\text{and}~ \|\widehat{\overline{\bm U}}_{r+1}  - {\overline{\bm U}}_{r+1}{\bm O}_r\|_{\rm F} \le \phi 
\end{align}
hold with probability of at least $1-\frac{L}{N}$ where $\phi = {\black \frac{6\tilde{c}\sqrt{2K\rho N}\sigma_1 }{\sqrt{L}\sigma_K^2}}$.



Algorithm~\ref{algo:proposed} {\black (lines \ref{algoline:Ct}-\ref{algoline:Ut})} first estimates $\bm U_{T}$ and $\bm U_{T+1}$ by directly performing the top-$K$ SVD to $\widehat{\bm C}_T =  [\bm A_{T,m_T}^{\top}, \bm A_{T+1,m_{T}}^{\top}]^{\top}$.  Recall that we use the notation $\bm C_T =[\bm P_{T,m_T}^{\top}, \bm P_{T+1,m_T}^{\top}]^{\top}= [{\bm U}_{T}^{\top}, {\bm U}_{T+1}^{\top}]^{\top}\bm \Sigma_T{\bm V_{m_T}^{\top}}$.
Therefore, by \eqref{eq:svdnoise} and denoting $\bm O_T:=\bm O$, we have
\begin{align} \label{eq:svdnoise1}
\|\widehat{\bm U}_{T}- \bm U_{T}\bm O\|_{\rm F} \le \phi ~\text{and}~ \|\widehat{\bm U}_{T+1}- \bm U_{T+1}\bm O\|_{\rm F} \le \phi
\end{align}
{\black holding with probability of at least $1-\frac{L}{N}$}, where $\widehat{\bm U}_{T}$ and $\widehat{\bm U}_{T+1}$ are the factors resulted from the top-$K$ SVD of $\widehat{\bm C}_T$.

Next, the iterative steps (lines \ref{algoline:iterative1_begin}-\ref{algoline:iterative1_end}) in Algorithm \ref{algo:proposed} {\black (which includes the sub-algorithm \texttt{PairStitch})} perform the below operation to estimate the basis $\bm U_{r+1}$ for $r=T+1,\dots,L$ (where $T= \left\lfloor L/2 \right\rfloor$):
\begin{align}
\widehat{\bm U}_{r+1} &= 
 \widehat{\overline{\bm U}}_{{r+1}} \widehat{\overline{\bm U}}_{r}^{\dagger} \widehat{\bm U}_{r}.
\label{eq:lemsub1}
\end{align}


In order to analyze $\widehat{\bm U}_{r+1}$, {\black we first define the following notations:
\begin{align*}
    \alpha_{\max} := \max_{r} \sigma_{\max}(\bm M_{r}), \quad \alpha_{\min} := \min_{r} \sigma_{\min}(\bm M_{r}).
\end{align*}
The above implies that $ \frac{\alpha_{\max}}{\alpha_{\min}} = \max_{r}\kappa(\bm M_{r}) = \gamma $.
Under Assumption \ref{as:as1}, ${\rm rank}(\bm M_r)=K$ holds with probability 1. Hence, both $\alpha_{\max} < \infty$ and $\alpha_{\min} > 0$ also hold with probability one. 

Using the above notations and the assumptions of Theorem \ref{prop:subspace_estim_noisy}, we have the following three lemmas:
}



\begin{lemma} \label{lem:betaminmax}
{There exists constants  {\black $\beta_{\min} > 0$} and {\black $\beta_{\max}< \infty$} such that  for every $r \in [L-1]$, we have
\begin{align*}
\sigma_{\max}([\bm M_r,\bm M_{r+1}]) &\le \beta_{\max}= \sqrt{2K}\alpha_{\max},\\
\sigma_{\min}([\bm M_r,\bm M_{r+1}])  &\ge \beta_{\min}=\alpha_{\min}.
\end{align*}
}
\end{lemma}
\begin{lemma} \label{lem:leastsquare}
    Let $\widehat{\overline{\bm U}}_{{r}}$ denote the noisy estimate of $\overline{\bm U}_{{r}}\bm O_r$ obtained from the top-$K$ SVD of $\widehat{\bm C}_r$. Suppose that  $\|\widehat{\overline{\bm U}}_{{r}}- {\overline{\bm U}}_{{r}}\bm O_r\|_2 \le \frac{\alpha_{\min}}{2\beta_{\max}}$. Then, we have
$\|\widehat{\overline{\bm U}}_{r}^{\dagger} - {(\overline{\bm U}}_{{r}}\bm O_r)^{\dagger}\|_2 \le \frac{2\sqrt{2}\beta_{\max}^2}{\alpha_{\min}^2} \|\widehat{\overline{\bm U}}_{{r}}- {(\overline{\bm U}}_{{r}}\bm O_r)\|_2 $.
\end{lemma}

\begin{lemma} \label{lem:matrixmult}
Consider the relation in \eqref{eq:lemsub1}. Assume that $\|\widehat{\overline{\bm U}}_{{r}}- {\overline{\bm U}}_{{r}}\bm O_r\|_2 \le \frac{\alpha_{\min}}{2\beta_{\max}}$ and $\|\widehat{\overline{\bm U}}_{{r+1}}- {\overline{\bm U}}_{{r+1}}\bm O_r\|_2 \le \frac{\alpha_{\min}}{2\beta_{\max}}$, then we have
\begin{align}
\|\widehat{\bm U}_{r+1}- \bm U_{r+1}\bm O\|_{2} &\le 
\frac{\beta_{\max}}{\alpha_{\min}}\frac{\alpha_{\max}}{\beta_{\min}}\|\widehat{\overline{\bm U}}_{r+1}  - {\overline{\bm U}}_{r+1}{\bm O}_r\|_{2} +2\left(\frac{\alpha_{\max}}{\beta_{\min}}\right)^2 \|\widehat{\overline{\bm U}}_{{r}}^{\dagger} - {(\overline{\bm U}}_{{r}}\bm O_r)^{\dagger}\|_2\nonumber\\ &\quad+4\frac{\beta_{\max}}{\alpha_{\min}}\frac{\alpha_{\max}}{\beta_{\min}}\|\widehat{\bm U}_{r}- \bm U_{r}\bm O\|_{2}. \label{eq:bound_matrixmult}
\end{align}

\end{lemma}
The proofs of Lemmas~\ref{lem:betaminmax}-\ref{lem:matrixmult} can be found Sec.~\ref{app:betaminmax}, \ref{app:leastsquare}, and \ref{app:matrixmult}  of the {\black supplementary} material, respectively. {\black By applying Lemma \ref{lem:leastsquare} to the second term in \eqref{eq:bound_matrixmult} of Lemma \ref{lem:matrixmult}, we get
\begin{align}
\|\widehat{\bm U}_{r+1}- \bm U_{r+1}\bm O\|_{2} &\le 
\frac{\beta_{\max}}{\alpha_{\min}}\frac{\alpha_{\max}}{\beta_{\min}}\|\widehat{\overline{\bm U}}_{r+1}  - {\overline{\bm U}}_{r+1}{\bm O}_r\|_{2} +4\sqrt{2}\left(\frac{\alpha_{\max}\beta_{\max}}{\beta_{\min}\alpha_{\min}}\right)^2  \|\widehat{\overline{\bm U}}_{{r}}- {(\overline{\bm U}}_{{r}}\bm O_r)\|_2 \nonumber\\ &\quad+4\frac{\beta_{\max}}{\alpha_{\min}}\frac{\alpha_{\max}}{\beta_{\min}}\|\widehat{\bm U}_{r}- \bm U_{r}\bm O\|_{2}. \label{eq:bound_matrixmult1}
\end{align}

Employing Lemma \ref{lem:betaminmax}, we have
\begin{align}
    \frac{\beta_{\max}}{\alpha_{\min}}\frac{\alpha_{\max}}{\beta_{\min}} = \frac{\sqrt{2K}\alpha_{\max}}{\alpha_{\min}}\frac{\alpha_{\max}}{\alpha_{\min}} = \sqrt{2K}\gamma^2. \label{eq:kappadef}
\end{align}

Combining \eqref{eq:bound_matrixmult1} and \eqref{eq:kappadef} and defining $\overline{\kappa}=\sqrt{2K}\gamma^2$, we obtain
} 
\begin{align} 
\|\widehat{\bm U}_{r+1}- \bm U_{r+1}\bm O\|_{2} &\le  \overline{\kappa}\|\widehat{\overline{\bm U}}_{r+1}  - {\overline{\bm U}}_{r+1}{\bm O}_r\|_{2}+4\sqrt{2}\overline{\kappa}^2 \|\widehat{\overline{\bm U}}_{{r}} - {(\overline{\bm U}}_{{r}}\bm O_r)\|_2+4\overline{\kappa}\|\widehat{\bm U}_{r}- \bm U_{r}\bm O\|_{2}.\label{eq:Ucap}
\end{align}


Hence, we obtain the following inequalities with probability of at least $1-\frac{L}{N}$:
\begin{align*}
\|\widehat{\bm U}_{r+1}- \bm U_{r+1}\bm O\|_{\rm F} &\le \sqrt{K}\|\widehat{\bm U}_{r+1}- \bm U_{r+1}\bm O\|_{2}   \\
&~ \le \sqrt{K}\overline{\kappa} \|\widehat{\overline{\bm U}}_{r+1}  - {\overline{\bm U}}_{r+1}{\bm O}_r\|_{2} + 4\sqrt{2}\sqrt{K}\overline{\kappa}^2\|\widehat{\overline{\bm U}}_{{r}} - {(\overline{\bm U}}_{{r}}\bm O_r)\|_2\\
&~~~~+ 4\sqrt{K}\overline{\kappa}\|\widehat{\bm U}_{r}- \bm U_{r}\bm O\|_{2}\\
&~\le 5\sqrt{2K}\overline{\kappa}^2\phi +4\overline{\kappa}\sqrt{K}\|\widehat{\bm U}_{r}- \bm U_{r}\bm O\|_{\rm F},
\end{align*}
where $\phi = {\black \frac{6\tilde{c}\sqrt{2K\rho N}\sigma_1 }{\sqrt{L}\sigma_K^2}}$. The first inequality is obtained by using norm equivalence, the second inequality is obtained by applying  \eqref{eq:Ucap} and the last inequality is by applying \eqref{eq:svdnoise} and using the fact that $\overline{\kappa} \ge 1$.

By recursively applying the above result for any $r\in \{T+1,\dots,L-1\}$, we further have
\begin{align}
\|\widehat{\bm U}_{r+1}- \bm U_{r+1}\bm O\|_{\rm F} 
&~\le 5\sqrt{2K}\overline{\kappa}^2 \phi +4\sqrt{K}\overline{\kappa}\bigg(5\sqrt{2K}\overline{\kappa}^2\phi +4\sqrt{K}\overline{\kappa}\ldots\nonumber\\
 &\quad\quad\quad\big(5\sqrt{2K}\overline{\kappa}^2\phi +4\sqrt{K}\overline{\kappa}\|\widehat{\bm U}_{T+1}- \bm U_{T+1}\bm O\|_{\rm F}\big)\bigg)\nonumber\\
&~\le 5\sqrt{2K}\overline{\kappa}^2 \phi +4\sqrt{K}\overline{\kappa}\bigg(5\sqrt{2K}\overline{\kappa}^2\phi +4\sqrt{K}\overline{\kappa}\ldots\nonumber\\ &\quad\quad\quad\big(5\sqrt{2K}\overline{\kappa}^2\phi +4\sqrt{K}\overline{\kappa}\phi\big)\bigg)\nonumber\\
&~ = 5\sqrt{2K}\overline{\kappa}^2 \phi{\left(1+4\sqrt{K}\overline{\kappa}+(4\sqrt{K}\overline{\kappa})^2+\ldots+ (4\sqrt{K}\overline{\kappa})^{r-T} \right)}, \nonumber 
\end{align}
where we have applied \eqref{eq:svdnoise1} for obtaining the last inequality.
By applying the geometric sum formula, for any $r\in \{T+1,\dots,L-1\}$, we have
\begin{align} \label{eq:boundUt}
\|\widehat{\bm U}_{r+1}- \bm U_{r+1}\bm O\|_{\rm F}  &\le \frac{5\sqrt{2K}\overline{\kappa}^2\phi \left((4\sqrt{K}\overline{\kappa})^{r-T+1}-1\right)}{4\sqrt{K}\overline{\kappa}-1}\nonumber\\
&\le \frac{5\sqrt{2K}\overline{\kappa}^2\phi (4\sqrt{K}\overline{\kappa})^{r-T+1}}{4\sqrt{K}\overline{\kappa}-1}.
\end{align}


Following the same derivation, one can show that for $r=2,\dots,T$, we have 
\begin{align} \label{eq:boundUt1}
\|\widehat{\bm U}_{r-1}- \bm U_{r-1}\bm O\|_{\rm F}  \le \frac{5\sqrt{2K}\overline{\kappa}^2\phi (4\sqrt{K}\overline{\kappa})^{T-r+2}}{4\sqrt{K}\overline{\kappa}-1}.
\end{align}
Consequently, we get 
\begin{align*}
\|\widehat{\bm U}- \bm U\bm O\|_{\rm F}
&\le \sum_{r=1}^L \|\widehat{\bm U}_r- \bm U_r\bm O\|_{F}\\
& \quad = \sum_{r=1}^{T-1} \|\widehat{\bm U}_r- \bm U_r\bm O\|_{F}+\|\widehat{\bm U}_{T}- \bm U_{T}\bm O\|_{\rm F}+\|\widehat{\bm U}_{T+1}- \bm U_{T+1}\bm O\|_{\rm F}+\sum_{r=T+2}^L \|\widehat{\bm U}_r- \bm U_r\bm O\|_{F}\\
&\quad \le \frac{5\sqrt{2K}\overline{\kappa}^2\phi (4\sqrt{K}\overline{\kappa})^{T+1}}{(4\sqrt{K}\overline{\kappa}-1)^2}+  \frac{5\sqrt{2K}\overline{\kappa}^2\phi (4\sqrt{K}\overline{\kappa})^{T+1}}{(4\sqrt{K}\overline{\kappa}-1)^2}\\
&\quad =\frac{10\sqrt{2K}\overline{\kappa}^2\phi (4\sqrt{K}\overline{\kappa})^{L/2+1}}{(4\sqrt{K}\overline{\kappa}-1)^2},
\end{align*}
where the first inequality is by the triangle inequality, the second inequality is obtained by \eqref{eq:svdnoise1}, \eqref{eq:boundUt}, \eqref{eq:boundUt1}, and by the geometric sum formula.

Substituting $\phi ={\black \frac{6\tilde{c}\sqrt{2K\rho N}\sigma_1 }{\sqrt{L}\sigma_K^2}}$ into the above result, we have the following result with probability at least $1-L(L-1)/N$:
\begin{align} \label{eq:Uinitial}
\|\widehat{\bm U}- \bm U\bm O\|_{\rm F} &\le \frac{15K\tilde{c}\overline{\kappa}^2 (4\sqrt{K}\overline{\kappa})^{(L/2+1)}\sigma_1\sqrt{\rho N}}{2(\sqrt{K}\overline{\kappa}-1)^2 \sigma_K^2 \sqrt{L}}.
\end{align}
Note that we have applied union bound to reach the probability $1-L(L-1)/N$, since the bound \eqref{eq:svdnoise} derived from Lemma~\ref{lem:adjacency} are invoked for every $r \in [L-1]$. Next, we bound $\sigma_1$ and $\sigma_K$ using the following lemma:
\begin{lemma} \label{lem:sigmabound}
	Assume that the columns of $\bm M$ are generated from any continuous distribution and there exist positive constants $c$ and $C$, where $c \le C$ depending only on distribution of the columns of $\bm M$. Also assume that $(N/L) \ge {\frac{4 }{c}\log(NK/L)}$. Then, the following holds true with probability of at least $1-(3L(L-1)/N)$:
\begin{align}
\sigma_K  &\ge {\black \sqrt{2}c\sigma_{\min} (\bm B)N/L}, \label{eq:lambdamin1}\\
\sigma_1  &\le {\black \sqrt{2}C\sigma_{\max} (\bm B) N/L}. \label{eq:lambdamax1}
\end{align}
\end{lemma}
The proof can be found in Sec.~\ref{sup:sigmabound} of the supplementary material.

%
%
%


Applying Lemma~\ref{lem:sigmabound} in \eqref{eq:Uinitial} and substituting $\overline{\kappa} = \sqrt{2K}\gamma^2$, we get the following with probability at least $1-(4L(L-1)/N)$:
{
\begin{align}
\|\widehat{\bm U}- \bm U\bm O\|_{F} 
&\le \frac{15K^2\tilde{c}C\gamma^4 (4\sqrt{2}K\gamma^2)^{(L/2+1)}\kappa(\bm B)\sqrt{\rho}}{c^2\sigma_{\min}(\bm B)(\sqrt{2}K\gamma^2-1)^2 \sqrt{2(N/L)}} := \zeta. \label{eq:Ubound}
\end{align}
}

\color{black}


\subsection{\black Membership Estimation} \label{app:estimM}
Note that in the $\bm A=\bm P$ case, ${\sf range}(\bm U) = {\sf range}(\bm M^{\top})$. Therefore, there exists a nonsingular matrix $\bm G \in \mathbb{R}^{K \times K}$ such that $\bm U = \bm M^{\top} \bm G$. 
In the binary observation case,
in order to estimate the membership matrix $\bm M$, Algorithm~\ref{algo:proposed} applies \texttt{SPA} to the estimated $\widehat{\bm U}$. To analyze this step, we start by considering the following noisy NMF model:
 \begin{align}
     {\widehat{\bm U}^{\top}} = {\bm O^{\top}\bm G^{\top}}\bm M+\bm N, \label{eq:noisy_nmf1}
 \end{align}
 where $\bm O {\black \in \mathbb{R}^{K \times K}}$ is the orthogonal matrix, $\bm M$ satisfies simplex constraints, i.e.,	$\bm M \ge \bm 0 , \bm 1^{\top}\bm M = \bm 1^{\top}$ and $\bm N \in \mathbb{R}^{K \times N}$ is the noise matrix, whose Euclidean norm upper bound was {\black obtained} in the previous subsection {\black by \eqref{eq:Ubound} such that $\|\bm N\|_{\rm F} \le \zeta$}. From the noisy observation $\widehat{\bm U}$, one can get an estimate for $\bm M$ employing \texttt{SPA}. {\black However, \texttt{SPA} can provably work only when $\bm M$ satisfies $\varepsilon$-separability. For this, we use Proposition \ref{lem:Lm} to get the condition for $\varepsilon$-separability for $\bm M$.
 Then, combining} the noise robustness analysis of \texttt{SPA} in \cite{Gillis2012} {\black and Proposition \ref{lem:Lm}}, we show the following:
\begin{lemma} \label{lem:Mestim}
    Consider the noisy model in \eqref{eq:noisy_nmf1} and assume that $\|\bm N\|_{\rm F} \le \zeta$ and {\black that the assumptions in Proposition \ref{lem:Lm} hold true}. Also, suppose that
    \begin{align}\label{eq:cond_zetaepsilon3}
	\zeta\le \frac{1}{4\sqrt{2}K\alpha_{\max}\widetilde{\kappa}(\bm M)}  \quad  \text{and}\quad \varepsilon \le  \frac{1}{4\sqrt{2}K\gamma\widetilde{\kappa}( \bm M)}.
	\end{align}
      Then, Algorithm~\ref{algo:proposed} outputs $\widehat{\bm M}$ such that \begin{align}\label{eq:Mestim_bound}
     \underset{\bm \Pi}{\min}{{\|\widehat{\bm M}-\bm \Pi\bm M\|_{\rm F}}} \le 
\sigma_{\max}(\bm M)\sqrt{2}K\gamma\widetilde{\kappa}(\bm M)\left(3\varepsilon+4\zeta\right),
\end{align}
where $\widetilde{\kappa}(\bm M) = 1+160K\gamma^2$.
\end{lemma}
The proof is provided in Sec~\ref{sup:spanoisy} of the supplementary material.

\color{black}

{\black \subsection{Putting Together}
To make the bounds {for $\bm U$ and $\bm M$ given by \eqref{eq:Ubound} and \eqref{eq:Mestim_bound}, respectively,} legitimate, we need the assumptions in Lemmas \ref{lem:svds}-\ref{lem:Mestim} hold true. 

The assumptions on $\bm C_r$ in Lemma~\ref{lem:svds} given by the equalities in \eqref{eq:lem1assump} are satisfied if
\begin{align}
{\rm rank}(\bm M_r) = {\rm rank}(\bm B)&= K,\quad \forall r \in [L] \quad  \text{and} \label{eq:cond1}\\
\tilde{c}\sqrt{\rho N/L} &\le  \sigma_K \label{eq:cond2}.
\end{align}
Under Assumption \ref{as:as1}, ${\rm rank}(\bm M_r) = K$ hold with probability one \cite{sidiropoulos2014parallel,sidiropoulos2012multi}.
Using \eqref{eq:lambdamin1}, The condition \eqref{eq:cond2} can be further written as
\begin{align}
    N/L &\ge \frac{\tilde{c}^2\rho}{2\sigma^2_{\min}(\bm B)c^2} \label{eq:cond33}.
\end{align}

The assumptions in Lemma~\ref{lem:leastsquare} and Lemma~\ref{lem:matrixmult} can be written as:
\begin{align*}
\phi = \frac{6\tilde{c}\sqrt{2K\rho N}\sigma_1 }{\sqrt{L}\sigma_K^2} &\le \frac{6\tilde{c}C\kappa (\bm B)\sqrt{LK\rho} }{c^2\sigma_{\min}(\bm B)\sqrt{ N}} \\
&\le \frac{\alpha_{\min}}{2\beta_{\max}} = \frac{\alpha_{\min}}{2\sqrt{2K}\alpha_{\max}} = \frac{1}{2\sqrt{2K}\gamma},
\end{align*}
where we have utilized Lemma~\ref{lem:sigmabound} to obtain the first inequality and Lemma~\ref{lem:betaminmax} to obtain the second equality.
Hence, we obtain the condition such that
\begin{align}
    N/L &\ge \frac{(6\sqrt{2}\tilde{c}C)^2\kappa^2 (\bm B)K^2\rho \gamma^2 }{c^4\sigma^2_{\min}(\bm B)}  .\label{eq:cond3}
\end{align}
We can see that the condition in \eqref{eq:cond33} is satisfied if \eqref{eq:cond3} holds.

Also, the below conditions from Lemmas \ref{lem:adjacency} and \ref{lem:sigmabound}  have to be satisfied:
\begin{subequations} \label{eq:condLN}
\begin{align}
    L \le \frac{4\rho N}{d},&\quad \rho \ge \frac{Lc_0 \log(3N/L)}{3N}\\
        (N/L) &\ge \frac{4}{c}\log(NK/L).
\end{align}
\end{subequations}


{ Next, we proceed to analyze the conditions in Lemma \ref{lem:Mestim}. The conditions in Proposition \ref{lem:Lm} has to be satisfied by Lemma \ref{lem:Mestim} as well which implies that the below has to hold:
\begin{align} \label{eq:thmLm11}
	N = \Omega\left((G(\varepsilon,\bm \nu))^{-2}{\rm log}\left(K/{\mu}\right)\right),
\end{align}
where $G(\varepsilon,\bm \nu)$ is given by \eqref{eq:Gdef}.

Lemma \ref{lem:Mestim} has a condition on $\zeta$ given by \eqref{eq:cond_zetaepsilon3}. By substituting the value of $\zeta$ given by \eqref{eq:Ubound}, the condition on $\zeta$ in \eqref{eq:cond_zetaepsilon3} can be  re-written as:
\begin{align}
     N/L \ge \frac{(30K^3 \tilde{c} C \gamma^4 \kappa(\bm B)\widetilde{\kappa}(\bm M) \alpha_{\max})^2\rho (4\sqrt{2}K\gamma^2)^{L+2}}{(c^2\sigma_{\min}(\bm B)(K\gamma^2-1)^2)^2} \label{eq:condN1}.
\end{align}

}

Note that the condition on $N$ given by \eqref{eq:cond3} is satisfied if the above condition on $N$ holds true.

Hence, if the conditions \eqref{eq:thmLm11} and \eqref{eq:condN1} are satisfied together with \eqref{eq:cond1} and \eqref{eq:condLN}, the bounds for $\bm U$ and $\bm M$ given by Theorem~\ref{prop:subspace_estim_noisy} hold.

Next, we proceed to derive the probability with which the final bounds for $\bm U$ and $\bm M$ hold true. For this, we summarize the probabilistic events appeared in the proof as follows:
\begin{enumerate}
    \item The first event is that the subspace perturbation error associated with every adjacency matrix blocks $\bm C_r$, $r \in [L-1]$ is bounded. This happens with probability of at least $1-L(L-1)/N$ (see Lemma \ref{lem:adjacency});
    \item The second event is that the singular values $\sigma_1$ and $\sigma_K$ associated with the adjacency matrix blocks $\bm C_r$, $r \in [L-1]$ are bounded. This holds with probability of at least $1-(3L(L-1)/N)$ (see Lemma \ref{lem:sigmabound});
    \item The third event is that there exist anchor nodes in the membership matrix $\bm M$ up to an error bounded by $\varepsilon$. This happens with probability of at least $1-\mu$ according to Proposition \ref{lem:Lm}. By substituting $\mu = L(L-1)/N$ in \eqref{eq:thmLm} of Proposition \ref{lem:Lm}, we can get that the result in Proposition \ref{lem:Lm} holds with probability of at least $1-L(L-1)/N$.
\end{enumerate}

Hence, by applying the union bound to the three event probabilities listed above, the final bounds for $\bm U$ and $\bm M$ hold true with probability of at least $1-(5L(L-1)/N)$.

}





\section{Proof of Lemma \ref{lem:svds}}\label{app:svds}
We invoke the following lemma from \cite{yu2014useful} to characterize the estimation accuracy of the factors resulting from the top-$K$ SVD of $\widehat{\bm C}_r$:
\begin{lemma} \cite{yu2014useful}\label{lem:svds1}
	Let $\bm C \in \mathbb{R}^{m \times n}$ and $\widehat{\bm C} \in \mathbb{R}^{m \times n}$ have singular values $\alpha_1 \ge \alpha_2 \ge \dots \ge \alpha_{\min(m,n)}$ and $\widehat{\alpha}_1 \ge \widehat{\alpha}_2 \ge \dots \widehat{\alpha}_{\min(m,n)}$, respectively. Fix $1\le t \le s \le {\rm rank}(\bm C)$ and assume that $\min(\alpha_{t-1}^2-\alpha_t^2,\alpha_s^2-\alpha_{s+1}^2) > 0$, where $\alpha_0^2 := \infty$ and $\alpha_{{\rm rank}(\bm C)+1} := 0$. Let $q := s-t+1$ and let $\bm U = \begin{bmatrix}\bm u_t, \bm u_{t+1} ,\ldots, \bm u_{s} \end{bmatrix} \in \mathbb{R}^{m \times q}$ and  $\widehat{\bm U} = \begin{bmatrix}\widehat{\bm u}_t, \widehat{\bm u}_{t+1}, \ldots,  \widehat{\bm u}_{s} \end{bmatrix} \in \mathbb{R}^{m \times q}$ have orthonormal columns satisfying $\bm C^{\top} \bm u_j = \alpha_j \bm v_j$ and $\widehat{\bm C}^{\top} \widehat{\bm u}_j = \widehat{\alpha}_j \widehat{\bm v}_j$ for $j = t,t+1,\dots,s$ and let $\bm V = \begin{bmatrix}\bm v_t, \bm v_{t+1}, \ldots, \bm v_{s} \end{bmatrix} \in \mathbb{R}^{n \times q}$ and  $\widehat{\bm V} = \begin{bmatrix}\widehat{\bm v}_t, \widehat{\bm v}_{t+1} , \ldots, \widehat{\bm v}_{s} \end{bmatrix} \in \mathbb{R}^{n \times q}$ have orthonormal columns satisfying $\bm C \bm v_j = \alpha_j \bm u_j$ and $\widehat{\bm C} \widehat{\bm v}_j = \widehat{\alpha}_j \widehat{\bm u}_j$ for $j = t,t+1,\dots,s$. Then there exists an orthogonal matrix ${\bm O} \in \mathbb{R}^{q \times q}$ such that 
	\begin{align*}
	\|\widehat{\bm U}  - \bm U{\bm O}\|_{\rm F}
	&\le \frac{2^{3/2}(2\alpha_1+\|\widehat{\bm C}-\bm C\|_{2})\min(q^{1/2}\|\widehat{\bm C}-\bm C\|_{2},\|\widehat{\bm C}-\bm C\|_{\rm F})}{\min(\alpha_{t-1}^2-\alpha_t^2,\alpha_s^2-\alpha_{s+1}^2)}
	\end{align*}
	and the same upper bound holds for $\|\widehat{\bm V}  - \bm V{\bm O}\|_{\rm F}$.
\end{lemma}

By letting $t=1$ and $s=K$, using the assumption that  ${\rm rank}(\bm C_{r}) = K$ and applying Lemma \ref{lem:svds1}, we have
\begin{align*}
\|\widehat{\overline{\bm U}}_{\ell_r}  - {\overline{\bm U}}_{\ell_r}{\bm O}_r\|_{\rm F}
&= \frac{2^{3/2}\sqrt{K}(2\sigma_{\max}(\bm C_r)+\|\widehat{\bm C}_r-\bm C_r\|_2)\|\widehat{\bm C}_r-\bm C_r\|_2}{\sigma_{\min}(\bm C_{r})^2}\\
&\quad\le \frac{2^{3/2}\sqrt{K}3\sigma_{\max}(\bm C_r)\|\widehat{\bm C}_r-\bm C_r\|_{2}}{\sigma_{\min}(\bm C_{r})^2},
\end{align*}
where $\bm O_r \in \mathbb{R}^{K \times K}$ is orthogonal, the first equality is due to the fact that for any matrix $\bm X$ of rank at least $K$, $\|\bm X\|_{\rm F} \le \sqrt{{\rm rank}(\bm X)} \|\bm X\|_{2} $ and ${\rm rank}(\bm X) \ge K$ and the last inequality {\black is} by using the assumption that $ \|\widehat{\bm C}_r-\bm C_r\|_2 \le\|\bm C_{r}\|_2=\sigma_{\max}(\bm C_r)$. 

Therefore, using the defined $\sigma_1$ and $\sigma_K$, we have for any $r \in [L-1]$,
\begin{align} \label{eq:lem1bound1}
\|\widehat{\overline{\bm U}}_{\ell_r}  - {\overline{\bm U}}_{\ell_r}{\bm O}_r\|_{\rm F} \le \frac{6\sqrt{2K}\sigma_1\|\widehat{\bm C}_r-\bm C_r\|_{2}}{\sigma_K^2}.
\end{align}

Note that upper bound in \eqref{eq:lem1bound1} is valid for $\|\widehat{\overline{\bm U}}_{\ell_{r+1}}  - {\overline{\bm U}}_{\ell_{r+1}}{\bm O}_r\|_{\rm F}$ as well.

\section{Proof of Lemma~\ref{lem:adjacency}}
\label{sup:adjacency}

Consider the below lemma:
\begin{lemma} \cite{lei2015consistency}\label{lem:graph}
	Let $\bm A \in \mathbb{R}^{N \times N}$ be the symmetric binary adjacency matrix of a random graph on $N$ nodes in which the edges $\bm A(i,j)$ occur independently.  Let $\overline{\bm P}$ be a matrix of size $N \times N$ such that the entries $\overline{\bm P}(i,j) \in [0,1]$. Assume that $\bm A(i,j)\sim{\sf Bernoulli}(\overline{\bm P}(i,j))$ for $i<j$, $\bm A(j,i) = \bm A(i,j)$ for $j > i$ and $\bm A(i,i)=0$ for every $i \in [N]$. Set $\mathbb{E}[\bm A]=\overline{\bm P}$ and assume that $d \ge \max({\rho N},c_0\log N)$ where $c_0 > 0$. Then, for any $t >0$, there exists a constant  $c_t =c(t,c_0)$ such that
	\begin{align*}
	\|\bm A-\overline{\bm P}\|_2 \le c_t\sqrt{d},
	\end{align*}
	with probability greater than $1-N^{-t}$.
\end{lemma}

To utilize Lemma \ref{lem:graph} in our case, we analyze diagonal pattern (rightmost pattern in Fig.~\ref{fig:query}) and nondiagonal patterns (for e.g., the first and second patterns in Fig.~\ref{fig:query}) of EQP separately. For the diagonal pattern where $\ell_r=r$ and $\ell_{r+1}=m_r=r+1$ for every $r$, we consider the below sub-adjacency matrix for any $r \in [L-1]$:
\begin{align*}
\widetilde{\bm A}_r = \begin{bmatrix} \bm A_{r,r} ~,~ \bm A_{r,r+1} \\
\bm A_{r+1,r}~ ,~ \bm A_{r+1,r+1} \end{bmatrix},
\end{align*}
where $\widetilde{\bm A}_r \in \mathbb{R}^{2N/L \times 2N/L}$.

Now, let us define
$$\widetilde{\bm P}_r = \begin{bmatrix} \bm P_{r,r} ~,~ \bm P_{r,r+1} \\
\bm P_{r+1,r} ~,~\bm P_{r+1,r+1} \end{bmatrix}.$$

It can be noted that
$    \mathbb{E}[\widetilde{\bm A}_r]= \widetilde{\bm P}_r -{\rm diag}(\widetilde{\bm P}_r).
$ Denoting $\overline{\bm P}_r:= \widetilde{\bm P}_r -{\rm diag}(\widetilde{\bm P}_r)$ and $\rho :=\underset{i,j \in [N]}{\max}~{\bm P}(i,j)$, we have 
\begin{align}
    \|\widetilde{\bm A}_r-\widetilde{\bm P}_r\|_2 &= \|\widetilde{\bm A}_r-\overline{\bm P}_r+\overline{\bm P}_r-\widetilde{\bm P}_r\|_2 \nonumber\\
    &= \|\widetilde{\bm A}_r-\overline{\bm P}_r-{\rm diag}(\widetilde{\bm P}_r)\|_2 \nonumber \\ 
    &\le \|\widetilde{\bm A}_r-\overline{\bm P}_r\|_2 + \|{\rm diag}(\widetilde{\bm P}_r)\|_2 \nonumber \\
    &\le \|\widetilde{\bm A}_r-\overline{\bm P}_r\|_2 + \underset{i}{\max}~{\widetilde{\bm P}_r}(i,i) \nonumber\\
    &\le \|\widetilde{\bm A}_r-\overline{\bm P}_r\|_2 + \rho. \label{eq:AP}
\end{align}
In order to bound the first term  $\|\widetilde{\bm A}_r-\overline{{\bm P}_r}\|_2$, we can utilize Lemma~\ref{lem:graph}.
Hence, by assuming 
\begin{align} \label{eq:dcond}
d \ge \max(\frac{2\rho N}{L},c_0\log(2N/L)),
\end{align}
we apply Lemma~\ref{lem:graph} and \eqref{eq:AP} and obtain
$\|\widetilde{\bm A}_r-\widetilde{\bm P}_r\|_2 \le c_t\sqrt{d} +\rho $
with probability greater than $1-(2N/L)^{-t}$. Next, we will rewrite the condition \eqref{eq:dcond} in a simpler form. Suppose \begin{align} \label{eq:rhocond1}
    \frac{2\rho N}{L} \ge c_0\log(2N/L),
\end{align}
then the condition \eqref{eq:dcond} becomes $d \ge \frac{2\rho N}{L}$. If we choose $L$ to be small enough, then we can satisfy this condition on $d$, i.e., if we have 
\begin{align}\label{eq:Lcond1}
     L \le \frac{4\rho N}{d},
\end{align}
then the condition $d \ge \frac{2\rho N}{L}$ gets automatically satisfied.
Therefore, by using the assumptions \eqref{eq:rhocond1}-\eqref{eq:Lcond1} and  fixing $t=1$,
 we can apply Lemma~\ref{lem:graph} to obtain \begin{align}
 \|\widetilde{\bm A}_r-\widetilde{\bm P}_r\|_2 &\le c_1 \sqrt{d} +\rho \nonumber \\
 &\le 2c_1\sqrt{\rho N/L}+\rho \nonumber\\
 &\le (2c_1+1)\sqrt{\rho N/L} \label{eq:APbound}
 \end{align}
 with probability at least $1-(2N/L)^{-1}$. In the above, the first inequality is by using the relation \eqref{eq:AP}, the second inequality is by employing the assumed bound on $L$ given in \eqref{eq:Lcond1} and the third inequality is by using the fact that $\rho \le 1$.

Then, we have
\begin{align}
&\|\widehat{\bm C}_r-\bm C_r\|_2 \nonumber\\
&\quad= \left\|[\bm A_{r,r+1}^{\top}, \bm A_{r+1,r+1}^{\top}]^{\top}-[\bm P_{r,r+1}^{\top}, \bm P_{r+1,r+1}^{\top}]^{\top}\right\|_{2} \nonumber\\
& \quad\le  \|\widetilde{\bm A}_r-\widetilde{\bm P}_r\|_2 \le (2c_1+1)\sqrt{\rho N/L} \label{eq:cd}
\end{align}
holds with probability greater than $1-(2N/L)^{-1}$, where we have applied \eqref{eq:APbound} to obtain the last inequality.

{ In order to apply Lemma~\ref{lem:graph} for non-diagonal patterns in EQP, we start by defining ${\bm A}_r^*$ and ${\bm P}_r^*$ in the following way:
\begin{align*}
    {\bm A}_r^* &= \begin{bmatrix} \bm 0 &\bm 0&\bm A_{\ell_r,m_r} \\
    \bm 0 & \bm 0& \bm A_{\ell_{r+1},m_r}\\
\bm A_{\ell_r,m_r}^{\top}& \bm A_{\ell_{r+1},m_r}^{\top}& \bm 0\end{bmatrix},\\    
{\bm P}_r^* &= \begin{bmatrix} \bm 0 &\bm 0&\bm P_{\ell_r,m_r} \\
    \bm 0 & \bm 0& \bm P_{\ell_{r+1},m_r}\\
\bm P_{\ell_r,m_r}^{\top}& \bm P_{\ell_{r+1},m_r}^{\top}& \bm 0\end{bmatrix}
\end{align*}
where $\bm 0$ represents a zero matrix which is of size $N/L \times N/L$. With this definition, we can see that ${\bm A}_r^*(i,j) \sim{\sf Bernoulli}({\bm P}_r^*(i,j))$ for $i<j$, ${\bm A}_r^*(j,i) = {\bm A}_r^*(i,j)$ for $j > i$, ${\bm A}_r^*(i,i)=0$ for every $i$ and $\mathbb{E}[{\bm A}_r^*]={\bm P}_r^*$. Therefore, we can apply Lemma~\ref{lem:graph} by fixing $t=1$ to result in
	\begin{align} \label{eq:APtilde}
	\|{\bm A}_r^*-{\bm P}_r^*\|_2 \le c_1\sqrt{d},
	\end{align}
	with probability greater than $1-(3N/L)^{-1}$. Similar as before, we can derive the conditions in order to obtain \eqref{eq:APtilde}. It can be easily shown that the sufficient conditions to be satisfied are
	\begin{align} \label{eq:condLN2}
	    \frac{3\rho N}{L} \ge c_0\log(3N/L),\quad  L \le \frac{6\rho N}{d}.
	\end{align}

Then, for non-diagonal EQP patterns, we have
\begin{align}
\|\widehat{\bm C}_r-\bm C_r\|_2
&= \left\|[\bm A_{\ell_r,m_r}^{\top}, \bm A_{\ell_{r+1},m_r}^{\top}]^{\top}-[\bm P_{\ell_r,m_r}^{\top}, \bm P_{\ell_{r+1},m_r}^{\top}]^{\top}\right\|_{2} \nonumber\\
& \quad\le  \|{\bm A}^*_r-{\bm P}_r^*\|_2 \le c_1\sqrt{d} \le \sqrt{6}c_1\sqrt{\rho N/L}, \label{eq:nondiag}
\end{align}
where we obtained the last inequality by applying the assumed bound on $L$ as given by the second condition in \eqref{eq:condLN2}. 

Hence, from \eqref{eq:cd} and \eqref{eq:nondiag}, we can see that $\|\widehat{\bm C}_r-\bm C_r\|_2$ is bounded for any patterns in EQP with probability of at least $1-(2N/L)^{-1}$ if the below conditions are satisfied.
	\begin{align*} 
	    \frac{3\rho N}{L} \ge c_0\log(3N/L),\quad  L \le \frac{4\rho N}{d}.
	\end{align*}

}

\section{Proof of Lemma~\ref{lem:betaminmax}} \label{app:betaminmax}



Consider the below relation for any $r \in [L-1]$:
\begin{align*}
    \sigma^2_{\max}([\bm M_r,\bm M_{r+1}])&= \|[\bm M_r,\bm M_{r+1}]\|_2^2
     \le \|[\bm M_r,\bm M_{r+1}]\|_{\rm F}^2 \\
    & = \|\bm M_r\|_{\rm F}^2 +\|\bm M_{r+1}\|_{\rm F}^2\\
    &\le K \|\bm M_r\|_{2}^2 + K \|\bm M_{r+1}\|_{2}^2 \le 2K \alpha_{\max}^2,
\end{align*}
where we have utilized the norm equivalence for the first and second inequalities and applied {\black the definition of $\alpha_{\max}$} for the last inequality.
Hence, we have
\begin{align*}
    \sigma_{\max}([\bm M_r,\bm M_{r+1}]) &\le \sqrt{2K}\alpha_{\max} := \beta_{\max}.
\end{align*}
Next, consider the below for any $\bm x \in \mathbb{R}^{K}$ such that $\|\bm x\|_2=1$:
\begin{align*}
    \|[\bm M_r,\bm M_{r+1}]^{\top}\bm x\|_2^2 &= \|\bm M_r^{\top}\bm x\|_2^2 + \|\bm M_{r+1}^{\top}\bm x\|_2^2 \\
    \Rightarrow \underset{\bm x}{\min}~\|[\bm M_r,\bm M_{r+1}]^{\top}\bm x\|^2_2 & = \underset{\bm x}{\min}~\|\bm M_r^{\top}\bm x\|^2_2 + \underset{\bm x}{\min}~\|\bm M_{r+1}^{\top}\bm x\|^2_2 \\
        \Rightarrow \underset{\bm x}{\min}~\|[\bm M_r,\bm M_{r+1}]^{\top}\bm x\|^2_2 &\ge  \underset{\bm x}{\min}~\|\bm M_r^{\top}\bm x\|^2_2 \\
    \Rightarrow \sigma_{\min}([\bm M_r,\bm M_{r+1}]) &\ge \sigma_{\min}(\bm M_r) \ge \alpha_{\min} := \beta_{\min},
\end{align*}
where we have utilized {\black the definition of $\alpha_{\min}$} for the last inequality.


\section{Proof of Lemma \ref{lem:leastsquare}}\label{app:leastsquare}

We utilize the following classic result to characterize the pseudo-inverse:
\begin{lemma}  \label{lem:pseudoinverse}\cite{Wedin1973Perturbation}
Consider any matrices {\black $\bm Y, \bm Z, \bm E \in \mathbb{R}^{m \times n}$} such that {\black $\bm Z = \bm Y+\bm E$}. Suppose {\black ${\rm rank}(\bm Z) = {\rm rank}(\bm Y)$}. Then, we have
\begin{align*}
    {\black \|\bm Z^{\dagger}-\bm Y^{\dagger}\|_2 \le \sqrt{2}\|\bm Y^{\dagger}\|_2\|\bm Z^{\dagger}\|_2 \|\bm E \|_2.}
\end{align*}

\end{lemma}

By fixing {\black $\bm Z := \widehat{\overline{\bm U}}_{{r}}^{\dagger}$ and $\bm Y = {(\overline{\bm U}}_{{r}}\bm O_r)^{\dagger}$}, we can apply Lemma~\ref{lem:pseudoinverse} to obtain
\begin{align} \label{eq:pseudo}
    &\|\widehat{\overline{\bm U}}_{{r}}^{\dagger} - {(\overline{\bm U}}_{{r}}\bm O_r)^{\dagger}\|_2 \nonumber\\ 
    &\quad \le \sqrt{2}\|{(\overline{\bm U}}_{{r}}\bm O_r)^{\dagger}\|_2 \|\widehat{\overline{\bm U}}_{{r}}^{\dagger}\|_2 \|\widehat{\overline{\bm U}}_{{r}}- {(\overline{\bm U}}_{{r}}\bm O_r)\|_2.
\end{align}

Regarding the first term in the R.H.S. of the above equation, we have
\begin{align} \label{eq:Ucap1}
    \|{(\overline{\bm U}}_{{r}}\bm O_r)^{\dagger}\|_2 &= \sigma_{\max}({(\overline{\bm U}}_{{r}}\bm O_r)^{\dagger})\nonumber\\
    &= \frac{1}{\sigma_{\min}(\overline{\bm U}_r\bm O_r)} = \frac{1}{\sigma_{\min}(\overline{\bm U}_r)},
\end{align}
where the last equality is due to the orthogonality of $\bm O_r$.

Similarly, we have
\begin{align}\label{eq:Ucap2}
   \|\widehat{\overline{\bm U}}_{{r}}^{\dagger}\|_2 = \sigma_{\max}(\widehat{\overline{\bm U}}_{{r}}^{\dagger}) = \frac{1}{\sigma_{\min}(\widehat{\overline{\bm U}}_{{r}})}.
\end{align}

To proceed, we aim to get a lower bound for $\sigma_{\min}(\widehat{\overline{\bm U}}_{{r}})$. For this, we have the following result:
\begin{lemma} \label{lem:sigmaUr_p}
The below relation holds true if $\|\widehat{\overline{\bm U}}_{{r}}- {\overline{\bm U}}_{{r}}\bm O_r\|_2 \le \sigma_{\min}({\overline{\bm U}}_{{r}})/2$ :
\begin{align*}
    \sigma_{\min}(\widehat{\overline{\bm U}}_{{r}}) &\ge \sigma_{\min}({\overline{\bm U}}_{{r}})/2.
\end{align*}
\end{lemma}

{\it Proof:}
Consider the below set of relations for any vector $\bm x \in \mathbb{R}^{K }$ satisfying $\|\bm x\|=1$:
\begin{align*}
    \|\widehat{\overline{\bm U}}_{{r}}\bm x\|_2 &= \|{\overline{\bm U}}_{{r}}\bm O_r\bm x+(\widehat{\overline{\bm U}}_{{r}}- {\overline{\bm U}}_{{r}}\bm O_r)\bm x \|_2\\
    &\ge  \|{\overline{\bm U}}_{{r}}\bm O_r\bm x\|_2- \|(\widehat{\overline{\bm U}}_{{r}}- {\overline{\bm U}}_{{r}}\bm O_r)\bm x\|_2,\\
    \Rightarrow \underset{\bm x}{\min}~\|\widehat{\overline{\bm U}}_{{r}}\bm x\|_2 &\ge \underset{\bm x}{\min}~ \|{\overline{\bm U}}_{{r}}\bm O_r\bm x\|_2 - \underset{\bm x}{\max}~\|(\widehat{\overline{\bm U}}_{{r}}- {\overline{\bm U}}_{{r}}\bm O_r)\bm x\|_2,\\
    \Rightarrow \sigma_{\min}(\widehat{\overline{\bm U}}_{{r}}) &\ge \sigma_{\min}({\overline{\bm U}}_{{r}}) - \|\widehat{\overline{\bm U}}_{{r}}- {\overline{\bm U}}_{{r}}\bm O_r\|_2,
\end{align*}
where the first inequality is by applying the triangle inequality.

Using the assumption that $\|\widehat{\overline{\bm U}}_{{r}}- {\overline{\bm U}}_{{r}}\bm O_r\|_2 \le \sigma_{\min}({\overline{\bm U}}_{{r}})/2$, we get $\sigma_{\min}(\widehat{\overline{\bm U}}_{{r}}) \ge \sigma_{\min}({\overline{\bm U}}_{{r}})/2$.
\hfill$\square$

Applying \eqref{eq:Ucap1}, \eqref{eq:Ucap2} and Lemma~\ref{lem:sigmaUr_p} in \eqref{eq:pseudo}, we get
\begin{align} \label{eq:Upseudo}
    \|\widehat{\overline{\bm U}}_{{r}}^{\dagger} - {(\overline{\bm U}}_{{r}}\bm O_r)^{\dagger}\|_2 \le \frac{2\sqrt{2}}{\sigma_{\min}^2({\overline{\bm U}}_{{r}})} \|\widehat{\overline{\bm U}}_{{r}}- {(\overline{\bm U}}_{{r}}\bm O_r)\|_2.
\end{align}

The final step is to characterize $\sigma_{\min}({\overline{\bm U}}_{{r}})$. For this, we utilize the following result:
\begin{lemma} \label{lem:sigmaU}
    {\black For every $r$},  we have
    \begin{align*}
        \sigma_{\min}(\overline{\bm U}_{r}) &\ge \frac{\sigma_{\min}(\bm M_r)}{\sigma_{\max}([\bm M_r,\bm M_{r+1}])} =\frac{\alpha_{\min}}{\beta_{\max}}, \\
      \sigma_{\max}(\overline{\bm U}_{r}) &\le \frac{\sigma_{\max}(\bm M_r)}{\sigma_{\min}([\bm M_r,\bm M_{r+1}])} =\frac{\alpha_{\max}}{\beta_{\min}}  .
    \end{align*}
\end{lemma}

{\it Proof:}
 Recall the below set of relations (assuming $\ell_r=r$ for every $r$):
 \begin{align}
     \bm C_r &= [\bm P_{r,m_r}^{\top}, \bm P_{r+1,m_r}^{\top}]^{\top} 
             = [\bm M_r,\bm M_{r+1}]^{\top}\bm B \bm M_{m_r}. \label{eq:Cr_r1}
 \end{align}
 
The top-$K$ SVD of $\bm C_r$ results in the below relation:
 \begin{align}
     \bm C_r = [\overline{\bm U}_{r}^{\top}, \overline{\bm U}_{r+1}^{\top}]^{\top}\bm \Sigma_r{\bm V^{\top}_{m_r}}.\label{eq:Cr_r2}
 \end{align}
 
 From \eqref{eq:Cr_r1} and \eqref{eq:Cr_r2}, we get that there exists a nonsingular matrix $\overline{\bm G}_r$ such that
 \begin{align} \label{eq:UGM}
     [\overline{\bm U}_{r}^{\top}, \overline{\bm U}_{r+1}^{\top}]^{\top} = [\bm M_r,\bm M_{r+1}]^{\top} \overline{\bm G}_r,
 \end{align}
  where the matrix $[\overline{\bm U}_{r}^{\top}, \overline{\bm U}_{r+1}^{\top}]^{\top}$ is semi-orthogonal. We invoke the below fact to proceed further.
 
 \begin{Fact} \label{fact:orthU}
     Consider the equation $\overline{\bm U} =  \overline{\bm M}^{\top}\overline{\bm G}$ where the matrix $\overline{\bm U}$ is tall and is semi-orthogonal. Then the below holds:
     \begin{align*}
         \sigma_{\max}(\overline{\bm G}) = 1/\sigma_{\min}(\overline{\bm M})\quad\text{and}\quad \sigma_{\min}(\overline{\bm G}) = 1/\sigma_{\max}(\overline{\bm M}).
     \end{align*}
      \end{Fact}
    {\it Proof:}
     Since $\overline{\bm U}$ is a semi-orthogonal matrix, we have
\begin{align*}
    \overline{\bm U}^{\top}\overline{\bm U} &= \bm I,\\
    \implies \overline{\bm G} ^{\top} \overline{\bm M}  \overline{\bm M}^{\top}\overline{\bm G} &= \bm I,\\
    \implies  \overline{\bm M}  \overline{\bm M}^{\top} &= \overline{\bm G}^{-\top}\overline{\bm G}^{-1} = (\overline{\bm G} \overline{\bm G}^{\top})^{-1}.
\end{align*}
The above relation implies that $\sigma_{\max}(\overline{\bm G}) = 1/\sigma_{\min}(\overline{\bm M})$, $\sigma_{\min}(\overline{\bm G}) = 1/\sigma_{\max}(\overline{\bm M})$ hold true.
\hfill$\square$
 
 Due to Fact~\ref{fact:orthU}, from \eqref{eq:UGM}, we get
 \begin{subequations}\label{eq:GMM}
 \begin{align} 
     \sigma_{\max}(\overline{\bm G}_r) &= \frac{1}{\sigma_{\min}([\bm M_r,\bm M_{r+1}])},\\
      \sigma_{\min}(\overline{\bm G}_r) &= \frac{1}{\sigma_{\max}([\bm M_r,\bm M_{r+1}])}.
 \end{align}
 \end{subequations}
 
 Next we consider the below relation obtained from \eqref{eq:UGM}:
 \begin{align*}
     \overline{\bm U}_{r} =\bm M_r^{\top} \overline{\bm G}_r .
 \end{align*}
 Since $\bm M_r$ is full row-rank, we have
 \begin{align*}
     \sigma_{\min}(\overline{\bm U}_{r})  &= \underset{\|\bm x\|_2=1}{\min}~\| \bm M_r^{\top}\overline{\bm G}_r\bm x\|_2
     \ge  \underset{\|\bm x\|_2=1}{\min}~\sigma_{\min}(\bm M_r)\|\overline{\bm G}_r\bm x\|_2 \\
     &= \sigma_{\min}(\bm M_r)\underset{\|\bm x\|_2=1}{\min}~\|\overline{\bm G}_r\bm x\|_2
     =\sigma_{\min}(\bm M_r) \sigma_{\min}(\overline{\bm G}_r) \\
     &= \frac{\sigma_{\min}(\bm M_r)}{\sigma_{\max}([\bm M_r,\bm M_{r+1}])} \ge \frac{\alpha_{\min}}{\beta_{\max}},
 \end{align*}
 where we have applied \eqref{eq:GMM} to obtain the last equality and used Lemma~\ref{lem:betaminmax} for the last inequality.
 

Similarly, we can easily have,
\begin{align*}
    \sigma_{\max}(\overline{\bm U}_{r}) &\le  \sigma_{\max}(\overline{\bm G}_r) \sigma_{\max}(\bm M_r) \\
    &=\frac{\sigma_{\max}(\bm M_r)}{\sigma_{\min}([\bm M_r,\bm M_{r+1}])} \le \frac{\alpha_{\max}}{\beta_{\min}}.
\end{align*}
\hfill$\square$


Combining Lemma~\ref{lem:sigmaU} with \eqref{eq:Upseudo}, we have
\begin{align} \label{eq:leastsquares2}
    \|\widehat{\overline{\bm U}}_{{r}}^{\dagger} - {(\overline{\bm U}}_{{r}}\bm O_r)^{\dagger}\|_2 \le \frac{2\sqrt{2}\beta_{\max}^2}{\alpha_{\min}^2} \|\widehat{\overline{\bm U}}_{{r}}- {(\overline{\bm U}}_{{r}}\bm O_r)\|_2
\end{align}

\section{Proof of Lemma \ref{lem:matrixmult}}\label{app:matrixmult}

For simpler notations, let us define $\bm G_1 : = \overline{\bm U}_{{r+1}}{\bm O}_{r}$, $\bm G_2 := ({\overline{\bm U}_{r}{\bm O}_{r})^{\dagger}}$ and $\bm G_3 := {\bm U}_{{r}}\bm O$.  Also we define $\widehat{\bm G}_1 : = \widehat{\overline{\bm U}}_{{r+1}}$, $\widehat{\bm G}_2 := (\widehat{\overline{\bm U}}_{r})^{\dagger}$ and $\bm G_3 := \widehat{\bm U}_{{r}}\bm O$. With these, consider the below:
\begin{align}
\left\|\widehat{\bm G}_1\widehat{\bm G}_2\widehat{\bm G}_3-\bm G_1\bm G_2\bm G_3\right\|_{2}
&= \left\|\widehat{\bm G}_1\widehat{\bm G}_2\widehat{\bm G}_3-\bm G_1\bm G_2\bm G_3 - \widehat{\bm G}_1\bm G_2\bm G_3+\widehat{\bm G}_1\bm G_2\bm G_3\right\|_{2} \nonumber\\
&= \left\|\left(\widehat{\bm G}_1-\bm G_1\right)\bm G_2\bm G_3+\widehat{\bm G}_1\left(\widehat{\bm G}_2\widehat{\bm G}_3-\bm G_2\bm G_3\right)\right\|_2 \nonumber\\
&\le \left\|\left(\widehat{\bm G}_1-\bm G_1\right)\bm G_2\bm G_3\right\|_2+\left\|\widehat{\bm G}_1\left(\widehat{\bm G}_2\widehat{\bm G}_3-\bm G_2\bm G_3\right)\right\|_2 \nonumber\\
&= \left\|\left(\widehat{\bm G}_1-\bm G_1\right)\bm G_2\bm G_3\right\|_2 \nonumber\\
&\quad\quad+\left\|\widehat{\bm G}_1(\widehat{\bm G}_2-\bm G_2)\bm G_3+\widehat{\bm G}_1\widehat{\bm G}_2\left(\widehat{\bm G}_3-\bm G_3\right)\right\|_2 \nonumber\\
&\le \|\bm G_2\|_2\|\bm G_3\|_2\left\|\widehat{\bm G}_1-\bm G_1\right\|_2+\|\widehat{\bm G}_1\|_2\|\bm G_3\|_2 \left\|\widehat{\bm G}_2-\bm G_2\right\|_2 \nonumber \\
&\quad \quad  +\|\widehat{\bm G}_1\|_2\|\widehat{\bm G}_2\|_2 \left\|\widehat{\bm G}_3-\bm G_3\right\|_2, \label{eq:prodbound}
\end{align}
where we have applied triangle inequality and the fact that {\black $\|\bm Y\bm Z\|_2 \le \|\bm Y\|_2 \|\bm Z\|_2$} to obtain the first and last inequalities, respectively.

 We introduce the following lemma and then proceed to characterize $\|\widehat{\bm G}_1\|_2$, $\|\bm G_2\|_2$, $\|\widehat{\bm G}_2\|_2$ and $\|\bm G_3\|_2$.
 \begin{lemma} \label{lem:sigmamaxUr}
{ \black The below relation holds for every $r$:} $$\sigma_{\max}(\bm U_r) \le \frac{\alpha_{\max}}{\beta_{\min}}.$$
\end{lemma}
The proof of the lemma is provided in Sec.~\ref{sup:sigmamaxUr}.

\begin{itemize}
\item \textbf{Bound for $\|\widehat{\bm G}_1\|_2 = \|\widehat{\overline{\bm U}}_{{r+1}}\|_2$}
\begin{align*}
    \|\widehat{\bm G}_1\|_2 &= \|\widehat{\overline{\bm U}}_{{r+1}}\|_2 
      = \|\widehat{\overline{\bm U}}_{{r+1}} - {\overline{\bm U}}_{{r+1}}\bm O_r+{\overline{\bm U}}_{{r+1}}\bm O_r\|_2\\
      &\le \|\widehat{\overline{\bm U}}_{{r+1}} - {\overline{\bm U}}_{{r+1}}\bm O_r\|_2 + \|{\overline{\bm U}}_{{r+1}}\bm O_r\|_2\\
      &\le \sigma_{\max}( {\overline{\bm U}}_{{r+1}}) + \sigma_{\max}( {\overline{\bm U}}_{{r+1}})\\
      &= 2\sigma_{\max}( {\overline{\bm U}}_{{r+1}}) \le 2\frac{\alpha_{\max}}{\beta_{\min}}.
\end{align*}
where we have used triangle inequality for the first inequality, used the assumption that $\|\widehat{\overline{\bm U}}_{{r+1}} - {\overline{\bm U}}_{{r+1}}\bm O_r\|_2 \le \sigma_{\min}( {\overline{\bm U}}_{{r+1}})/2$ for the second inequality and invoked Lemma~\ref{lem:sigmaU} for the last inequality.

\item \textbf{Bound for $\|\bm G_2\|_2=\|({\overline{\bm U}_{r}{\bm O}_{r})^{\dagger}}\|_2$}
\begin{align*}
    \|\bm G_2\|_2=\|({\overline{\bm U}_{r}{\bm O}_{r})^{\dagger}}\|_2 = 1/\sigma_{\min}(\overline{\bm U}_{r}) \le \frac{\beta_{\max}}{\alpha_{\min}},
\end{align*}
where we have applied Lemma~\ref{lem:sigmaU} for the last inequality.

\item \textbf{Bound for $\|\widehat{\bm G}_2\|_2=\|\widehat{\overline{\bm U}}_{r}^{\dagger}\|_2$}
\begin{align*}
   \|\widehat{\bm G}_2\|_2&=\|\widehat{\overline{\bm U}}_{r}^{\dagger}\|_2 = 1/\sigma_{\min}(\widehat{\overline{\bm U}}_{r}^{\dagger})
       \le 2/\sigma_{\min}(\overline{\bm U}_{r})  \le 2\frac{\beta_{\max}}{\alpha_{\min}},
\end{align*}
where we have used applied Lemma~\ref{lem:sigmaUr_p} by using the assumption that $\|\widehat{\overline{\bm U}}_{{r}} - {\overline{\bm U}}_{{r}}\bm O_r\|_2 \le \sigma_{\min}( {\overline{\bm U}}_{{r}})/2$ for the first inequality and invoked Lemma~\ref{lem:sigmaU} for the last inequality.

\item \textbf{Bound for $\|\bm G_3\|_2=\| \bm U_r\bm O\|_2$}
\begin{align*}
   \|\bm G_3\|_2&=\| \bm U_r\bm O\|_2
  = \| \bm U_r\|_2 \le \frac{\alpha_{\max}}{\beta_{\min}},
\end{align*}
where we have invoked Lemma~\ref{lem:sigmamaxUr} for the inequality.
\end{itemize}

Applying these upper bounds in \eqref{eq:prodbound},  we finally have
\begin{align*}
\left\|\widehat{\bm G}_1\widehat{\bm G}_2\widehat{\bm G}_3-\bm G_1\bm G_2\bm G_3\right\|_{2} 
&\le  \frac{\beta_{\max}}{\alpha_{\min}}\frac{\alpha_{\max}}{\beta_{\min}}\left\|\widehat{\bm G}_1-\bm G_1\right\|_2+2\frac{\alpha_{\max}}{\beta_{\min}}\frac{\alpha_{\max}}{\beta_{\min}} \left\|\widehat{\bm G}_2-\bm G_2\right\|_2\\ &\quad \quad +4\frac{\beta_{\max}}{\alpha_{\min}}\frac{\alpha_{\max}}{\beta_{\min}}\left\|\widehat{\bm G}_3-\bm G_3\right\|_2.
\end{align*}

\section{Proof of Lemma~\ref{lem:sigmamaxUr}} \label{sup:sigmamaxUr}

	Theorem~\ref{prop:subspace_estim} shows that there exists a certain nonsingular $\bm G$ such that for each true basis $\bm U_r$ for $r \in [L]$, we have:
	\begin{align} \label{eq:Ul}
	\bm U_r = \bm M_r^{\top}\bm G .
	\end{align}

	Recall that, under noiseless case, Algorithm~\ref{algo:proposed} first directly estimates $\bm U_{T}$ and $\bm U_{T+1}$ by performing the top-$K$ SVD to ${\bm C}_T =  [\bm P_{T,m_{T}}^{\top}, \bm P_{T+1,m_{T}}^{\top}]^{\top}$ (assuming $\ell_r=r$ for every $r$) where
	\begin{align} 
	\bm C_T = [{\bm U}_{T}^{\top}, {\bm U}_{T+1}^{\top}]^{\top}\bm \Sigma_T{\bm V^{\top}_{m_T}}.
	\end{align}
	
	The above implies that,
	\begin{align*}
	[{\bm U}_{T}^{\top}, {\bm U}_{T+1}^{\top}]^{\top} =  [\bm M_T, \bm M_{T+1}]^{\top}\bm G
	\end{align*}
	holds where $[{\bm U}_{T}^{\top}, {\bm U}_{T+1}^{\top}]$ is semi-orthogonal. By utilizing Fact~\ref{fact:orthU}, we get $\sigma_{\max}(\bm G) = 1/\sigma_{\min}([\bm M_T, \bm M_{T+1}])$ and $\sigma_{\min}(\bm G) = 1/\sigma_{\max}([\bm M_T, \bm M_{T+1}])$. Using this result, from \eqref{eq:Ul}, we have
	\begin{align*}
	\sigma_{\max}(\bm U_{r}) &\le \sigma_{\max}(\bm G)\sigma_{\max}(\bm M_r)\\
	&\le \frac{\sigma_{\max}(\bm M_r)}{\sigma_{\min}([\bm M_T, \bm M_{T+1}])} \le \frac{\alpha_{\max}}{\beta_{\min}},
	\end{align*}
	where we have utilized Lemma~\ref{lem:betaminmax} to obtain the last inequality.
	

\section{Proof of Lemma~\ref{lem:sigmabound}}\label{sup:sigmabound}
Let us first consider the below:
\begin{align*}
\bm C_r = [\bm P^{\top}_{\ell_r,m_r},\bm P^{\top}_{\ell_{r+1},m_r}]^{\top}= [\bm M_{\ell_r},\bm M_{\ell_{r+1}}]^{\top}\bm B\bm M_{m_r}.
\end{align*}
Consider the $K$-th eigenvalue of the symmetric matrix $\bm C_r^{\top}\bm C_r$:
\begin{align}
&\lambda_{K} (\bm C_r^{\top}\bm C_r) \nonumber\\
&=\lambda_{K} (\underbrace{\bm M_{m_r}^{\top}\bm B[\bm M_{\ell_r},\bm M_{\ell_{r+1}}][\bm M_{\ell_r},\bm M_{\ell_{r+1}}]^{\top}}_{\bm W_1 \in \mathbb{R}^{N/L \times K}}\underbrace{\bm B\bm M_{m_r}}_{\bm W_2 \in \mathbb{R}^{K \times N/L}},\label{eq:W12}
\end{align}
where $\lambda_K$ denotes the $K$th eigenvalue with $\lambda_1 \ge \ldots \ge \lambda_K$.

We utilize the following fact to proceed further:
\begin{lemma}[Theorem 1.3.22]\cite{horn2012matrix} \label{fact:eigen}
	Consider two matrices $\bm W_1 \in \mathbb{R}^{N \times K}$ and $\bm W_2 \in \mathbb{R}^{K \times N}$. Let $\lambda_1,\ldots,\lambda_K$ be the nonzero eigenvalues of the matrix $\bm W_1 \bm W_2$. Then, the matrix $\bm W_2\bm W_1$ also holds the same set of eigenvalues.
\end{lemma}

Combining Lemma \ref{fact:eigen} and \eqref{eq:W12}, we have
\begin{align} \label{eq:SS}
\lambda_{K} (\bm C_r^{\top}\bm C_r) 
& =
\lambda_{K} (\bm B\bm M_{m_r}\bm M_{m_r}^{\top}\bm B[\bm M_{\ell_r},\bm M_{\ell_{r+1}}][\bm M_{\ell_r},\bm M_{\ell_{r+1}}]^{\top}).
\end{align}

Next, we consider the following matrix 
\begin{align} \label{eq:Hi}
\bm H_r := \bm M_{m_r}\bm M_{m_r}^{\top} = \sum_{n=1}^{N/L}\bm M_{m_r}(:,n)\bm M_{m_r}(:,n)^{\top}.
\end{align}

To proceed, we have the lemma from \cite{Panov2017consistent}  to bound $\lambda_{K}(\bm H_r)$ and $\lambda_{1}(\bm H_r)$.

\begin{lemma} \cite{Panov2017consistent}\label{lem:concentrationM}
	Assume that the columns of $\bm M$ are generated from any continuous distribution. Then, there exists positive constants $c$ and $C$ depending only on distribution of the columns of $\bm M$ such that
	\begin{align*}
	{\sf Pr}\left(\lambda_{K}\left(\bm H_r \right) \le {\black cN/L}\right) &\le K\exp(-{\black cN/4L}),\\
	{\sf Pr}\left(\lambda_{1}\left(\bm H_r \right) \ge {\black CN/L}\right) &\le \frac{K}{2^{CN/L}}.
	\end{align*}
\end{lemma}

To simplify the bound on the {\black probability}, let us find the conditions at which the following is satisfied:
\begin{align}
K\exp(-{\black cN/4L}) &\le (N/L)^{-1} \label{eq:KNL1},\\
\quad\frac{K}{2^{CN/L}} &\le (N/L)^{-1}.\label{eq:KNL2}
\end{align}

The conditions \eqref{eq:KNL1} and \eqref{eq:KNL2} are equivalent to the following:
\begin{align}
(N/L) &\ge \frac{4}{c}\log(NK/L), \label{eq:NL1}\\
(N/L) &\ge \frac{\log_2 e}{C}\log(NK/L). \label{eq:NL2}
\end{align}
We can immediately see that if \eqref{eq:NL1} is satisfied, \eqref{eq:NL2} is also satisfied since $c \le C$.

Therefore, 
we get that if $(N/L) \ge \frac{4}{c}\log(NK/L)$, with probability of at least $1-(N/L)^{-1}$,
\begin{subequations} \label{eq:H_r}
\begin{align}
\lambda_{K}\left(\bm M_{m_r}\bm M_{m_r}^{\top} \right) &{\black \ge  cN/L},\\
\lambda_{1}\left(\bm M_{m_r}\bm M_{m_r}^{\top} \right) &{\black \le  CN/L}.
\end{align}
\end{subequations}

%


Using the above result, we can also have
\begin{subequations} \label{eq:H_2r}
\begin{align}
\lambda_{K} ([\bm M_{\ell_r},\bm M_{\ell_{r+1}}][\bm M_{\ell_r},\bm M_{\ell_{r+1}}]^{\top}) \ge {\black 2cN/L},\\
\lambda_{1} ([\bm M_{\ell_r},\bm M_{\ell_{r+1}}][\bm M_{\ell_r},\bm M_{\ell_{r+1}}]^{\top}) \le {\black 2CN/L},
\end{align}
\end{subequations}
each holding with probability at least $1-(2N/L)^{-1}${\black , if $(2N/L) \ge  \frac{4}{c}\log(2NK/L)$}.

Thus, applying \eqref{eq:H_r} and \eqref{eq:H_2r} in \eqref{eq:SS}, we get
\begin{align*}
\lambda_{K} (\bm C_r^{\top}\bm C_r)
&\ge \lambda_{K}^2 (\bm B)\lambda_{K} ([\bm M_{\ell_r},\bm M_{\ell_{r+1}}][\bm M_{\ell_r},\bm M_{\ell_{r+1}}]^{\top})\lambda_{K} (\bm M_{m_r}\bm M_{m_r}^{\top})\\
&\ge \lambda_{K}^2 (\bm B) c(2N/L)c(N/L) = 2\lambda_{K}^2 (\bm B){\black c^2N^2/L^2},\\
&= 2\sigma_{\min}^2 (\bm B){\black c^2N^2/L^2},
\end{align*}
with probability at least $1-(2N/3L)^{-1}$ where we have used the facts that $\lambda_{1}((\bm A\bm B)^{-1}) \le \lambda_{1}(\bm A^{-1})\lambda_{1}(\bm B^{-1})$ and $\lambda_{1}(\bm A^{-1}) = 1/\lambda_{K}(\bm A)$ for obtaining the first inequality and used $\lambda_{K} (\bm B)=\sigma_{\min} (\bm B)$ for the last equality.

Similarly, we can get the below with probability at least $1-(2N/3L)^{-1}$:
\begin{align*}
\lambda_{1} (\bm C_r^{\top}\bm C_r) 
&\le \lambda_{1}^2 (\bm B)\lambda_{1} ([\bm M_{\ell_r},\bm M_{\ell_{r+1}}][\bm M_{\ell_r},\bm M_{\ell_{r+1}}]^{\top})\lambda_{1} (\bm M_{m_r}\bm M_{m_r}^{\top})\\
&\le \lambda_{1}^2 (\bm B) C(2N/L)C(N/L) = 2\lambda_{1}^2 (\bm B){\black C^2N^2/L^2},\\
&=2\sigma_{\max}^2 (\bm B){\black C^2N^2/L^2}.
\end{align*}

Therefore, by taking union bound over every $r \in [L-1]$, with probability at least $1-(3L(L-1)/N)$, the following equations hold simultaneously for every $r$:
\begin{align}
\sigma_{\min}([\bm P^{\top}_{\ell_r,m_r},\bm P^{\top}_{\ell_{r+1},m_r}]^{\top})=\sigma_{\min}(\bm C_r)
&=  \sqrt{\lambda_{K} (\bm C_r^{\top}\bm C_r)} \ge \sqrt{2}\sigma_{\min} (\bm B){\black cN/L} \label{eq:lambdamin},\\
\sigma_{\max}([\bm P^{\top}_{\ell_r,m_r},\bm P^{\top}_{\ell_{r+1},m_r}]^{\top})=\sigma_{\max}(\bm C_r)
&=  \sqrt{\lambda_{1} (\bm C_r^{\top}\bm C_r)}\le \sqrt{2}\sigma_{\max} (\bm B){\black CN/L} \label{eq:lambdamax}.
\end{align}

The equations \eqref{eq:lambdamin} and \eqref{eq:lambdamax} immediately follow that the below holds with probability at least $1-(3L(L-1)/N)$:
\begin{align*}
\sigma_K &= \underset{r\in \{1,\ldots,L-1\}}{\min}~ \sigma_{\min}([\bm P^{\top}_{\ell_r,m_r},\bm P^{\top}_{\ell_{r+1},m_r}]^{\top})\\
&\ge \sqrt{2}\sigma_{\min} (\bm B){\black cN/L} ,\\
\sigma_1 &= \underset{r\in \{1,\ldots,L-1\}}{\max}~ \sigma_{\max}([\bm P^{\top}_{\ell_r,m_r},\bm P^{\top}_{\ell_{r+1},m_r}]^{\top})\\
&\le \sqrt{2}\sigma_{\max} (\bm B){\black CN/L} .
\end{align*}
{\black Combining the conditions to obtain \eqref{eq:H_r} and \eqref{eq:H_2r}, we get that the above results hold true if $(N/L) \ge \frac{4}{c}\log(NK/L)$.}

\section{Proof of Lemma \ref{lem:Mestim} } \label{sup:spanoisy}
We consider the given noisy NMF model:
\begin{align*}
    {\widehat{\bm U}^{\top}} = {\bm O^{\top}\bm G^{\top}}\bm M+\bm N.
\end{align*}
{\black It is given that the assumptions in Proposition \ref{lem:Lm} hold true. This immediately means that $\bm M$ satisfies $\varepsilon$-separability (see Definition \ref{def:sep}), i.e., there exists a set of indices $ \bm \varLambda:=\{q_1,\dots,q_K\}$} such that
$
\bm M(:,{\black \bm \varLambda}) = \bm{I}_{K} + \bm{E},
$ where
$\bm{E} \in \mathbb{R}^{K \times K}$ is the error matrix with 
$
\|\bm{E}(:,k)\|_{2}  \le  \varepsilon
$
and $\bm{I}_{K}$ is the identity matrix of size $K\times K$. 
Without loss of generality, we let {\black $\bm \varLambda=\{1,\ldots,K\}$}. Therefore, we have
\begin{align*}
{\widehat{\bm U}^{\top}} &= {\bm O^{\top}\bm G^{\top}}\bm M+\bm N = 
 {\bm O^{\top}\bm G^{\top}} [\bm{I}_{K} + \bm E ~,~ \widetilde{\bm M}]+\bm N \\
&= {\bm O^{\top}\bm G^{\top}} [\bm{I}_{K} ,\widetilde{\bm M}] + [{\bm O^{\top}\bm G^{\top}} \bm{E} , \bm{0}]+{\bm N},
\end{align*}
where the zero matrix $\bm{0}$ has the same dimension of $\widetilde{\bm{M}}$.

Let $\widetilde{\bm N} = [{\bm O^{\top}\bm G^{\top}} \bm{E} , \bm{0}]+{\bm N}$.  This gives us the noisy NMF model as 
\begin{align} \label{eq:noisynmf_spa}
{\widehat{\bm U}^{\top}} = \underbrace{{\bm O^{\top}\bm G^{\top}}}_{\bm W} \underbrace{[\bm{I}_{K} ,\widetilde{\bm M}]}_{\bm M} + \widetilde{\bm N}.
\end{align}

Then, the below holds for any $k \in [K]$:
\begin{align}
\| \widetilde{\bm{N}}(:,k) \|_2 & \le \|\bm G \bm O\|_2 \|\bm{E}(:,k)\|_2+ \|{{\bm N}}(:,k)\|_2, \nonumber\\
& \le \sigma_{\rm max}(\bm G)\varepsilon + \zeta := \widetilde{\zeta}(\bm G), \label{n1}
\end{align}
where the first inequality is by applying triangle inequality and the second inequality is by using the orthogonality of $\bm O$ and also by using the given assumption $\|{{\bm N}}(:,k)\|_2 \le \zeta$.

We utilize the below lemma from \cite{Gillis2012} which characterizes the estimation accuracy of SPA under the NMF model \eqref{eq:noisynmf_spa}.

\begin{lemma}\cite{Gillis2012} \label{thm:spa}
	Consider the noisy model in \eqref{eq:noisynmf_spa} and suppose that $\bm W$ is full rank and $\bm M$ satisfies the simplex constraints, i.e., $\bm M \ge \bm 0 , \bm 1^{\top}\bm M = \bm 1^{\top}$. Let each column of $\widetilde{\bm N}$ satisfies the condition $\|\widetilde{\bm N}(:,k)\|_2  \le \widetilde{\zeta}$ with
	\begin{align} \label{eq:spabound}
	\widetilde{\zeta}\le {\sigma_{\min}(\bm W)}\varrho\widetilde{\kappa}^{-1}(\bm W).
	\end{align}
	Then the \texttt{SPA} Algorithm returns indices {\black $\widehat{\bm \varDelta} = \{\hat{q}_1,\dots,\hat{q}_K\}$} such that for each $k$,
	\begin{align} \label{eq:spagillis}
	\underset{\pi}{\min}~\|{\widehat{\bm U}}({\black \hat{q}_{k}},:)-\bm W(:,\pi(k))\|_2 \le \widetilde{\kappa}(\bm W)\widetilde{\zeta},
	\end{align}
	where $\widetilde{\kappa}(\bm W):=1+80\kappa^2(\bm W)$, $\varrho := \text{\rm min}\left(\frac{1}{2\sqrt{K-1}},\frac{1}{4}\right)$ and $\pi$ is certain permutation for $k \in [K]$.
\end{lemma}


%

   Therefore, we can apply Lemma \ref{thm:spa} to the the noisy model in \eqref{eq:noisynmf_spa} and obtain an estimate $\widehat{\bm G} = \widehat{\bm U}({\black \widehat{\bm \varDelta}},:)$ by applying \texttt{SPA} such that
\begin{align} \label{eq:estimG}
    \|\widehat{\bm G}-\bm \Pi\bm G\bm O \|_{\rm F} \le \sqrt{K}\widetilde{\kappa}( \bm G)\widetilde{\zeta}(\bm G).
\end{align}
In the above, we applied $\bm W = \bm O^{\top}\bm G^{\top}$, used the norm equivalence to convert from $\ell_2$-norm to the Frobenius norm, used the fact that $\sigma_{\min}(\bm G \bm O) = \sigma_{\min}(\bm G)$ and $\sigma_{\max}(\bm G \bm O) = \sigma_{\max}(\bm G)$ due to the orthogonality of the matrix $\bm O$ and $\bm \Pi$ is a permutation matrix such that $\bm \Pi = \underset{ \widetilde{\bm \Pi}}{{\rm arg~min}}~\|\widehat{\bm G}-\widetilde{\bm \Pi}\bm G\bm O\|_{\rm F}$. Also, the condition in Lemma \ref{thm:spa} can be represented as
\begin{align} \label{eq:cond_zetaepsilon1}
	\widetilde{\zeta}(\bm G) = \sigma_{\rm max}(\bm G)\varepsilon + \zeta \le \frac{\sigma_{\min}(\bm G)\varrho}{\widetilde{\kappa}(\bm G)}.
	\end{align}

Next, we proceed to characterize the singular values of $\bm G$ to express the above bound in more intuitive way. To this end, we have the following lemma: 
\begin{lemma} \label{lem:sigmaUG}
The below relations hold:
\begin{align*}
          \sigma_{\max}(\bm G) &= 1/\sigma_{\min}([\bm M_T, \bm M_{T+1}]) \le \frac{1}{\alpha_{\min}}, \\ \sigma_{\max}(\bm G) &\ge \sigma_{\min}(\bm G) = 1/\sigma_{\max}([\bm M_T, \bm M_{T+1}])\ge \frac{1}{\sqrt{2K}\alpha_{\max}},\\ \kappa(\bm G) &\le \sqrt{2K}\gamma.
\end{align*}

\end{lemma}
\begin{proof}
Recall that, under noiseless case, Algorithm~\ref{algo:proposed} first directly estimates $\bm U_{T}$ and $\bm U_{T+1}$ by performing the top-$K$ SVD to ${\bm C}_T =  [\bm P_{T,m_T}^{\top}, \bm P_{T+1,m_T}^{\top}]^{\top}$ where
\begin{align} \label{eq:CTsvd}
\bm C_T = [{\bm U}_{T}^{\top}, {\bm U}_{T+1}^{\top}]^{\top}\bm \Sigma_T{\bm V^{\top}_{m_T}}.
\end{align}

The above relation implies that,
$
[{\bm U}_{T}^{\top}, {\bm U}_{T+1}^{\top}]^{\top} =  [\bm M_T, \bm M_{T+1}]^{\top}\bm G
$
holds where $[{\bm U}_{T}^{\top}, {\bm U}_{T+1}^{\top}]^{\top}$ is semi-orthogonal.

Applying Fact~\ref{fact:orthU}, we can obtain
\begin{align*}
 \sigma_{\max}(\bm G) &= 1/\sigma_{\min}([\bm M_T, \bm M_{T+1}]) \quad \text{and}\\
  \sigma_{\min}(\bm G) &= 1/\sigma_{\max}([\bm M_T, \bm M_{T+1}]).
 \end{align*}
 By applying Lemma~\ref{lem:betaminmax}, we can finally get
 \begin{align*}
      \sigma_{\max}(\bm G) \le 1/\alpha_{\min},& \quad \sigma_{\min}(\bm G) \ge 1/(\sqrt{2K}\alpha_{\max}) \\
        \kappa(\bm G) &\le \sqrt{2K}\gamma.
 \end{align*}

\end{proof}
Applying Lemma~\ref{lem:sigmaUG}, we have
\begin{align*}
\widetilde{\kappa}(\bm G)&=1+80\kappa^2(\bm G) \le 1+160K\gamma^2 := \widetilde{\kappa}(\bm M),\\
     \widetilde{\zeta}(\bm G)&=\sigma_{\rm max}(\bm G)\varepsilon + \zeta \le \frac{\varepsilon}{\alpha_{\min}} + \zeta .
\end{align*}
Hence the bound in \eqref{eq:estimG} can be written as
\begin{align} \label{eq:estimG1}
    \|\widehat{\bm G}-\bm \Pi\bm G\bm O \|_{\rm F} \le \sqrt{K}\widetilde{\kappa}(\bm M)({\varepsilon}/{\alpha_{\min}} + \zeta),
\end{align}
and the condition \eqref{eq:cond_zetaepsilon1} can be rewritten as:
\begin{align} \label{eq:cond_zetaepsilon2}
	\frac{\varepsilon}{\alpha_{\min}} + \zeta \le \frac{\varrho}{\sqrt{2K}\alpha_{\max}\widetilde{\kappa}(\bm M)}.
	\end{align}
Using the given assumptions on $\varepsilon$ and $\zeta$, we can see that \eqref{eq:cond_zetaepsilon2} gets satisfied.

Now that we have $\bm G$ estimated as given by \eqref{eq:estimG1}, the next step is to estimate the matrix $\bm M$.
Algorithm~\ref{algo:proposed} estimates $\bm M$ via least squares.
To characterize this step, we consider the following lemma:
\begin{lemma}\cite{Wedin1973Perturbation}\label{lem:leastsquare1}
	Consider the noiseless model $\bm X = \bm W \bm H$. Suppose that $\widehat{\bm X}$ and $\widehat{\bm W}$ are the noisy estimates of $\bm X$ and $\bm W$, respectively.  Assume that there exist constants $\phi_1,\phi_2 > 0$ such that 
	$\|\widehat{\bm W}-\bm W\|_2 \le {\phi_1}\|\bm W\|_2,\quad \|\widehat{\bm X}-\bm X\|_2 \le \phi_2\|\bm X\|_2$ and $\phi_1\kappa(\bm W) < 1$.
	Then we have 
	$${\frac{\|\widehat{\bm H}-\bm H\|_2}{\|\bm H\|_2}} \le \frac{\kappa(\bm W)({\phi_1}+{\phi_2})}{1-{\phi_1}\kappa(\bm W)} + \kappa(\bm W)\phi_1,$$ where  $\widehat{\bm H}$ is the least squares solution to the problem $\widehat{\bm X}=\widehat{\bm W}\bm H$.
\end{lemma}

By fixing $ \bm X := \bm O^{\top}\bm U^{\top}$, $\bm W:= \bm O^{\top}\bm G^{\top} \bm \Pi^{\top}$ and $\bm H:= \bm \Pi\bm M$, we can apply Lemma~\ref{lem:leastsquare1}.
By applying the matrix norm equivalence and considering the fact that the matrix $\bm O$ is orthogonal, the constants $\phi_1$ and $\phi_2$ in Lemma~\ref{lem:leastsquare1} takes the following values:
\begin{align*}
    \phi_1 &=\sqrt{K}\widetilde{\kappa}( \bm M)({\varepsilon}/{\alpha_{\min}} + \zeta)/\sigma_{\max}(\bm G),~\phi_2 = \zeta/\sigma_{\max}(\bm U),
\end{align*}
where the first equality is from \eqref{eq:estimG1},  the second equality is by using the given assumption that $\|\widehat{\bm U}-\bm U\bm O\|_{\rm F} = \|\bm N\|_{\rm F} \le \zeta$, and $\widetilde{\kappa}(\bm M)=1+160\gamma^2$. 

Then, by applying Lemma~\ref{lem:leastsquare1}, the estimation bound for $\bm M$ can be expressed as
\begin{align} \label{eq:Mbound1}
{\frac{\|\widehat{\bm M}-\bm \Pi\bm M\|_2}{\|\bm M\|_2}} \le \frac{\kappa(\bm G)({\phi_1}+{\phi_2})}{1-{\phi_1}\kappa(\bm G)} + \kappa(\bm G)\phi_1.
\end{align}

Using the given assumptions on $\varepsilon$ and $\zeta$, we can see that the denominator term $1-{\phi_1}\kappa(\bm G) $ is lower bounded:
\begin{align}
    1-{\phi_1}\kappa(\bm G) &=1-\left(\frac{\sqrt{K}\widetilde{\kappa}( \bm M)\kappa(\bm G)({\varepsilon}/{\alpha_{\min}} + \zeta)}{\sigma_{\max}(\bm G)}\right) ,\nonumber\\
    &= 1-\left(\frac{\sqrt{K}\widetilde{\kappa}( \bm M)({\varepsilon}/{\alpha_{\min}} + \zeta)}{\sigma_{\min}(\bm G)}\right) \nonumber\\
    &\ge 1-\left({\sqrt{2}K\alpha_{\max}\widetilde{\kappa}( \bm M)({\varepsilon}/{\alpha_{\min}} + \zeta)}\right) \nonumber\\
    &\ge 1/2.\label{eq:condphi}
\end{align}
where we have used the assumptions on $\varepsilon$ and $\zeta$ for the last inequality.
Note the result in \eqref{eq:condphi} also makes sure that the condition ${\phi_1}\kappa(\bm G) < 1$ in Lemma~\ref{lem:leastsquare1} is satisfied.

By applying \eqref{eq:condphi} in the bound \eqref{eq:Mbound1}, we get 
\begin{align}
{\frac{\|\widehat{\bm M}-\bm \Pi\bm M\|_2}{\|\bm M\|_2}}
&\le 2{\kappa(\bm G)({\phi_1}+{\phi_2})} + \kappa(\bm G)\phi_1 = \kappa(\bm G)(3\phi_1+2\phi_2) \nonumber\\
&= {\kappa(\bm G)\left(\frac{3\sqrt{K}\widetilde{\kappa}( \bm M)({\varepsilon}/{\alpha_{\min}} + \zeta)}{\sigma_{\max}(\bm G)}+\frac{2\zeta}{\sigma_{\max}(\bm U)}\right)}. \label{eq:Mbound2norm}
\end{align}
Consider the below to lower bound $\sigma_{\max}(\bm U)$ where $\bm U^{\top} = \bm G^{\top}\bm M$:
 \begin{align}
     \sigma_{\max}(\bm U) &\ge \sigma_{\min}(\bm G)\sigma_{\max}(\bm M)\nonumber\\
     &= \frac{\sigma_{\max}(\bm M)}{\sigma_{\max}([\bm M_T, \bm M_{T+1}])} \ge 1, \label{eq:boundsingularU}
 \end{align}
 where we used Lemma~\ref{lem:sigmaUG} to obtain the first equality and used the fact that $\sigma_{\max}(\bm M) \ge \sigma_{\max}([\bm M_T, \bm M_{T+1}])$ for the last inequality.

Applying Lemma~\ref{lem:sigmaUG} and \eqref{eq:boundsingularU}, we have
\begin{align*}
{\|\widehat{\bm M}-\bm \Pi\bm M\|_2} 
&\le 
\sigma_{\max}(\bm M)\left(3\sqrt{K}\widetilde{\kappa}(\bm M)\left(\frac{\varepsilon}{\alpha_{\min}}+\zeta\right)\sqrt{2K}\alpha_{\max}+2\sqrt{2K}\gamma\zeta\right)\\
&~= \sigma_{\max}(\bm M)\left(3\sqrt{2}K\widetilde{\kappa}(\bm M)\gamma\varepsilon+3\sqrt{2}K\widetilde{\kappa}(\bm M)\alpha_{\max}\zeta+2\sqrt{2K}\gamma\zeta\right)\\
&~\le \sigma_{\max}(\bm M)\left(3\sqrt{2}K\widetilde{\kappa}(\bm M)\gamma\varepsilon+4\sqrt{2}K\widetilde{\kappa}(\bm M)\gamma\zeta\right)\\
&~= \sigma_{\max}(\bm M)\sqrt{2}K\gamma\widetilde{\kappa}(\bm M)\left(3\varepsilon+4\zeta\right),
\end{align*}
where $\widetilde{\kappa}(\bm M) = 1+160K\gamma^2$.

{\black 
\section{Successive Projection Algorithm (SPA)} \label{sup:spa}

In this section, we detail the successive projection algorithm (\texttt{SPA}) \cite{Gillis2012,fu2014self,araujo2001successive} presented in Algorithm \ref{algo:SPA}.

Assume that the separability condition holds.
For the model $
\bm U^\T = \bm G^\T \bm M \in \mathbb{R}^{K \times N}
$, \texttt{SPA} estimates the anchor nodes $\bm \varLambda = \{q_1,\dots, q_K\}$ through a series of steps as follows:
\begin{subequations}\label{eq:spa}
	\begin{align}
	{q}_1& = \arg\max_{n\in[N]}~\left\|\bm U(n,:)\right\|_2^2, \label{eq:spa1}\\
	{q}_k& = \arg\max_{n\in[N]}~\left\|\bm P^\perp_{\overline{\bm G}(:,1:k-1)}\bm U(n,:)\right\|_2^2,~k>1, \label{eq:spa2}
	\end{align}
\end{subequations}
where 
$\overline{\bm G}(:,1:k-1)=[\bm U^{\top}(:,{q}_1),\ldots,\bm U^{\top}(:,{q}_{k-1})]$ and $\bm P^\perp_{\overline{\bm G}(:,1:k-1)}$ is a projector onto the orthogonal  complement of ${\rm range}(\overline{\bm G}(:,1:k-1))$.

\begin{algorithm}[t]
{\black
\caption{\small Successive Projection Algorithm (\texttt{SPA}) \cite{Gillis2012}}
\label{algo:SPA}
\begin{algorithmic}[1]
\small
	\Require $\bm U $, $K$

\State Initialize $\widetilde{\bm U} \leftarrow \bm U $;

\ForEach{$k=1:K$}
    \State ${q}_k\leftarrow \arg\max_{n\in[N]}~\|\widetilde{\bm U}(n,:)\|_{\rm 2}^2$;\algorithmiccomment{Select anchor node }
    
    \State $\bm u_k \leftarrow \widetilde{\bm U}({q}_k,: )$;
    
    \State $\widetilde{\bm U} \leftarrow   \widetilde{\bm U}\left(\bm I_K - \frac{\bm u_k^{\top} \bm u_k}{\|\bm u_k\|_2}\right)  $;
\EndForEach

\State $\bm G \leftarrow \bm U(\{{q}_1,\dots,{q}_K\},:) $;


 
 
 \State \Return  $\bm G$.
 
\end{algorithmic}
}
\end{algorithm}

Eq. \eqref{eq:spa1} indicates that the first anchor node ${q}_1$ can be identified by finding the index of the row of $\bm U$ that has the maximum $\ell_2$ norm. This can be understood by considering the below for any $n \in [N]$,
\begin{align*}
\|{\bm U}(n,:)\|_2 &= \left\|\sum_{k=1}^K\bm G(k,:)\bm M(k,n)\right\|_2 
\leq \sum_{k=1}^K\left\|\bm G(k,:)\bm M(k,n)\right\|_2\\
&= \sum_{k=1}^K\bm M(k,n)\left\|\bm G(k,:)\right\|_2 
\leq \max_{k=1,\ldots,K}\left\|\bm G(k,:)\right\|_2,
\end{align*}
where the first inequality is due to triangle inequality, the second equality is due to the non-negativity of $\bm M$ and the last inequality is by the sum-to-one condition on the rows of $\bm M$.
In the above, all the equalities hold simultaneously if and only if $\bm M(:,n)=\bm e_k^\T$ for a certain $k$, i.e., $n \in \{q_1,\dots, q_K\}$.

After identifying the first anchor node ${q}_1$, all the rows of $\bm U$ are projected onto the orthogonal complement of $\bm U({q}_1,:)$  such that the same index ${q}_1$ will not show up in the subsequent steps. From the projected columns, the next anchor node is then identified in a similar way as shown in Eq. \eqref{eq:spa2}. Following this procedure, \texttt{SPA} can identify the anchor node set $\bm \varLambda$ in $K$ steps.

The work in \cite{Gillis2012} showed that under the following noisy model 
\begin{equation}\label{eq:SMF2}
    \widehat{\bm U}^\T = \bm G^\T \bm M +\bm N ,~\bm M\geq \bm 0,~\bm 1^\T\bm M =\bm 1^\T,
\end{equation}
where $\widehat{\bm U}$ is the noisy estimate of $\bm U$, \texttt{SPA} can still provably identify $\bm M$ under the separability condition if the noise $\bm N$ has bounded norm (see Theorem \ref{thm:spa}). 

}

\section{Permutation Fixing Procedure for Baseline Algorithms} \label{sup:baselines}

For synthetic and real data experiments, we employ the procedure in Algorithm \ref{algo:fix_perm} for the baseline algorithms to align the permutations of the estimated $\widehat{\bm M}_\ell$ over different blocks. {\black Note that this aligning procedure is not from any of the baseline papers, but is only presented for benchmarking the baselines under the diagonal EQP pattern.} {\black To explain this procedure, let us consider the diagonal EQP pattern as shown in Fig. \ref{fig:EQP_align} (a toy example with $L=3$).  In order to learn the membership matrix $\bm M = [\bm M_1, \bm M_2, \bm M_3]$ using the baselines, we may run the baseline algorithm first on $\bm A_{1,2}$ and subsequently on $\bm A_{2,3}$ to obtain the below (assuming no estimation error):
\begin{align*}
    [\bm M_1, \bm M_2] &= \texttt{BaselineAlgorithm}(\bm A_{1,2})\\
    [\overline{\bm M}_2, \overline{\bm M}_3] &= \texttt{BaselineAlgorithm}(\bm A_{2,3})
\end{align*}
 where $\overline{\bm M}_2=\bm \Pi{\bm M}_2$ and $\overline{\bm M}_3= \bm \Pi{\bm M}_3$. Here, $\bm \Pi \in \{0,1\}^{K \times K}$ is a row permutation matrix which is unknown. In order to learn the complete membership matrix, i.e., $\bm M = [\bm M_1, \bm M_2, \bm M_3]$, we need to identify this row permutation matrix $\bm \Pi$. This can be achieved by using any permutation fixing algorithms such as Hungarian algorithm \cite{jonker1986improving} by inputting $\bm M_2$ and $\overline{\bm M}_2(=  \bm \Pi{\bm M}_2)$. Algorithm \ref{algo:fix_perm} uses this idea in an iterative fashion to learn $\bm M$ using the baseline algorithms for general $L$ case. 

		\begin{figure}[H]
		\black 
		\centering
			\includegraphics[scale=0.25]{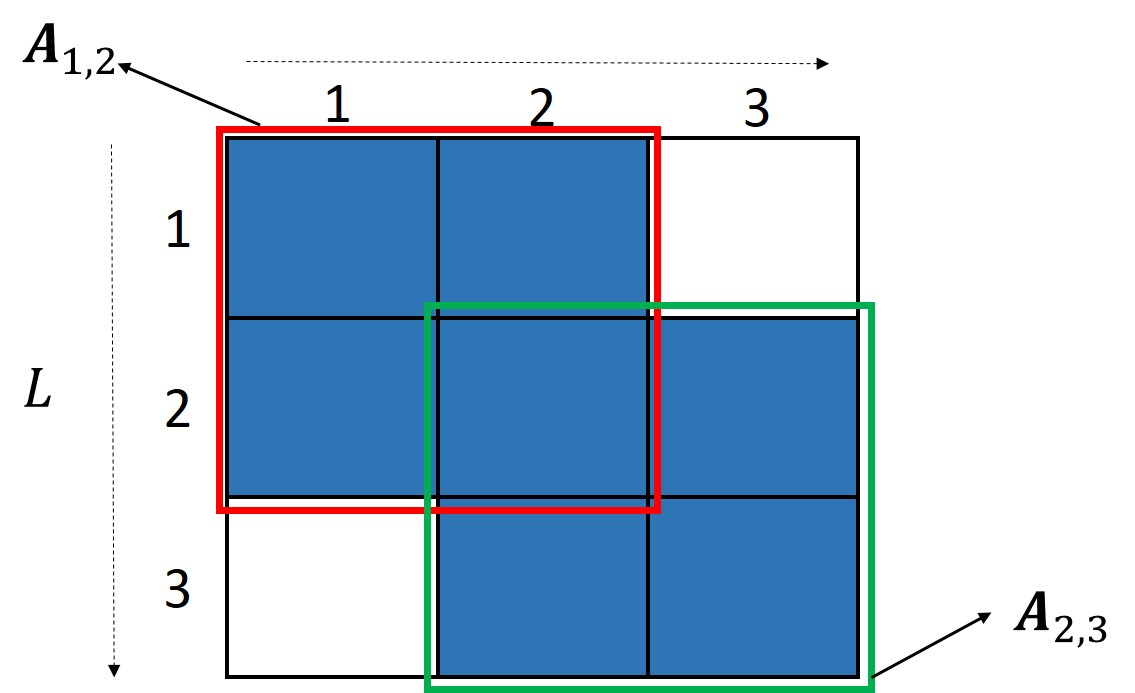}
			\caption{Adjacency matrix observed via diagonal EQP pattern ($L=3$). The blue shaded region represents the observed blocks in the adjacency matrix.}
			\label{fig:EQP_align}
		\end{figure}}

In Algorithm \ref{algo:fix_perm}, \texttt{BaselineAlgorithm} can be any graph clustering algorithm, for e.g., \texttt{GeoNMF}, {\black \texttt{SPACL},} \texttt{CD-MVSI}, \texttt{CD-BNMF}, \texttt{SC (unnorm.)}, \texttt{SC (norm.)}, \texttt{$k$-means} etc.

\begin{algorithm}[H]
\small
{
\caption{\texttt{Permutation Fixing}}
\label{algo:fix_perm}
\begin{algorithmic}[1]
\Require  $L$,$K$,$\{\bm A_{\ell,\ell}\}_{\ell=1}^L$, $\{\bm A_{\ell,\ell+1}\}_{\ell=1}^{L-1}$, \texttt{BaselineAlgorithm}.
 
 \State $\bm \Pi \leftarrow \bm I_K$;
 
\ForEach{$\ell=1$ {\bfseries to} $L-1$}
	
 \State  ${\bm H_\ell} \leftarrow \begin{bmatrix} \bm A_{\ell,\ell} ~,~ \bm A_{\ell,\ell+1}\\
\bm A_{\ell,\ell+1}^{\top} ~,~ \bm A_{\ell+1,\ell+1}\end{bmatrix}$;

 \State  $[\overline{\bm M}_\ell, \overline{\bm M}_{\ell+1}] \leftarrow$ \texttt{BaselineAlgorithm}($\bm H_\ell$,$K$);
 
\If{$\ell>1$}
	
 \State  $\bm \Pi \leftarrow $ \texttt{HugarianAlgorithm}($\overline{\bm M}_\ell$,$\bm M_\ell^{(prev)}$)\cite{jonker1986improving};
\EndIf

 \State  $[{\bm M}_\ell , {\bm M}_{\ell+1}] \leftarrow \bm \Pi^{\top}[\overline{\bm M}_\ell ,\overline{\bm M}_{\ell+1}]$;

 \State  $\bm M_\ell^{(prev)} \leftarrow {\bm M}_{\ell+1}$;
 \EndForEach
 \State  $\widehat{\bm M}=[\bm M_1,\ldots,\bm M_L]$;

\State \Return   $\widehat{\bm M}$.

\end{algorithmic}
}
\end{algorithm}

\begin{table}[H]
\black 
  \centering
  \caption{$K=5$, $L=10$, $\bm \nu = [\frac{1}{5} \frac{1}{2} 1 \frac{1}{2} \frac{1}{5}]^{\top} $ and $\eta=0.1$}
   \resizebox{0.7\linewidth}{!}{
    \begin{tabular}{c|c|c|c|c|c}
    \hline
    {Graph Size} & {Metric} & \texttt{BeQuec} & \texttt{GeoNMF} & \texttt{CD-MVSI} & \texttt{CD-BNMF} \\
    \hline
    \hline
    \multirow{3}[6]{*}{2000} & {MSE} & \textbf{0.5992} & 0.8277 & 0.7403 & 0.6893 \\
\cline{2-6}          & {RE} & \textbf{0.5378} & 0.5695 & 0.8282 & 0.7194 \\
\cline{2-6}          & {SRC} & \textbf{0.4452} & 0.2807 & 0.2211 & 0.3074 \\
    \hline
    \multirow{3}[6]{*}{4000} & {MSE} & \textbf{0.3123} & 0.3704 & 0.7325 & 0.6662 \\
\cline{2-6}          & {RE} & \textbf{0.6213} & 0.8930 & 0.8127 & 0.7245 \\
\cline{2-6}          & {SRC} & \textbf{0.5908} & 0.5855 & 0.2165 & 0.3204 \\
    \hline
    \multirow{3}[6]{*}{8000} & {MSE} & \textbf{0.1804} & 0.1859 & 0.7090 & 0.6743 \\
\cline{2-6}          & {RE} & \textbf{0.3021} & 0.3597 & 0.8215 & 0.7251 \\
\cline{2-6}          & {SRC} & \textbf{0.7465} & 0.7417 & 0.2376 & 0.3019 \\
    \hline
    \multirow{3}[6]{*}{10000} & {MSE} & {0.1454} & \textbf{0.1435} & 0.7023 & 0.6626 \\
\cline{2-6}          & {RE} & \textbf{0.2865} & 0.3580 & 0.8198 & 0.7210 \\
\cline{2-6}          & {SRC} & \textbf{0.7715} & 0.7585 & 0.2490 & 0.3110 \\
    \hline
    \hline
    \end{tabular}%
    }
  \label{tab:synth_alpha2}%
\end{table}%

\begin{table}[H]
\black
  \centering
  \caption{$K=5$, $L=10$, $\bm \nu = [\frac{1}{5} \frac{1}{3} \frac{1}{2} \frac{1}{5} 1]^{\top} $ and $\eta=0.2$}
   \resizebox{0.7\linewidth}{!}{
    \begin{tabular}{c|c|c|c|c|c}
    \hline
    {Graph Size ($N$)} & {Metric} & \texttt{BeQuec} & \texttt{GeoNMF} & \texttt{CD-MVSI} & \texttt{CD-BNMF} \\
    \hline
    \hline
    \multirow{3}[6]{*}{2000} & {MSE} & \textbf{0.7742} & 0.8266 & 0.8793 & 0.8217 \\
\cline{2-6}          & {RE} & \textbf{0.7198} & 0.8494 & 0.8237 & 0.8104 \\
\cline{2-6}          & {SRC} & \textbf{0.2761} & 0.2737 & 0.1770 & 0.1669 \\
    \hline
    \multirow{3}[6]{*}{4000} & {MSE} & \textbf{0.5280} & 0.3613 & 0.8810 & 0.8320 \\
\cline{2-6}          & {RE} & \textbf{0.5868} & 0.5263 & 0.8336 & 0.8161 \\
\cline{2-6}          & {SRC} & \textbf{0.4894} & 0.5496 & 0.1655 & 0.1431 \\
    \hline
    \multirow{3}[6]{*}{8000} & {MSE} & \textbf{0.1733} & 0.2082 & 0.8875 & 0.8313 \\
\cline{2-6}          & {RE} & \textbf{0.3178} & 0.3829 & 0.8439 & 0.8170 \\
\cline{2-6}          & {SRC} & \textbf{0.6886} & 0.6618 & 0.1575 & 0.1329 \\
    \hline
    \multirow{3}[6]{*}{10000} & {MSE} & \textbf{0.1385} & 0.2794 & 0.8861 & 0.8344 \\
\cline{2-6}          & {RE} & \textbf{0.2869} & 0.4508 & 0.8444 & 0.8184 \\
\cline{2-6}          & {SRC} & \textbf{0.7124} & 0.6216 & 0.1517 & 0.1245 \\
    \hline
    \hline
    \end{tabular}%
    }
  \label{tab:synth_Bmedium}%
\end{table}%

\begin{table}[H]
\black 
  \centering
  \caption{$K=5$, $L=10$, $\bm \nu = [\frac{1}{5} \frac{1}{3} \frac{1}{2} \frac{1}{5} 1]^{\top} $ and $\eta=0.3$}
   \resizebox{0.7\linewidth}{!}{
    \begin{tabular}{c|c|c|c|c|c}
    \hline
    {Graph Size ($N$)} & {Metric} & \texttt{BeQuec} & \texttt{GeoNMF} & \texttt{CD-MVSI} & \texttt{CD-BNMF} \\
    \hline
    \hline
    \multirow{3}[6]{*}{2000} & {MSE} & \textbf{0.7838} & 0.9181 & 0.8811 & 0.8494 \\
\cline{2-6}          & {RE} & \textbf{0.7136} & 0.8997 & 0.8135 & 0.8182 \\
\cline{2-6}          & {SRC} & \textbf{0.2991} & 0.2246 & 0.1634 & 0.1281 \\
    \hline
    \multirow{3}[6]{*}{4000} & {MSE} & \textbf{0.5952} & 0.4236 & 0.8790 & 0.8403 \\
\cline{2-6}          & {RE} & \textbf{0.5970} & 0.6097 & 0.8274 & 0.8159 \\
\cline{2-6}          & {SRC} & \textbf{0.4367} & 0.5100 & 0.1537 & 0.1254 \\
    \hline
    \multirow{3}[6]{*}{8000} & {MSE} & \textbf{0.2748} & 0.3225 & 0.8643 & 0.8467 \\
\cline{2-6}          & {RE} & \textbf{0.4157} & 0.5030 & 0.8181 & 0.8211 \\
\cline{2-6}          & {SRC} & \textbf{0.6332} & 0.5834 & 0.1607 & 0.1025 \\
    \hline
    \multirow{3}[6]{*}{10000} & {MSE} & \textbf{0.2005} & 0.2770 & 0.8728 & 0.8456 \\
\cline{2-6}          & {RE} & \textbf{0.3385} & 0.4903 & 0.8330 & 0.8214 \\
\cline{2-6}          & {SRC} & \textbf{0.6820} & 0.6131 & 0.1548 & 0.1006 \\
    \hline
    \hline
    \end{tabular}%
    }
  \label{tab:synth_Bweak}%
\end{table}%

{\black 
\section{Additional Synthetic Data Experiments} \label{supp:more_exp}
In this section, we present additional synthetic data experiments. 

Tables \ref{tab:synth_alpha2}-\ref{tab:synth_Bweak} present the results averaged over 20 random trials for different $\bm \nu$'s (i.e., different Dirichlet parameters) and $\eta$'s.
We let entries of $\bm \nu$ to be different, which simulates the scenario where unbalanced clusters are present. The values of $\eta$ are varied to test different levels of interactions among the clusters.
In these tables, we follow the same diagonal EQP pattern as in the manuscript. Note that in most of the cases, the proposed method outperforms the baselines. 
 One can also see that all methods' performance deteriorates when $\eta$ increases from 0.1 to 0.3. Nonetheless, the proposed method keeps its margin against the baselines under all $\eta$'s.

}

\end{document}